\documentclass{article} 
\usepackage{iclr2027_conference,times}

\usepackage{amsmath,amsfonts,bm}

\def\eqref#1{equation~\ref{#1}}

\def\1{\bm{1}}

\DeclareMathAlphabet{\mathsfit}{\encodingdefault}{\sfdefault}{m}{sl}
\SetMathAlphabet{\mathsfit}{bold}{\encodingdefault}{\sfdefault}{bx}{n}

\usepackage{hyperref}
\usepackage{url}

\usepackage{microtype}      
\usepackage{xcolor}         
\usepackage{amsmath,amssymb,amsthm}
\usepackage{thmtools}
\usepackage{thm-restate}
\usepackage{algorithm,algorithmic}
\usepackage{graphicx}
\usepackage{booktabs}
\usepackage{subcaption}
\usepackage{xspace}
\usepackage{enumitem}
\newcommand{\T}{\mathcal{T}} 
\usepackage{wrapfig}

\newtheorem{assumption}{Assumption}

\title{AdaGamma: State-Dependent Discounting for Temporal Adaptation in Reinforcement Learning}

\author{ Yaomin Wang$^{1,2}$, Jianting Pan$^{1}$, Ran Tian$^{2}$, Xiaoyang Li$^{2}$, \\\textbf{Yu Zhang}$^{2}$,  \textbf{Hengle Qin}$^{2,}$\footnotemark[1], \textbf{Tianshu YU}$^{1,}$\footnotemark[1]\\
    $^{1}$School of Data Science, The Chinese University of Hong Kong, Shenzhen\\
    $^{2}$JD.com\\
    \texttt{\{yaominwang, jiantingpan\}@link.cuhk.edu.cn}\\
    \texttt{\{tianran12, lixiaoyang1, zhangyu1496, qinhengle\}@jd.com}\\
    \texttt{\{yutianshu\}@cuhk.edu.cn}
}

\iclrfinalcopy 
\begin{document}

\maketitle

\begin{abstract}
The discount factor in reinforcement learning controls both the effective
planning horizon and the strength of bootstrapping, yet most deep RL methods use
a single fixed value across all states. State-dependent discounting is a natural
way to adapt temporal credit assignment to local state structure, but naively
learning a discount function inside deep actor--critic methods creates a
self-referential target: the learned discount scales the same bootstrap term
against which it is trained. Consequently, TD-error-based objectives can admit a
degenerate shortcut, reducing target dependence on downstream value estimates
instead of learning meaningful temporal structure. We propose \textbf{AdaGamma},
a practical framework for state-dependent discounting in deep actor--critic
reinforcement learning. AdaGamma learns a bounded discount function
\(\gamma_\phi(s)\) and trains it using a return-consistency objective that
matches a one-step adaptive bootstrap to a multi-step return computed with a
lagged, stop-gradient reference discount. This removes the direct target
manipulation path and encourages the discount to encode state-dependent value
propagation. We analyze the Bellman operator induced by a fixed
state-dependent discount in finite tabular MDPs, proving contraction, exact soft
policy iteration, and Lipschitz sensitivity under standard boundedness
assumptions. Empirically, AdaGamma integrates into SAC and PPO, avoids
TD-collapse, improves continuous-control benchmarks, and achieves statistically
significant gains in a four-week online A/B test on the JD Logistics platform. 
\end{abstract}

\section{Introduction}
\label{sec:intro}

The discount factor is a central design choice in reinforcement learning: it sets the effective planning horizon and the strength of bootstrapping in value estimation~\citep{sutton1998reinforcement,sutton2018reinforcement,Puterman1994MarkovDP}. Most deep RL methods, including TD3~\citep{fujimoto2018addressingfunctionapproximationerror}, TRPO~\citep{schulman2017trustregionpolicyoptimization}, DDPG~\citep{lillicrap2019continuouscontroldeepreinforcement}, PPO~\citep{schulman2017proximal}, and SAC~\citep{haarnoja2018soft}, use one fixed value for all states~\citep{mnih2015human,schulman2017proximal,haarnoja2018soft}.

For a discounted objective, $\gamma$ is part of the problem specification. But when deep RL is used to optimize an undiscounted or long-horizon return, $\gamma<1$ is better viewed as an estimation-time bias--variance choice rather than a property of the task. Discounting shortens the horizon over which errors compound, trading asymptotic fidelity for lower variance, and thus acts as a complexity control on the value class~\citep{amit2020discount,rathnam2024rethinking}. This trade-off is inherently local: target variance, transition reliability, and reachable remaining horizon all vary across states~\citep{francois2016discount,pitis2019rethinkingdiscountfactorreinforcement,hu2022rolediscountfactoroffline}. A single global $\gamma$ therefore imposes one compromise on states that may require different ones, motivating state-conditioned discounting.

Learning a discount module $\gamma_\phi(s)$ is the natural response, but it creates a structural problem. Unlike most auxiliary modules in actor--critic systems, $\gamma_\phi$ directly multiplies the bootstrap target it is trained against. Hence any objective that scores $\gamma_\phi$ by the resulting TD residual admits a degenerate shortcut: reducing the bootstrap coefficient weakens the target's dependence on downstream value estimates, allowing the loss to decrease without learning meaningful state-dependent temporal structure. The module becomes a lever for manipulating the target rather than a description of value propagation. Empirically, this is what occurs. On \texttt{SafetyPointGoal1-v0} under SAC, with the paper's $\gamma$-regularizers held fixed, squared-TD training of $\gamma_\phi$ degenerates to a near-constant discount ($\mathrm{Std}_s[\gamma_\phi(s)]\approx0.0014$), while return-consistency retains $\mathrm{Std}_s[\gamma_\phi(s)]\approx0.022$. A cross-validated TD-error objective also collapses despite its held-out split, giving mean $\gamma=0.9000$ on all three reported environments (Table~\ref{tab:gamma_sacs}).

AdaGamma removes this self-reference instead of penalizing its symptoms. We supervise $\gamma_\phi$ against an $n$-step return computed with a lagged scalar reference discount $\bar{\gamma}$, and the supervisory target is stop-gradient with respect to the current $\gamma_\phi(s_t)$. Shrinking the horizon can no longer trivially reduce the loss: the one-step estimate $r+\gamma_\phi(s)V(s')$ matches the multi-step return only when $\gamma_\phi(s)$ captures how much longer-horizon value the immediate transition carries. The learned discount thus performs adaptive bootstrapping, propagating farther where future information is reliable and relying less on bootstrapped targets where it is not.

On the theory side, we characterize the Bellman operator induced by a \emph{fixed} state-dependent discount function in a finite tabular MDP. Under bounded rewards and $\gamma(\cdot)\le\gamma_{\max}<1$, the operator is a contraction, exact soft policy iteration converges, and $Q_\gamma^\pi$ is Lipschitz in $\gamma(\cdot)$. We state the scope explicitly: the analysis covers this tabular operator with a frozen $\gamma(\cdot)$, not function approximation and not the joint dynamics of $\gamma_\phi$ and the critic, and it does not claim that state-dependent discounting dominates constant discounting for every MDP. The results align most directly with Bellman-style updates and soft policy iteration; our PPO variant is an empirical adaptation rather than an instance of the theory.

We instantiate AdaGamma in SAC and adapt it to PPO through corresponding return-estimation changes. Both improve continuous-control benchmarks, with larger gains in SAC. Finally, we deploy the SAC variant in a live recommender system on the JD Logistics platform, where it outperforms a SAC baseline in a four-week online A/B test.

\paragraph{Contributions.}
\begin{itemize}
  \item \textbf{Failure mode and method.}
  We identify a failure mode of TD-trained learned discounting: because $\gamma_\phi(s)$ appears in its own target, it can converge to a state-independent fixed point. AdaGamma avoids this by fitting $\gamma_\phi$ to multi-step returns computed with an external reference discount.

  \item \textbf{Theory.}
  We characterize the tabular Bellman operator induced by a fixed state-dependent discount, establishing contraction, exact soft policy iteration, and Lipschitz sensitivity under bounded rewards and $\gamma(\cdot)\le\gamma_{\max}<1$, and we state what the analysis does and does not cover.

  \item \textbf{Integration and validation.}
  AdaGamma integrates into SAC and PPO; we report continuous-control results, a TD-versus-RC collapse ablation, and a four-week online A/B test on JD Logistics.
\end{itemize}

Related work is discussed in Appendix~\ref{app:related}.

\section{Preliminaries}
\label{sec:prelim}

\paragraph{Notation}
We consider an infinite-horizon MDP $(\mathcal{S}, 
\mathcal{A}, p, r)$ with continuous state space $\mathcal{S}$ 
and action space $\mathcal{A}$. The transition density is denoted by 
$p(s'|s,a)$, equivalently 
$p:\mathcal{S}\times\mathcal{A}\times\mathcal{S}\to[0,\infty)$, 
and the reward function $r: \mathcal{S} 
\times \mathcal{A} \to [r_{\min}, r_{\max}]$ is bounded. 
We use $\rho_\pi(s_t)$ and $\rho_\pi(s_t, a_t)$ to denote 
the state and state-action marginals of the trajectory 
distribution induced by policy $\pi(a_t | s_t)$.

\paragraph{The Discount Factor's Dual Role}
The discount factor $\gamma$ serves two conceptually 
distinct functions: it defines the \emph{objective} 
(the $\gamma$-discounted return) and it acts as a 
\emph{damping coefficient} in the Bellman backup that 
controls how much bootstrapped future estimates influence 
current value updates. A low $\gamma$ prioritizes 
short-term rewards and reduces variance at the cost of 
bias; a high $\gamma$ encourages long-horizon planning 
but amplifies errors in value estimates. In standard 
algorithms, $\gamma$ is fixed, typically to 0.99 or 0.995.

\paragraph{Soft Actor-Critic (SAC)}
SAC \citep{haarnoja2018soft} maximizes the entropy-augmented 
objective:
\begin{equation}
\small
J(\pi) =
\mathbb{E}_{\pi}\left[
\sum_{t=0}^{\infty}\gamma^t
\left(r(s_t,a_t)+\alpha \mathcal{H}(\pi(\cdot|s_t))\right)
\right].
\label{eq:sac_obj}
\end{equation}
The soft Q-function satisfies the Bellman equation:
\begin{equation}
\small
Q^\pi(s_t, a_t) = r(s_t, a_t) + \gamma \, 
\mathbb{E}_{s_{t+1}\sim p}\left[V^\pi(s_{t+1})\right],
\label{eq:sac_bellman}
\end{equation}
where $V^\pi(s_t) = \mathbb{E}_{a_t\sim\pi}
[Q^\pi(s_t,a_t) - \alpha \log \pi(a_t|s_t)]$.

\paragraph{PPO and Generalized Advantage Estimation}
PPO \citep{schulman2017proximal} uses GAE 
to compute advantage estimates. The standard GAE recursion 
with fixed $\gamma$ and trace parameter $\lambda$ is:
\begin{align}
\delta_t &= r_t + \gamma V(s_{t+1}) - V(s_t), 
\label{eq:td_standard}\\
\hat{A}_t &= \delta_t + \gamma\lambda\,\hat{A}_{t+1}, 
\quad \hat{A}_{T-1} = \delta_{T-1}.
\label{eq:gae_standard}
\end{align}
Expanding the recursion yields 
$\hat{A}_t = \sum_{l=0}^{T-1-t}(\gamma\lambda)^l \delta_{t+l}$, 
a geometrically-weighted sum of TD residuals.

\section{AdaGamma: Method and Theory}
\label{sec:method}

We present AdaGamma as a unified framework with three components:
(i)~the gamma network architecture (Section~\ref{sec:architecture}),
(ii)~algorithm-specific adapters that integrate $\gamma_\phi(s)$ into bootstrapped value estimation (Sections~\ref{sec:offpolicy_adapter} and~\ref{sec:onpolicy_adapter}), and
(iii)~the training objective for the gamma network (Section~\ref{sec:gamma_training}).
We then provide supporting theory for the state-dependent discount operator underlying the method.
Our theoretical results characterize basic operator-level properties such as well-posedness, policy-evaluation behavior, and comparison bounds under suitable assumptions; they are not intended as a full end-to-end convergence guarantee for the complete deep RL algorithm with function approximation. Algorithm~\ref{alg:sac_adf} in Appendix~\ref{app:sac_ada_algo} presents the full SAC-AdaGamma procedure, and Algorithm~\ref{alg:ppo_adf} in Appendix~\ref{app:ppo_ada_algo} presents the PPO-AdaGamma variant.

\subsection{Gamma Network Architecture}
\label{sec:architecture}

The gamma network is a small MLP $g_\phi: \mathcal{S} 
\to \mathbb{R}$ whose output is passed through a sigmoid 
and rescaled to $[\gamma_{\min}, \gamma_{\max}]$:
\begin{equation}
\gamma_\phi(s_t) = \gamma_{\min} + 
(\gamma_{\max} - \gamma_{\min}) \cdot 
\sigma(g_\phi(s_t)),
\label{eq:gamma_param}
\end{equation}
where $\sigma(\cdot)$ is the sigmoid function and 
$\gamma_{\min}, \gamma_{\max}$ are hyperparameters 
(we use $\gamma_{\min}=0.900$, $\gamma_{\max}=0.999$ 
in experiments). This bounded parameterization ensures 
$\gamma_\phi(s) \in [\gamma_{\min}, \gamma_{\max}] 
\subset [0,1)$ for all states and keeps the bootstrap coefficient
strictly below one. Together with bounded rewards and standard
boundedness or integrability conditions on the entropy term, this keeps
the induced discounted backup well behaved.

The gamma network is deliberately kept lightweight---a 
two-layer MLP with hidden dimension 256---to avoid 
overfitting and to ensure negligible computational overhead. 
It takes the same state representation as the policy and 
value networks but uses no shared parameters, preventing 
gradient interference.

\subsection{The Common Interface: Bootstrapped Value Targets}
\label{sec:interface}

Despite their differences, both SAC and PPO share one operation where \(\gamma\) appears: the bootstrapped target used in value estimation. AdaGamma interfaces with each algorithm at precisely this point, replacing the scalar $\gamma$ with $\gamma_\phi(s_t)$.

\subsubsection{Off-Policy Adapter (SAC)}
\label{sec:offpolicy_adapter}

For SAC, the soft Q-function with adaptive discount becomes:
\begin{equation}
\small
Q^\pi(s_t, a_t) = r(s_t, a_t) + 
\mathbb{E}_{s_{t+1}\sim p}\left[
\gamma_\phi(s_t) \, \mathbb{E}_{a_{t+1}\sim\pi}
\left[Q^\pi(s_{t+1}, a_{t+1}) - 
\alpha \log\pi(a_{t+1}|s_{t+1})\right]
\right].
\label{eq:sac_adaptive_q}
\end{equation}
The practical target computation becomes:
\begin{equation}
\small
\hat{Q}(s_t, a_t) = r_t + \gamma_\phi(s_t)(1-d_t)
\left[\min_{i=1,2} Q_{\bar{\theta}_i}(s_{t+1}, a_{t+1}') 
- \alpha \log\pi_\psi(a_{t+1}'|s_{t+1})\right],
\label{eq:sac_target}
\end{equation}
where $a_{t+1}' \sim \pi_\psi(\cdot|s_{t+1})$ and 
$d_t \in \{0,1\}$ is the terminal flag. The policy update 
and entropy tuning are unchanged from standard SAC.

\subsubsection{On-Policy Adapter (PPO): Modified GAE}
\label{sec:onpolicy_adapter}

PPO is an empirical adaptation of the same interface; we do not claim
a parallel theory for clipping or GAE. The state-dependent discount
enters the TD residual and the GAE recursion:
\begin{equation}
\delta_t = r_t + \gamma_\phi(s_t)(1-d_t)\,V(s_{t+1}) - V(s_t),
\label{eq:td_adaptive}
\end{equation}
and the modified GAE backward recursion becomes:
\begin{equation}
\hat{A}_{T-1} = \delta_{T-1}, \qquad
\hat{A}_t = \delta_t + 
\gamma_\phi(s_t)\lambda(1-d_t)\hat{A}_{t+1}.
\label{eq:gae_adaptive}
\end{equation}

Unrolling~(\ref{eq:gae_adaptive}) yields a product-of-discounts
weighting of future residuals
(Proposition~\ref{prop:gae} in Appendix~\ref{app:proof_gae}).

A low-$\gamma$ nonterminal state along a trajectory suppresses subsequent
residuals, inducing a learned form of temporal segmentation; terminal states
truncate the recursion in the usual way. Advantage normalization and freezing
$\gamma_\phi$ within a PPO epoch are described in Appendix~\ref{app:ppo_ada_algo}.

\subsection{Training the Gamma Network}
\label{sec:gamma_training}

\subsubsection{The Collapse Problem}
\label{sec:collapse}

A natural first attempt is to train $\gamma_\phi$ by minimizing the squared TD
error $\delta_t^2$, since $\gamma_\phi(s_t)$ appears in $\delta_t$. However, this
creates a degenerate shortcut: $\gamma_\phi$ can collapse to a state-independent
fixed point rather than learn meaningful state-dependent horizons. We verify this
TD collapse empirically in Section~\ref{sec:td_exp}, and show that AdaGamma avoids it.

\subsubsection{Return-Consistency Objective}
\label{sec:return_consistency}

We instead train the gamma network via a 
\emph{return-consistency} objective. The idea is to 
use multi-step returns as a supervision signal that the 
gamma network cannot trivially game.

For each transition $(s_t, a_t, r_t, s_{t+1})$, compute 
two value estimates:
\begin{enumerate}
\item The one-step bootstrap under the learned $\gamma$: 
$\hat{V}_1(s_t) = r_t + \gamma_\phi(s_t)\,V(s_{t+1})$.
\item An $n$-step Monte Carlo return under a reference 
discount $\bar{\gamma}$, truncated at terminals:
\begin{equation}
G_t^{(n)}
=
\sum_{k=0}^{n-1}
\left(
\prod_{j=0}^{k-1}\bar{\gamma}(1-d_{t+j})
\right)r_{t+k}
+
\left(
\prod_{j=0}^{n-1}\bar{\gamma}(1-d_{t+j})
\right)V(s_{t+n}),
\label{eq:masked_nstep_return}
\end{equation}
computed with stop-gradient on $V$.
\end{enumerate}
The gamma network is trained to minimize:
\begin{equation}
L_\gamma^{\text{RC}}(\phi) = \mathbb{E}_{(s_t,a_t)\sim\mathcal{D}}
\left[\left(r_t + \gamma_\phi(s_t)\,
\text{sg}[V(s_{t+1})] 
- \text{sg}[G_t^{(n)}]\right)^2\right],
\label{eq:return_consistency}
\end{equation}
where $\text{sg}[\cdot]$ denotes the stop-gradient operator. 
The gamma network learns to set $\gamma_\phi(s_t)$ so that the local one-step bootstrap better matches a longer-horizon target, thereby adapting the degree of temporal propagation to the local state. In states where one-step bootstrapping provides a reliable summary of longer-horizon value, the learned discount tends to be larger. In states where immediate transitions are less predictive or more weakly aligned with longer-horizon outcomes, the learned discount tends to be smaller.
Critically, pushing $\gamma_\phi(s_t)$ toward the lower end of its
admissible range does \emph{not} in general minimize this loss, because
the corresponding one-step estimate need not equal the multi-step return
$G_t^{(n)}$.

In its simplest form, $\bar{\gamma}$ can be fixed to a constant 
such as $0.99$. In our SAC safety implementation, however, we use 
a slowly adaptive reference discount to better match the current 
training distribution. After warmup, every $M$ episodes we sample 
a replay batch $\mathcal{B}$, compute the replay-mean predicted 
discount $\mathbb{E}_{s \sim \mathcal{B}}[\gamma_\phi(s)]$, and 
update
\begin{equation}
\bar{\gamma} \leftarrow (1-\tau_{\mathrm{ref}})\bar{\gamma}
+ \tau_{\mathrm{ref}} \, \mathbb{E}_{s \sim \mathcal{B}}[\gamma_\phi(s)].
\label{eq:gamma_ref_ema}
\end{equation}
The supervisory target is computed using a lagged scalar reference discount
and is stop-gradient with respect to the current $\gamma_\phi(s_t)$. Thus
the gamma network cannot reduce the loss by directly changing the same
bootstrap coefficient that defines its target. The EMA update keeps the
$n$-step supervisory target smoother than using a fully state-dependent
product inside $G_t^{(n)}$, while allowing the effective reference horizon
to track the replay distribution over the course of training.

For off-policy algorithms, the $n$-step return is computed 
from sequences of transitions in the replay buffer and is truncated at
episode boundaries. For on-policy algorithms, the rollout data naturally 
provides multi-step returns, again with terminal masks applied. For SAC, \(V(s_{t+1})\) denotes the sampled soft value
\(\min_i Q_{\bar\theta_i}(s_{t+1},a')-\alpha\log\pi(a'|s_{t+1})\), with \(a'\sim\pi(\cdot|s_{t+1})\). For PPO, \(V\) is the learned state-value network.

A cross-validated TD baseline, used only as a comparator, is defined
in Appendix~\ref{app:cv_loss}.

\subsubsection{Full Training Objective}
\label{sec:full_objective}

The complete gamma network loss combines the 
return-consistency objective with regularization:
\begin{equation}
J_\gamma(\phi) = L_\gamma^{\text{RC}}(\phi) 
+ \lambda_{\text{dev}}\,\mathbb{E}_{s_t\sim\mathcal{D}}
\left[(\gamma_\phi(s_t) - \gamma_{\text{target}})^2\right]
+ \lambda_{\text{var}}\,\text{Var}_{s_t\sim\mathcal{D}}
[\gamma_\phi(s_t)]
+ \lambda_{\text{bound}}\,L_{\text{boundary}},
\label{eq:full_gamma_loss}
\end{equation}
where the deviation penalty anchors $\gamma_\phi$ near 
$\gamma_{\text{target}}$, the variance penalty encourages 
smoothness, and $L_{\text{boundary}}$ discourages values 
near the boundaries:
\begin{equation}
L_{\text{boundary}} = \mathbb{E}_{s_t\sim\mathcal{D}}
\left[\text{ReLU}(\gamma_{\min}+\epsilon_b - \gamma_\phi(s_t))
+ \text{ReLU}(\gamma_\phi(s_t) - \gamma_{\max}+\epsilon_b)\right].
\label{eq:boundary}
\end{equation}

\subsection{Operator-Level Analysis for State-Dependent Discounting}
\label{sec:theory}

We now study the Bellman operators induced by state-dependent discounting
and establish supporting guarantees for the resulting policy-evaluation
and policy-improvement dynamics. Our goal is not to prove end-to-end
convergence of the full AdaGamma algorithm, which additionally involves
neural function approximation, stochastic optimization, target networks,
an adaptive reference discount, and alternating actor--critic--gamma
updates. Instead, we isolate the operator-level consequences of replacing
a scalar discount factor by a state-dependent discount function.

For analytical clarity, we specialize the general setting of Section~\ref{sec:gamma_training} to
a finite tabular MDP. Unless stated otherwise, the following standing assumption is imposed throughout this section.

\begin{assumption}
\label{assumption}
The state space $\mathcal S$ and the action space $\mathcal A$ are finite with $m:=|\mathcal A|$, the
reward is bounded as
\(
|r(s,a)|\le R,\ \forall (s,a)\in\mathcal S\times\mathcal A, 
\)
and the entropy temperature $\alpha>0$ is fixed. All expectations are
evaluated exactly. The state-dependent discount function satisfies
\(
\gamma:\mathcal S\to[0,\beta],
\)
where \(\beta<1\).    
\end{assumption}
For policy-iteration results, we additionally consider the unrestricted
class of stationary tabular policies and assume exact policy evaluation
and exact statewise soft-greedy improvement.

\subsubsection{Soft Policy Evaluation}

We first consider policy evaluation under a fixed state-dependent discount
function. For a stationary policy $\pi$, define
\begin{equation}
V_Q^\pi(s)
=
\sum_{a\in\mathcal{A}}
\pi(a\mid s)
\bigl(Q(s,a)-\alpha\log\pi(a\mid s)\bigr),
\end{equation}
with the convention $0\log 0=0$. The corresponding state-dependent soft
Bellman operator is
\begin{equation}
(\T_{\gamma}^\pi Q)(s,a)
=
r(s,a)
+
\gamma(s)\,
\mathbb{E}_{s'\sim P(\cdot\mid s,a)}
\left[V_Q^\pi(s')\right].
\label{eq:bellman_op}
\end{equation}

\begin{restatable}{lemma}{policyeval}
(Fixed-policy contraction under state-dependent discounting)
\label{lem:policy_eval}
Under Assumption~\ref{assumption}, for any stationary policy
$\pi$, the operator $\T_{\gamma}^\pi$ is a $\beta$-contraction in the sup norm:
\begin{equation}
\|\T_{\gamma}^\pi Q-\T_{\gamma}^\pi Q'\|_\infty
\le
\beta\|Q-Q'\|_\infty .
\label{eq:policy_eval_contraction}
\end{equation}
Consequently, $\T_{\gamma}^\pi$ admits a unique fixed point
$Q_\gamma^\pi$, and the exact iteration
$Q^{k+1}=\T_{\gamma}^\pi Q^k$ converges to $Q_\gamma^\pi$ from any
initial $Q^0$. 
\end{restatable}

\begin{proof}
    See Appendix~\ref{app:proof_eval}.
\end{proof}

The $n$-fold operator $(\T_{\gamma}^\pi)^n$ is a $\beta^n$-contraction
and shares the same unique fixed point; along a trajectory the
difference of backups is weighted by the product of encountered
discounts
(Proposition~\ref{prop:revised-n-step} in Appendix~\ref{app:proof_nstep}).

\subsubsection{Exact Soft Policy Iteration under a Fixed State-Dependent Discount}

We next study exact soft policy improvement and policy iteration for a
\emph{fixed} state-dependent discount function. Throughout this subsection, the same discount function $\gamma(\cdot)$
is used in both policy evaluation and policy improvement.

For $\alpha>0$, define the soft optimality operator
\begin{equation}
(\T_{\gamma} Q)(s,a)
=
r(s,a)
+
\gamma(s)\,
\mathbb{E}_{s'\sim P(\cdot\mid s,a)}
\left[
\alpha\log
\sum_{b\in\mathcal{A}}
\exp\!\left(\frac{Q(s',b)}{\alpha}\right)
\right].
\end{equation}
For any $Q$, the corresponding statewise soft-greedy policy is
\begin{equation}
\label{eq:revised-soft-greedy}
\pi_Q(a\mid s)
=
\frac{\exp(Q(s,a)/\alpha)}
{\sum_{b\in\mathcal{A}}\exp(Q(s,b)/\alpha)}.
\end{equation}

Exact soft-greedy improvement is monotone
(Lemma~\ref{lem:soft_improve} in Appendix~\ref{app:proof_improve}): replacing $\pi$ by
$\pi_{Q^\pi_\gamma}$ cannot decrease the soft $Q$-function.

\begin{restatable}{theorem}{softpi}
(Exact soft policy iteration)
\label{thm:revised-soft-policy-iteration}
Under Assumption~\ref{assumption}, let $\pi_0$ be any stationary tabular
policy and define, for $k\ge 0$, \(Q_k := Q_\gamma^{\pi_k}\) and \(\pi_{k+1} = \pi_{Q_k}\),
where $Q_\gamma^{\pi_k}$ denotes the exact soft $Q$-function of
$\pi_k$ under the fixed state-dependent discount $\gamma(\cdot)$. Then $\{Q_k\}_k$ converges in the sup norm to the unique fixed point
$Q_\gamma^\star$ of the soft optimality operator $\T_\gamma$.
Moreover, \(\pi_k \rightarrow
\pi^\star := \pi_{Q_\gamma^\star}\), where $\pi^\star$ is the unique statewise soft-greedy policy associated
with $Q_\gamma^\star$. In addition,
\begin{equation}
Q_\gamma^\star(s,a)
\ge
Q_\gamma^\pi(s,a),
\qquad
\forall\,\pi,\quad
\forall\,(s,a)\in\mathcal S\times\mathcal A,
\label{eq:revised-spi-optimality}
\end{equation}
where $\pi$ ranges over all stationary tabular policies.
\end{restatable}

\begin{proof}
    See Appendix~\ref{app:proof_spi}.
\end{proof}

The map $\gamma(\cdot)\mapsto Q_\gamma^\pi$ is Lipschitz in the sup
norm, with constant scaling as
$(R+\alpha\log m)/((1-\beta_1)(1-\beta_2))$
(Theorem~\ref{thm:revised-sensitivity} in Appendix~\ref{app:proof_sens}).

If an exogenous sequence $\gamma_k$ varies sufficiently slowly, exact
fixed-policy backups track the moving fixed points
(Corollary~\ref{cor:revised-tracking} in Appendix~\ref{app:proof_track}); $\{\gamma_k\}$ is not the
endogenous return-consistency dynamics.

\section{Numerical Experiments}
\label{sec:experiments}

\subsection{Experimental Setup}
\label{sec:setup}
\textbf{Algorithms and baselines.} We instantiate AdaGamma in SAC \citep{haarnoja2018soft} and PPO \citep{schulman2017proximal}, using the same loss form~(\ref{eq:full_gamma_loss}) with algorithm-specific adapters (off-policy for SAC, on-policy for PPO). The gamma network is a two-layer MLP (hidden size 256) trained via the return-consistency objective of Section~\ref{sec:return_consistency}. Under a matched training budget, we compare against: fixed-$\gamma$ SAC/PPO ($\gamma{=}0.99$), an uncertainty-rule adaptive-$\gamma$ baseline (inspired by \citet{kim2022adaptive}; details in Appendix~\ref{app:uncertainty_rule}), and a cross-validated adaptive-$\gamma$ network. All methods use 5 seeds and are evaluated every $10^4$ steps. For AdaGamma, we use 5-step returns, initialize $\gamma$ and $\bar{\gamma}$ at $0.98$, warm up for $10^5$ steps, and update $\bar{\gamma}$ every 5 episodes (EMA coefficient $0.1$).

\textbf{Environments.} We evaluate AdaGamma in four settings: \texttt{SafetyPointGoal1-v0}~\citep{ji2023safety} for reward--cost tradeoffs and seed-shift generalization; \texttt{Humanoid-v4} and \texttt{Ant-v4} for high-dimensional continuous control; classic Gymnasium tasks~\citep{towers2024gymnasium} for smaller-scale SAC/PPO comparisons; and a JD Logistics target-marketing deployment for real population-specific decision making. Unless otherwise stated, we report mean $\pm$ standard deviation over random seeds.

Hyperparameter details are given in Appendix~\ref{app:hyperparams}; learning curves for SAC-AdaGamma and PPO-AdaGamma on \texttt{SafetyPointGoal1-v0}, \texttt{Humanoid-v4}, and \texttt{Ant-v4} are shown in Appendix~\ref{app:reward_curves}.

\begin{table}[t]
  \centering
 \caption{Test performance of fixed, adaptive, and AdaGamma variants on \texttt{SafetyPointGoal1-v0}, \texttt{Humanoid-v4}, and \texttt{Ant-v4}.}
  \label{tab:main_results}
  \scalebox{0.8}{
  \begin{tabular}{l|cc||l|cc}
    \toprule
    Method & Reward & Cost & Method & Reward & Cost\\
    \midrule
    \multicolumn{6}{c}{\texttt{SafetyPointGoal1-v0}} \\
    \hline
    SAC-Fixed-$\gamma$    & $27.82 \pm 0.91$ & $52.53 \pm 40.75$ & PPO-Fixed-$\gamma$    & $22.58 \pm 1.34$ & $51.26 \pm 34.51$ \\
    SAC-CrossValidate     & $27.34 \pm 1.05$ & $45.84 \pm 31.14$ & PPO-CrossValidate     & $20.62 \pm 5.81$ & $54.13 \pm 27.09$ \\
    SAC-Uncertainty       & $27.47 \pm 1.24$ & $37.46 \pm 32.93$ & PPO-Uncertainty       & $22.59 \pm 1.35$ & $51.25 \pm 34.51$ \\
    \hline
    SAC-AdaGamma          & $\mathbf{28.25 \pm 1.05}$ & $\mathbf{29.37 \pm 18.59}$ & PPO-AdaGamma          & $\mathbf{26.31 \pm 0.90}$ & $\mathbf{41.82 \pm 19.42}$ \\
    \hline
    \multicolumn{6}{c}{\texttt{Humanoid-v4}} \\
    \hline
    SAC-Fixed-$\gamma$      & 5370.80 $\pm$ 11.24 &  - & PPO-Fixed-$\gamma$      & 411.61 $\pm$ 39.99   & -\\
    SAC-CrossValidate       & 616.25 $\pm$ 21.88  &  - & PPO-CrossValidate       & 291.42 $\pm$ 14.48   & -\\
    SAC-Uncertainty         & 6606.08 $\pm$ 59.69  &  - & PPO-Uncertainty         & 434.32 $\pm$ 98.70   & -\\
    \hline
    SAC-AdaGamma            &  \textbf{6907.99 $\pm$ 21.76}  & - & PPO-AdaGamma            & \textbf{476.51 $\pm$ 51.29}   & - \\
    \hline
    \multicolumn{6}{c}{\texttt{Ant-v4}} \\
    \hline
    SAC-Fixed-$\gamma$      & 3101.48 $\pm$ 748.16  & -  & PPO-Fixed-$\gamma$      & 989.57 $\pm$ 113.71  & -  \\
    SAC-CrossValidate       & 1399.44 $\pm$ 107.24  & -  & PPO-CrossValidate       &  1036.33 $\pm$ 110.85  & -   \\
    SAC-Uncertainty         & 3620.87 $\pm$ 187.90  & -  & PPO-Uncertainty         & 1018.51 $\pm$ 222.44  & -  \\
    \hline
    SAC-AdaGamma            & \textbf{4129.57 $\pm$ 1155.08}  & -   & PPO-AdaGamma            & \textbf{1172.47 $\pm$ 189.84}   & -  \\
    \bottomrule
  \end{tabular}
  }
\vspace{-10pt}
\end{table}

\subsection{Continuous-control benchmarks}
\label{sec:safety_results}

Table~\ref{tab:main_results} reports SAC and PPO on
\texttt{SafetyPointGoal1-v0}, \texttt{Humanoid-v4}, and \texttt{Ant-v4}.
AdaGamma attains the best reward on all three tasks for both algorithms,
and the lowest cost on the safety task. The uncertainty-rule baseline
already improves on fixed $\gamma$, while the cross-validated TD
baseline is weak on \texttt{Humanoid-v4}. The largest gain is SAC on
\texttt{Humanoid-v4}. Snapshots are in Appendix~\ref{app:snapshots}.

\subsection{Online deployment on JD Logistics platform}
\label{sec:jd_logistics}
We further evaluate AdaGamma in a live recommender system on the JD Logistics
platform through a controlled four-week online A/B test.
We randomly assign approximately one million users
($N \approx 10^6$) to the experiment, split evenly between AdaGamma-enhanced SAC
and a standard SAC baseline ($\approx 5\times 10^5$ users per arm).
Users are stratified and split evenly between the two arms so that the
distributions of pre-experiment features are balanced, including user order
volume, order cycle, coupon usage frequency, coupon claiming statistics, and
other basic user attributes. All these pre-experiment balance checks pass
standard $p$-value tests.
During the experiment the two groups are treated identically except for the
deployed recommendation policy, so that the observed difference can be attributed
to the RL algorithm.

\paragraph{Task and reward.}
The policy selects coupon recommendations under a limited low-tier quota.
The reward combines re-scaled absolute order growth, relative order
growth, and a penalty for discounts deeper than a medium-tier reference,
so that order gains must outweigh subsidy cost
(Appendix~\ref{app:jd_reward}).

\paragraph{Primary metric and subsidy guardrail.}
The primary evaluation metric is the relative uplift in average
recommendation-induced orders per user. Each week we measure the incremental
improvement over the previous seven days relative to the state at group
assignment. As a guardrail we track the realized overall discount rate
(coupon subsidy as a fraction of GMV). The two arms finished the experiment
with essentially matched overall discount intensity (relative gap below $0.1\%$),
so the uplift in average recommendation-induced orders per user cannot be
attributed to a more aggressive subsidy. As shown in Figure~\ref{fig:JD_res},
AdaGamma consistently outperforms standard SAC over four weeks; standard
significance tests indicate reliable gains. These results suggest that adaptive
discounting reallocates a matched discount budget toward user states that
benefit from a longer planning horizon, rather than simply spending more.

\begin{wrapfigure}{r}{0.48\textwidth}
    \centering
   
    \includegraphics[width=0.46\textwidth]{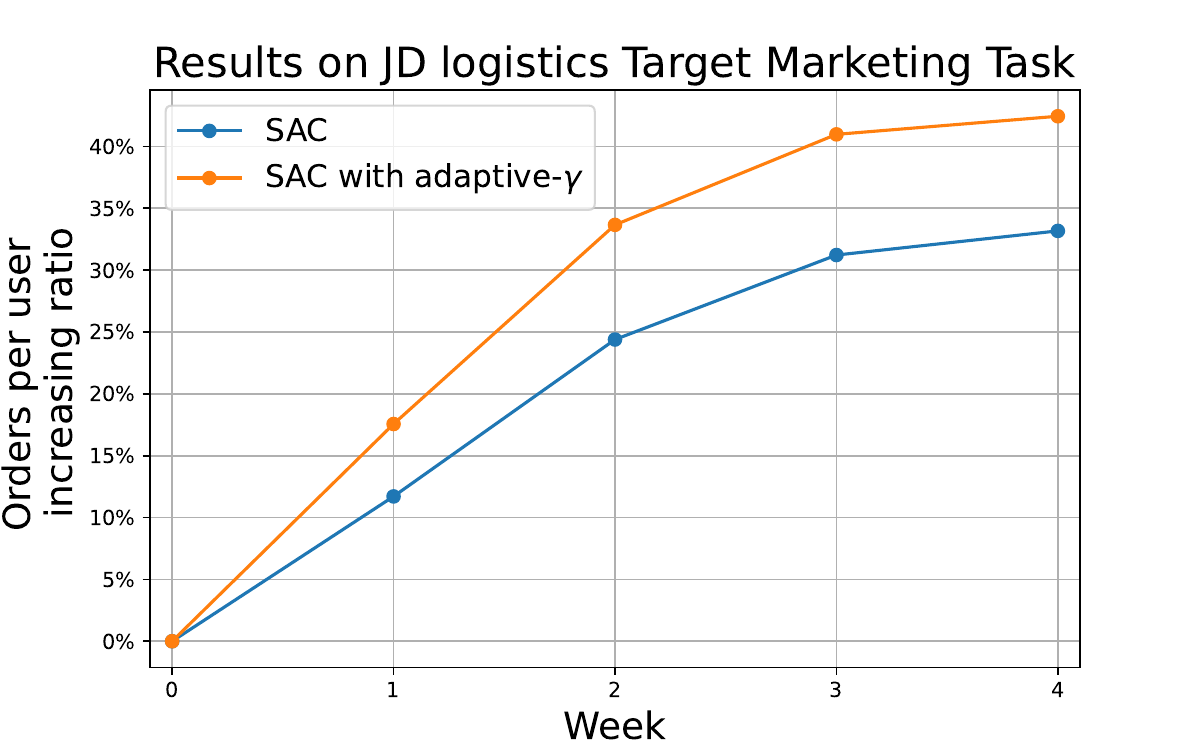}
    \caption{Four-week online A/B test on the JD Logistics target-marketing task.
    The plotted metric is the relative uplift in average recommendation-induced
    orders per user for standard SAC and SAC with AdaGamma. Overall discount
    intensity was matched across arms.}
    \label{fig:JD_res}
    \vspace{-10pt}
\end{wrapfigure}

\subsection{Analysis of Learned $\gamma_\phi(s)$}
\label{sec:gamma_analysis}

\textbf{Statewise structure.}
On \texttt{Ant-v4}, SAC-AdaGamma learns a non-constant $\gamma_\phi(s)$
that varies with torso height and planar velocity
(Figure~\ref{fig:heatmap} in Appendix~\ref{app:gamma_heatmap}).
Together with the matched fixed-$\gamma$ ablation in
Table~\ref{tab:ab_fixed_gamma}, this suggests that the gains are not explained
solely by the learned mean discount.

\textbf{Cross-algorithm consistency.}
Mean learned discounts of SAC-AdaGamma and PPO-AdaGamma are closely
aligned on most tasks
(Table~\ref{tab:gamma_cross_algorithm} in Appendix~\ref{app:consistency}).

\subsection{Empirical verification of TD-target collapse}
\label{sec:td_exp}
Section~\ref{sec:collapse} argues that scoring $\gamma_\phi$ by the squared TD
residual admits a degenerate shortcut: reducing the bootstrap coefficient can
weaken the target's dependence on downstream value estimates and lower the TD
loss without learning meaningful state-dependent temporal structure.
We test this claim on \texttt{SafetyPointGoal1-v0} under SAC, comparing
return-consistency (AdaGamma-RC) against a squared-TD $\gamma$ objective
(AdaGamma-TD) with the same backbone, $\gamma$-range $[0.9,0.999]$,
$100$k-step $\gamma$-warmup, and 5 seeds. Both arms disable adaptive
$\bar{\gamma}$, so the reference-discount EMA cannot serve as an anti-collapse
anchor. Collapse is measured by the distribution of $\gamma_\phi(s)$ on a probe
set of $2{,}048$ states frozen at the end of warmup.

Under the squared-TD objective, $\gamma_\phi$ degenerates to a near-constant
discount (mean ${\approx}0.911$, $\mathrm{Std}_s[\gamma_\phi(s)]{\approx}0.0014$),
whereas return-consistency retains a non-degenerate state-dependent distribution
(mean ${\approx}0.971$, $\mathrm{Std}_s[\gamma_\phi(s)]{\approx}0.022$---roughly $15{\times}$ larger
cross-state dispersion).
Evaluation return drops modestly under TD (${\approx}26.4$ vs ${\approx}27.2$),
consistent with a discount-module degeneracy rather than total unlearnability on
this task.
Figure~\ref{fig:collapse} and Table~\ref{tab:collapse} in
Appendix~\ref{app:collapse_table} summarize the distributions.
The same qualitative failure appears for the cross-validated TD
baseline in Table~\ref{tab:gamma_sacs}
(mean $\gamma{=}0.9000$ on all three reported environments).

\subsection{Ablation Studies}
\label{sec:ablations}

\textbf{Leave-one-out on the $\gamma$ objective.}
To isolate which term in~(\ref{eq:full_gamma_loss}) prevents the
TD-target failure of Section~\ref{sec:td_exp}, we expand that collapse
ablation into a leave-one-out on the three auxiliary penalties, keeping the
return-consistency objective fixed.
Concretely, we train RC with the full objective; RC with each of
$\lambda_{\mathrm{dev}}$, $\lambda_{\mathrm{var}}$, and
$\lambda_{\mathrm{bound}}$ set to zero in turn; and RC with all three
disabled.
The two squared-TD arms of Section~\ref{sec:td_exp} remain as collapse
controls, so the comparison holds the backbone, $\gamma$-range, warmup,
and seeds fixed and varies only the $\gamma$ objective and which auxiliaries
are present.
Default coefficients are
$\lambda_{\mathrm{dev}}{=}0.005$,
$\lambda_{\mathrm{var}}{=}0.012$,
$\lambda_{\mathrm{bound}}{=}0.05$; an ablated arm sets the corresponding
$\lambda$ to $0$.
All variants use seeds $42$--$46$, disable adaptive $\bar{\gamma}$, and
follow the frozen-probe protocol of Section~\ref{sec:td_exp}
(\texttt{SafetyPointGoal1-v0}, SAC, $300$k steps).
Shared hyperparameters are listed in Table~\ref{tab:collapse-hparams}.
Every return-consistency variant remains non-degenerate: dropping
$\lambda_{\mathrm{dev}}$, $\lambda_{\mathrm{var}}$, or
$\lambda_{\mathrm{bound}}$ individually, or disabling all three, yields
$\mathrm{Std}_s[\gamma_\phi(s)]\in[0.0215,0.0238]$, matching the full objective
($0.0216{\pm}0.0013$).
RC-only is slightly \emph{more} dispersed than the full loss, as expected
because $\lambda_{\mathrm{var}}$ penalizes cross-state variance rather than
creating it.
By contrast, replacing $L_\gamma^{\mathrm{RC}}$ with $\delta^2$ under the
\emph{same} regularizers still collapses to a near-constant discount
($\mathrm{Std}_s[\gamma_\phi(s)]{\approx}0.0014$); without regularizers it pins to a
point mass at $\gamma_{\min}$.
The auxiliaries therefore modulate the distribution---$\lambda_{\mathrm{var}}$
slightly compresses it, $\lambda_{\mathrm{bound}}$ reduces mass near the
floor---but they are not what prevents collapse.
The failure mode is the $\delta^2$ objective, and return-consistency is the
load-bearing term in~(\ref{eq:full_gamma_loss}).
Table~\ref{tab:gamma_objective_ablation} and
Figure~\ref{fig:gamma_objective_ablation} in
Appendix~\ref{app:gamma_objective_ablation} summarize the comparison;
Figure~\ref{fig:gamma_objective_diagnostics} reports mean-$\gamma$ and
floor-mass diagnostics.

\textbf{Training objective comparison.}
Section~\ref{sec:safety_results} shows that AdaGamma consistently outperforms the SAC and PPO baselines, and Section~\ref{sec:gamma_analysis} analyzes the learned \(\gamma_\phi(s)\). Here we focus on the SAC variants. Table~\ref{tab:gamma_sacs} in Appendix~\ref{app:training_objective} reports the mean learned \(\gamma_\phi(s)\). The fixed baseline stays at \(\gamma=0.99\), whereas the cross-validated baseline collapses to the minimum candidate in all environments, matching its weak performance. The uncertainty-based baseline chooses less extreme but still global discounts. In contrast, AdaGamma learns task-dependent average discounts, indicating that its gains come from adapting temporal weighting rather than uniformly increasing or decreasing \(\gamma\).

\textbf{Comparison to fixed-$\gamma$.}
Table~\ref{tab:gamma_cross_algorithm} compares the mean learned \(\gamma_{\phi}(s)\) under SAC and PPO. To test whether the gains simply reflect a better fixed discount, we set fixed \(\gamma\) to the corresponding learned mean and re-evaluate SAC and PPO on \texttt{SafetyPointGoal1-v0}, \texttt{Humanoid-v4}, and \texttt{Ant-v4}; results are given in Table~\ref{tab:ab_fixed_gamma} of Appendix~\ref{app:fixed_gamma}. In all cases, this matched fixed discount still underperforms AdaGamma. We also expand the fixed-\(\gamma\) grid on these tasks in Table~\ref{tab:grid_search} of Appendix~\ref{app:fixed_gamma}, where AdaGamma remains superior. Thus, the improvement stems from state-dependent discounting rather than a better global \(\gamma\).

\textbf{Network architecture.}
We vary the hidden dimension of the gamma network over $\{8, 16, 32, 64, 128, 256, 512\}$ on \texttt{Ant-v4}. Performance improves with capacity and saturates around 256 hidden units (Figure~\ref{fig:hidden_dim} in Appendix~\ref{app:hidden_dim}). Beyond this point, larger networks add computation with little gain, indicating that AdaGamma has low practical overhead.

\textbf{Return horizon $n$.}
We vary the multi-step return length \(n\in\{5,10,20\}\) in the return-consistency objective on \texttt{SafetyPointGoal1-v0} for both SAC and PPO. As shown in Table~\ref{tab:return_horizon} of Appendix~\ref{app:return_horizon}, \(n=5\) gives the best overall reward--cost tradeoff, while longer horizons reduce stability, particularly for PPO. This suggests that the return-consistency target should balance temporal coverage against estimation noise; in this task, \(n=5\) provides the best tradeoff.

\paragraph{Additional experiments.}
Results on classic control tasks are reported in Appendix~\ref{app:classic_results}. We also extend AdaGamma to DDPG~\citep{lillicrap2019continuouscontroldeepreinforcement} and TRPO~\citep{schulman2017trustregionpolicyoptimization}, with results on \texttt{SafetyPointGoal1-v0}, \texttt{Humanoid-v4}, and \texttt{Ant-v4} reported in Appendix~\ref{app:trpo_ddpg}.

\section{Conclusion}
\label{sec:conclusion}

\noindent
AdaGamma provides a simple and unified way to learn state-dependent discount factors in deep reinforcement learning. By interfacing only through the bootstrapped value target, it integrates naturally with both SAC and PPO, while the return-consistency objective prevents the degenerate behavior that arises from naive discount learning. On the theory side, we establish operator-level convergence properties for soft policy iteration with a fixed state-dependent discount function under tabular assumptions, and characterize the value discrepancy incurred. Empirically, AdaGamma improves reward--cost tradeoffs in safety-constrained settings, scales to high-dimensional continuous control, remains competitive on classic control benchmarks, and shows stable discount patterns across algorithms. Its positive results in a four-week JD Logistics A/B test further suggest that state-adaptive temporal weighting is not only algorithmically useful, but also practical in real-world sequential decision making.

\section*{AI Use Disclosure}
In this work, we used generative AI tools for assisting with the verification of mathematical proofs. We have not used generative AI tools for any other tasks with required disclosure, and the remaining required disclosure tasks are not applicable to this work. Additionally, we used generative AI tools for editing the paper to improve readability and for editing software code. We have reviewed all AI-assisted work. In particular, the AI-assisted code modifications were verified and tested for correctness by 2 authors. We take responsibility for the final content of this work, including text, claims, or artifacts produced with the aid of generative AI.

\newpage

\bibliography{iclr2027_conference}
\bibliographystyle{iclr2027_conference}

\newpage
\appendix
\section{Related Work}
\label{app:related}

\textbf{Fixed discounting and discount as regularization.}
Most deep RL methods use a fixed discount factor
\citep{sutton1998reinforcement,mnih2015human,haarnoja2018soft,schulman2017proximal},
despite the fact that $\gamma$ strongly influences the effective planning horizon, variance, and
optimization behavior
\citep{francois2016discount,amit2020discount,naik2019discountedreinforcementlearningoptimization,hu2022rolediscountfactoroffline}. 
A growing line of work interprets discounting not only as a preference parameter, but also as a form of
regularization
\citep{amit2020discount,rathnam2024rethinking}. In particular, \citet{amit2020discount} show that, for
several TD-learning methods, reducing $\gamma$ can be viewed as introducing an explicit regularization
effect, while \citet{rathnam2024rethinking} show that globally shrinking $\gamma$ may induce an
undesirable non-uniform prior over transition dynamics. Related analyses in offline RL further show that
$\gamma$ jointly controls regularization and pessimism, affecting the trade-off between policy quality
and data efficiency
\citep{hu2022rolediscountfactoroffline,foster2022offlinereinforcementlearningfundamental,zhan2022offlinereinforcementlearningrealizability}. 
These results provide part of the motivation for adaptive discounting, although they do not directly
address how to learn a state-dependent discount function in deep actor--critic training.

\textbf{Flexible discounting: schedules and non-exponential variants.}
Another line of work studies discounting schemes that vary over training, over time, or beyond the
standard exponential form. \citet{francois2016discount} discuss dynamic discount strategies, and
\citet{hou2021improvement} propose increasing discount factors over time steps to alleviate the
underweighting of distant rewards. More generally,
\citet{schultheis2022reinforcementlearningnonexponentialdiscounting} study RL with non-exponential
discount functions, and related work broadens the space of generalized return criteria
\citep{10.1609/aaai.v33i01.33017949,naik2019discountedreinforcementlearningoptimization}. 
These approaches enlarge the class of temporal preference models, but the discount rule is typically
global, time-indexed, or fixed by design. Our setting is different: we focus on learned
\emph{state-conditioned} discounting within a conventional bootstrapped actor--critic pipeline, where the
main challenge is not defining a richer return criterion, but ensuring that the learned discount does not
destabilize TD-based training.

\textbf{State-dependent discounting and generalized RL objectives.}
A more expressive formulation allows the discount factor to depend on the state or state-action pair.
\citet{wei2011markov} study MDPs with state-dependent discount factors and establish general existence and
characterization results. \citet{yoshida2013reinforcement} introduce state-dependent discounting in
model-free RL and derive an ExQ-learning algorithm with convergence guarantees. From a broader
decision-theoretic perspective, \citet{pitis2019rethinkingdiscountfactorreinforcement} argue that fixed
discounting is only one special case in a larger class of sequential preference models, while
\citet{naik2019discountedreinforcementlearningoptimization} emphasize that discounting is a substantive
modeling choice rather than a mere numerical convenience. These works establish that variable discounting
is mathematically meaningful and behaviorally expressive, but they do not directly provide a practical
implementation recipe for modern deep actor--critic methods with function approximation.

\textbf{Adaptive discounting in deep RL.}
Several recent works study adaptive discounting more directly in deep RL.
\citet{kim2022adaptive} propose an uncertainty- or advantage-driven adjustment rule for PPO and SAC,
while \citet{gu2022proximal} introduce a PPO variant with policy-feedback-dependent discounting.
Related approaches also adapt horizon-related quantities through dynamic schedules, uncertainty signals,
or return-design heuristics
\citep{francois2016discount,kim2022adaptive,gu2022proximal}. 
Compared with these methods, our approach learns an explicit state-conditioned neural module
$\gamma_\phi(s)$ and emphasizes a different algorithmic issue: the collapse behavior that can arise when
such a module is trained naively through bootstrapped TD objectives. Our return-consistency objective is
designed specifically to regularize this failure mode.

\textbf{Long horizons, statistical fragility, and temporal abstraction.}
Large discount factors require reliable long-horizon estimation, which becomes difficult when data are
noisy, sparse, or poorly mixed
\citep{francois2016discount,naik2019discountedreinforcementlearningoptimization,metelli2023tale}. 
More generally, sample-complexity and offline-RL analyses show that discounted problems become
statistically harder as the effective horizon grows or coverage deteriorates
\citep{foster2022offlinereinforcementlearningfundamental,hu2022rolediscountfactoroffline,jia2024offline1,zhan2022offlinereinforcementlearningrealizability}. 
In particular, \citet{metelli2023tale} study the $\gamma$-discounted mean estimation problem and derive
lower bounds depending on both the discount factor and the mixing properties of the underlying Markov
chain. This helps explain why a globally large $\gamma$ can be statistically fragile, and why adapting
bootstrap strength across states may be beneficial.

State-dependent discounting is also related to temporal abstraction and termination. In the options
framework, termination determines the duration of temporally extended behavior and thus implicitly
controls an effective planning horizon
\citep{precup2000eligibility,bacon2017optioncritic}. More broadly, temporally extended actions modulate
the scale of credit assignment and planning, and recent work shows benefits from decoupling behavior and
target termination conditions
\citep{harutyunyan2019termination}. Although our method does not introduce explicit options or
termination policies, the product-of-gammas structure in our modified GAE recursion plays a related
conceptual role: low-$\gamma$ regions shorten the effective credit-assignment horizon, while high-$\gamma$
regions preserve longer-range dependencies. In this sense, our approach is complementary to
temporal-abstraction methods rather than a replacement for them.

\section{Proofs}
\label{app:proofs}

\subsection{Proof of Proposition~\ref{prop:gae}}
\label{app:proof_gae}
\begin{restatable}{proposition}{propgae}
(Product-of-Gammas Weighting)
\label{prop:gae}
Under state-dependent discounting, for a nonterminal trajectory segment, the GAE advantage estimate expands as
\begin{equation}
\hat{A}_t
=
\delta_t
+
\sum_{l=1}^{T-1-t}
\left(
\prod_{k=0}^{l-1}
\gamma_\phi(s_{t+k})
\right)
\lambda^l
\delta_{t+l}.
\label{eq:product_gamma_gae}
\end{equation}
With terminal masks, each factor $\gamma_\phi(s_{t+k})$ is replaced by
$\gamma_\phi(s_{t+k})(1-d_{t+k})$.
\end{restatable}
\begin{proof}
We prove the result by induction on \(T-1-t\). The displayed expansion is for
a nonterminal trajectory segment; terminal masks simply truncate the recursion
because any factor \(1-d_{t+k}=0\) eliminates all subsequent terms.

\textbf{Base case.}
For \(t=T-1\), we have \(\hat{A}_{T-1}=\delta_{T-1}\).
The expansion also gives \(\hat{A}_{T-1}=\delta_{T-1}\), since
the summation is empty. Thus the claim holds.

\textbf{Inductive step.}
Assume that the expansion holds for \(\hat{A}_{t+1}\), i.e.,
\begin{equation}
\hat{A}_{t+1}
=
\delta_{t+1}
+
\sum_{l=1}^{T-2-t}
\left(
\prod_{k=0}^{l-1}\gamma_\phi(s_{t+1+k})
\right)
\lambda^l
\delta_{t+1+l}.
\end{equation}
Then, using the recursion
\(\hat{A}_t=\delta_t+\gamma_\phi(s_t)\lambda\hat{A}_{t+1}\), we obtain
\begin{align}
\hat{A}_t
&= \delta_t + \gamma_\phi(s_t)\lambda\,\hat{A}_{t+1} \\
&= \delta_t
+ \gamma_\phi(s_t)\lambda
\left[
\delta_{t+1}
+
\sum_{l=1}^{T-2-t}
\left(
\prod_{k=0}^{l-1}\gamma_\phi(s_{t+1+k})
\right)
\lambda^l
\delta_{t+1+l}
\right] \\
&= \delta_t
+ \gamma_\phi(s_t)\lambda\,\delta_{t+1}
+
\sum_{l=1}^{T-2-t}
\gamma_\phi(s_t)
\left(
\prod_{k=0}^{l-1}\gamma_\phi(s_{t+1+k})
\right)
\lambda^{l+1}
\delta_{t+1+l} \\
&= \delta_t
+
\left(
\prod_{k=0}^{0}\gamma_\phi(s_{t+k})
\right)
\lambda
\delta_{t+1}
+
\sum_{l=1}^{T-2-t}
\left(
\prod_{k=0}^{l}\gamma_\phi(s_{t+k})
\right)
\lambda^{l+1}
\delta_{t+1+l}.
\end{align}
Now let \(m=l+1\) in the last summation. Then \(m\) ranges from
\(2\) to \(T-1-t\), and we have
\begin{align}
\hat{A}_t
&= \delta_t
+
\left(
\prod_{k=0}^{0}\gamma_\phi(s_{t+k})
\right)
\lambda
\delta_{t+1}
+
\sum_{m=2}^{T-1-t}
\left(
\prod_{k=0}^{m-1}\gamma_\phi(s_{t+k})
\right)
\lambda^m
\delta_{t+m} \\
&= \delta_t
+
\sum_{m=1}^{T-1-t}
\left(
\prod_{k=0}^{m-1}\gamma_\phi(s_{t+k})
\right)
\lambda^m
\delta_{t+m}.
\end{align}
This completes the induction.
\end{proof}

\subsection{Proof of Lemma~\ref{lem:policy_eval}}
\label{app:proof_eval}
\policyeval*

\begin{proof}
For any $Q,Q'$ and $(s,a)$, the reward and entropy terms cancel, so
\[
\begin{aligned}
\left|
(\T_{\gamma}^\pi Q)(s,a)
-
(\T_{\gamma}^\pi Q')(s,a)
\right|
&\le
\gamma(s)\,
\mathbb{E}_{s' \sim P(\cdot | s,a)}
\left[
\left|V_Q^\pi(s')-V_{Q'}^\pi(s')\right|
\right] \\
&\le
\gamma(s)\,
\mathbb{E}_{s' \sim P(\cdot | s,a)}
\left[
\sum_{a'}\pi(a'|s')|Q(s',a')-Q'(s',a')|
\right] \\
&\le
\beta\|Q-Q'\|_\infty .
\end{aligned}
\]
Taking the supremum over $(s,a)$ proves~(\ref{eq:policy_eval_contraction}). The existence and uniqueness of
$Q_\gamma^\pi$, together with convergence from any initialization, then
follow from the Banach fixed-point theorem.
\end{proof}

\subsection{Proof of Proposition~\ref{prop:revised-n-step}}
\label{app:proof_nstep}
\begin{restatable}{proposition}{propnstep}
(Multi-step contraction and trajectory weighting)
\label{prop:revised-n-step}
Under Assumption~\ref{assumption}, for every integer
$n\ge1$, the $n$-fold Bellman operator $(\T_{\gamma}^\pi)^n$ is a
$\beta^n$-contraction:
\begin{equation}
\|(\T_{\gamma}^\pi)^nQ-(\T_{\gamma}^\pi)^nQ'\|_\infty
\le
\beta^n\|Q-Q'\|_\infty .
\label{eq:n_step_contraction}
\end{equation}
Moreover, $\T_{\gamma}^\pi$ and $(\T_{\gamma}^\pi)^n$ have the same unique
fixed point $Q_\gamma^\pi$. Along a trajectory generated by $P$ and $\pi$,
\begin{equation}
\begin{aligned}
\bigl((\T_{\gamma}^\pi)^nQ-(\T_{\gamma}^\pi)^nQ'\bigr)(s_0,a_0) =
\mathbb{E}_\pi
\left[
\left(
\prod_{j=0}^{n-1}\gamma(s_j)
\right)
(Q-Q')(s_n,a_n)
\,\middle|\, s_0,a_0
\right].
\end{aligned}
\label{eq:product_gamma_operator}
\end{equation}
\end{restatable}
\begin{proof}
Applying Lemma~\ref{lem:policy_eval} recursively gives
\[
\begin{aligned}
\|(\T_{\gamma}^\pi)^nQ-(\T_{\gamma}^\pi)^nQ'\|_\infty
&\le
\beta
\|(\T_{\gamma}^\pi)^{n-1}Q
-
(\T_{\gamma}^\pi)^{n-1}Q'\|_\infty \le \cdots \le
\beta^n\|Q-Q'\|_\infty,
\end{aligned}
\]
which proves~(\ref{eq:n_step_contraction}).

Since $Q_\gamma^\pi$ is the fixed point of $\T_{\gamma}^\pi$, we obtain \((\T_{\gamma}^\pi)^nQ_\gamma^\pi=Q_\gamma^\pi\). Because $(\T_{\gamma}^\pi)^n$ is itself a contraction, its fixed point is
unique; hence the one-step and $n$-step operators share the same fixed
point $Q_\gamma^\pi$.

Finally, repeatedly expanding the difference between the two Bellman
backups yields~(\ref{eq:product_gamma_operator}), which makes explicit
the product-of-discounts structure of multi-step propagation.
\end{proof}

\subsection{Proof of Lemma~\ref{lem:soft_improve}}
\label{app:proof_improve}
\begin{restatable}{lemma}{softimprove}
(Exact soft policy improvement)
\label{lem:soft_improve}
Let $\pi$ be any stationary tabular policy, and let $Q^\pi_\gamma$
denote its soft $Q$-function under the fixed discount function
$\gamma(\cdot)$. Define
$\pi_{\mathrm{new}}=\pi_{Q^\pi_\gamma}$. Then
\begin{equation}
Q^{\pi_{\mathrm{new}}}_\gamma(s,a)
\ge
Q^\pi_\gamma(s,a),
\qquad
\forall (s,a)\in\mathcal{S}\times\mathcal{A}.
\label{eq:soft_policy_improvement}
\end{equation}
\end{restatable}
\begin{proof}
For any vector $q\in\mathbb{R}^{|\mathcal{A}|}$, the Gibbs variational
identity gives
\begin{equation}
\alpha
\log
\sum_{a\in\mathcal{A}}
\exp\!\left(\frac{q_a}{\alpha}\right)
=
\max_{\mu\in\Delta(\mathcal{A})}
\sum_{a\in\mathcal{A}}
\mu(a)
\bigl(q_a-\alpha\log\mu(a)\bigr),
\label{eq:proof-gibbs}
\end{equation}
whose unique maximizer is
\[
\mu^\star(a)
=
\frac{\exp(q_a/\alpha)}
{\sum_{b\in\mathcal{A}}\exp(q_b/\alpha)}.
\]

Applying~(\ref{eq:proof-gibbs}) statewise with
$q_a=Q_\gamma^\pi(s,a)$ and using
$\pi_{\mathrm{new}}=\pi_{Q_\gamma^\pi}$ yields
\begin{equation}
V_{Q_\gamma^\pi}^{\pi_{\mathrm{new}}}(s)
\ge
V_{Q_\gamma^\pi}^{\pi}(s),
\qquad
\forall s\in\mathcal{S}.
\label{eq:statewise_soft_improvement}
\end{equation}
Because $\gamma(s)\ge0$, \eqref{eq:statewise_soft_improvement}
implies
\begin{equation}
\T_{\gamma}^{\pi_{\mathrm{new}}}Q_\gamma^\pi
\ge
\T_{\gamma}^\pi Q_\gamma^\pi
=
Q_\gamma^\pi,
\label{eq:bellman_policy_improvement}
\end{equation}
where the equality follows from the fact that $Q_\gamma^\pi$ is the
unique fixed point of $\T_{\gamma}^\pi$.

Moreover, $\T_{\gamma}^{\pi_{\mathrm{new}}}$ is monotone: for any
$Q_1\le Q_2$ componentwise,
\[
\T_{\gamma}^{\pi_{\mathrm{new}}}Q_1
\le
\T_{\gamma}^{\pi_{\mathrm{new}}}Q_2.
\]
Combining this monotonicity with~(\ref{eq:bellman_policy_improvement}) gives
\[
Q_\gamma^\pi
\le
\T_{\gamma}^{\pi_{\mathrm{new}}}Q_\gamma^\pi
\le
\left(\T_{\gamma}^{\pi_{\mathrm{new}}}\right)^2Q_\gamma^\pi
\le
\cdots .
\]

By Lemma~\ref{lem:policy_eval},
$\T_{\gamma}^{\pi_{\mathrm{new}}}$ is a $\beta$-contraction.
Hence its iterates converge to its unique fixed point
$Q_\gamma^{\pi_{\mathrm{new}}}$. Taking the limit in the above
monotone sequence yields
\[
Q_\gamma^\pi
\le
\lim_{\ell\to\infty}
\left(\T_{\gamma}^{\pi_{\mathrm{new}}}\right)^\ell
Q_\gamma^\pi
=
Q_\gamma^{\pi_{\mathrm{new}}},
\]
which proves~(\ref{eq:soft_policy_improvement}).
\end{proof}

\subsection{Proof of Theorem~\ref{thm:revised-soft-policy-iteration}}
\label{app:proof_spi}
\softpi*
\begin{proof}
The log-sum-exp mapping is $1$-Lipschitz with respect to the sup norm.
Together with $\gamma(s)\le\beta<1$, this implies that
$\T_\gamma$ is a $\beta$-contraction. Hence, by the Banach
fixed-point theorem, $\T_\gamma$ admits a unique fixed point,
denoted by $Q_\gamma^\star$.

We first establish monotonicity of the policy-iteration sequence.
By the Gibbs variational identity in~(\ref{eq:proof-gibbs}) and the
definition $\pi_{k+1}=\pi_{Q_k}$,
\begin{equation}
\T_\gamma Q_k = \T_\gamma^{\pi_{k+1}} Q_k \ge\T_\gamma^{\pi_k} Q_k=Q_k,
\label{eq:proof-optimality-step}
\end{equation}
where the last equality follows from
$Q_k=Q_\gamma^{\pi_k}$.

Since $\T_\gamma^{\pi_{k+1}}$ is monotone, repeated application
of~(\ref{eq:proof-optimality-step}) gives
\[
Q_k
\le
\T_\gamma^{\pi_{k+1}}Q_k
\le
\left(\T_\gamma^{\pi_{k+1}}\right)^2Q_k
\le \cdots .
\]
By Lemma~\ref{lem:policy_eval},
$\T_\gamma^{\pi_{k+1}}$ is a $\beta$-contraction, and its iterates
converge to its unique fixed point
$Q_\gamma^{\pi_{k+1}}=Q_{k+1}$. Therefore,
\begin{equation}
Q_k
\le
\T_\gamma Q_k
\le
Q_{k+1}.
\label{eq:proof-policy-sandwich}
\end{equation}
Thus, $\{Q_k\}$ is componentwise nondecreasing.

Next, we establish a uniform upper bound. For any stationary policy $\pi$,
the soft optimality operator dominates the corresponding policy-evaluation
operator:
\begin{equation}
\T_\gamma Q
\ge
\T_\gamma^\pi Q,
\qquad
\forall\, Q.
\label{eq:proof-operator-dominance}
\end{equation}
In particular, setting $Q=Q_\gamma^\pi$ yields
\[
\T_\gamma Q_\gamma^\pi
\ge
\T_\gamma^\pi Q_\gamma^\pi
=
Q_\gamma^\pi.
\]
Since $\T_\gamma$ is both monotone and contractive, iterating it
from $Q_\gamma^\pi$ gives a nondecreasing sequence converging to
$Q_\gamma^\star$. Hence
\[
Q_\gamma^\pi
\le
Q_\gamma^\star
\]
for every stationary tabular policy $\pi$. In particular,
\[
Q_k\le Q_\gamma^\star,
\qquad \forall k.
\]

Because $\mathcal S\times\mathcal A$ is finite, the componentwise
nondecreasing and bounded sequence $\{Q_k\}$ converges in the sup norm
to some $\overline Q$. Taking limits in~(\ref{eq:proof-policy-sandwich}) and using the continuity of
$\T_\gamma$ yields
\[
\overline Q
\le
\T_\gamma\overline Q
\le
\overline Q.
\]
Consequently,
\[
\T_\gamma\overline Q=\overline Q.
\]
By uniqueness of the fixed point of $\T_\gamma$,
\[
\overline Q=Q_\gamma^\star.
\]
Therefore,
\[
Q_k\rightarrow Q_\gamma^\star
\qquad\text{in the sup norm}.
\]

Finally, since $\alpha>0$, the statewise softmax mapping
$Q\mapsto\pi_Q$ is continuous. Hence
\[
\pi_{k+1}
=
\pi_{Q_k}
\rightarrow
\pi_{Q_\gamma^\star}
=
\pi^\star.
\]
The dominance argument above also establishes~(\ref{eq:revised-spi-optimality}).
\end{proof}

\subsection{Proof of Theorem~\ref{thm:revised-sensitivity}}
\label{app:proof_sens}
\begin{restatable}{theorem}{lipsens}
(Lipschitz sensitivity to the discount function)
\label{thm:revised-sensitivity}
Under Assumption~\ref{assumption}, fix a stationary policy $\pi$.
For $i\in\{1,2\}$, let
\(\gamma_i:\mathcal S\to[0,\beta_i]\), where \(\beta_i<1\). Then
\begin{equation}
\|Q_{\gamma_1}^\pi-Q_{\gamma_2}^\pi\|_\infty
\le
\frac{
C\,\|\gamma_1-\gamma_2\|_\infty
}{
(1-\beta_1)(1-\beta_2)
},
\label{eq:revised-sensitivity}
\end{equation}
where $C:=R+\alpha\log m$. In particular, if
$\gamma_2(s)\equiv\bar{\gamma}$ for some $\bar{\gamma}\in[0,1)$,
\begin{equation}
\|Q_{\gamma_1}^\pi-Q_{\bar{\gamma}}^\pi\|_\infty
\le
\frac{
(R+\alpha\log m)
\max_{s\in\mathcal S}|\gamma_1(s)-\bar{\gamma}|
}{
(1-\beta_1)(1-\bar{\gamma})
}.
\label{eq:revised-fixed-scalar-upper}
\end{equation}
\end{restatable}
\begin{proof}
Let
\[
H_\pi(s)
:=
-\sum_{a\in\mathcal A}
\pi(a\mid s)\log\pi(a\mid s)
\]
denote the entropy of $\pi(\cdot\mid s)$. Since $\mathcal A$ is finite,
\begin{equation}
0\le H_\pi(s)\le\log m,
\qquad
\forall s\in\mathcal S.
\label{eq:proof-entropy-bound}
\end{equation}

We first bound the soft value function associated with $Q_{\gamma_2}^\pi$.
Since $Q_{\gamma_2}^\pi$ is the fixed point of
$T_{\gamma_2}^\pi$, the Bellman equation and
$\gamma_2(s)\le\beta_2$ imply
\begin{align}
\|Q_{\gamma_2}^\pi\|_\infty
&\le
R+\beta_2\|V_{Q_{\gamma_2}^\pi}^\pi\|_\infty \le
R+\beta_2
\left(
\|Q_{\gamma_2}^\pi\|_\infty+\alpha\log m
\right).
\end{align}
Hence
\begin{equation}
\|Q_{\gamma_2}^\pi\|_\infty
\le
\frac{
R+\alpha\beta_2\log m
}{
1-\beta_2
}.
\label{eq:proof-q-bound}
\end{equation}
Consequently,
\begin{equation}
\|V_{Q_{\gamma_2}^\pi}^\pi\|_\infty
\le
\|Q_{\gamma_2}^\pi\|_\infty+\alpha\log m \le
\frac{R+\alpha\log m}{1-\beta_2}
=
\frac{C}{1-\beta_2}.
\label{eq:proof-v-bound}
\end{equation}

We now compare the two fixed-point equations. For every $(s,a)$,
adding and subtracting the term
$\gamma_1(s)\mathbb E[V_{Q_{\gamma_2}^\pi}^\pi(s')]$ yields
\begin{equation}
\label{eq:proof-difference}
Q_{\gamma_1}^\pi(s,a)-Q_{\gamma_2}^\pi(s,a)=
\gamma_1(s)
\mathbb E_{s'\sim P(\cdot\mid s,a)}
\left[
V_{Q_{\gamma_1}^\pi}^\pi(s')-V_{Q_{\gamma_2}^\pi}^\pi(s')
\right] +
\bigl(\gamma_1(s)-\gamma_2(s)\bigr)
\mathbb E_{s'\sim P(\cdot\mid s,a)}
\left[
V_{Q_{\gamma_2}^\pi}^\pi(s')
\right].
\end{equation}

Because the policy $\pi$ is fixed, the entropy terms cancel in
$V_{{Q_{\gamma_1}^\pi}}^\pi-V_{{Q_{\gamma_2}^\pi}}^\pi$. Therefore,
\begin{equation}
\|V_{{Q_{\gamma_1}^\pi}}^\pi-V_{{Q_{\gamma_2}^\pi}}^\pi\|_\infty
\le
\|{Q_{\gamma_1}^\pi}-{Q_{\gamma_2}^\pi}\|_\infty.
\label{eq:proof-v-difference}
\end{equation}

Taking sup norms in~(\ref{eq:proof-difference}) and using~(\ref{eq:proof-v-bound}), \eqref{eq:proof-v-difference}, we obtain
\begin{equation}
\|{Q_{\gamma_1}^\pi}-{Q_{\gamma_2}^\pi}\|_\infty
\le
\beta_1\|{Q_{\gamma_1}^\pi}-{Q_{\gamma_2}^\pi}\|_\infty
+
\|\gamma_1-\gamma_2\|_\infty
\frac{C}{1-\beta_2}.
\end{equation}
Rearranging gives
\[
\|{Q_{\gamma_1}^\pi}-{Q_{\gamma_2}^\pi}\|_\infty
\le
\frac{
C\,\|\gamma_1-\gamma_2\|_\infty
}{
(1-\beta_1)(1-\beta_2)
},
\]
which proves~(\ref{eq:revised-sensitivity}).

Finally, when $\gamma_2(s)\equiv\bar{\gamma}$, we may take
$\beta_2=\bar{\gamma}$, yielding~(\ref{eq:revised-fixed-scalar-upper}).
\end{proof}

\subsection{Proof of Corollary~\ref{cor:revised-tracking}}
\label{app:proof_track}
\begin{restatable}{corollary}{cortracking}
(Conditional tracking under a varying discount)
\label{cor:revised-tracking}
Under Assumption~\ref{assumption}, fix a stationary policy $\pi$ and let
$\gamma_k:\mathcal S\to[0,\beta]$ with $\beta<1$. Define
$Q_k^\star:=Q_{\gamma_k}^\pi$,
$Q_{k+1}=(\mathcal T_{\gamma_k}^\pi)^nQ_k$ and
$e_k:=\|Q_k-Q_k^\star\|_\infty$. Then
\begin{equation}
e_{k+1}
\le
\beta^n e_k
+
\frac{R+\alpha\log m}{(1-\beta)^2}
\|\gamma_{k+1}-\gamma_k\|_\infty.
\label{eq:revised-one-step-tracking}
\end{equation}
Hence if $\|\gamma_{k+1}-\gamma_k\|_\infty\to0$, then $e_k\to0$.
If $\sup_k\|\gamma_{k+1}-\gamma_k\|_\infty\le\eta$, then
\begin{equation}
\limsup_{k\to\infty}e_k
\le
\frac{(R+\alpha\log m)\eta}
{(1-\beta)^2(1-\beta^n)}.
\label{eq:revised-steady-tracking}
\end{equation}
\end{restatable}
\begin{proof}
Since $Q_k^\star=Q_{\gamma_k}^\pi$ is the unique fixed point of
$\T_{\gamma_k}^\pi$, it is also a fixed point of its $n$-fold
iterate \((\T_{\gamma_k}^\pi)^n Q_k^\star=Q_k^\star\). Therefore,
\begin{align}
e_{k+1}
&=
\|Q_{k+1}-Q_{k+1}^\star\|_\infty \notag\\
&=
\|(\T_{\gamma_k}^\pi)^n Q_k-Q_{k+1}^\star\|_\infty \notag\\
&\le
\|(\T_{\gamma_k}^\pi)^n Q_k-Q_k^\star\|_\infty
+
\|Q_k^\star-Q_{k+1}^\star\|_\infty \notag\\
&=
\|(\T_{\gamma_k}^\pi)^n Q_k
-
(\T_{\gamma_k}^\pi)^n Q_k^\star\|_\infty
+
\|Q_k^\star-Q_{k+1}^\star\|_\infty.
\end{align}
By Proposition~\ref{prop:revised-n-step},
$(\T_{\gamma_k}^\pi)^n$ is a $\beta^n$-contraction. Hence
\begin{equation}
e_{k+1}
\le
\beta^n e_k
+
\|Q_{k+1}^\star-Q_k^\star\|_\infty.
\label{eq:proof-tracking-basic}
\end{equation}

Next, by Theorem~\ref{thm:revised-sensitivity}, we have
\[
\|Q_{k+1}^\star-Q_k^\star\|_\infty
\le
\frac{R+\alpha\log m}{(1-\beta)^2}
\|\gamma_{k+1}-\gamma_k\|_\infty.
\]
Substituting this bound into~(\ref{eq:proof-tracking-basic}) yields~(\ref{eq:revised-one-step-tracking}).

Let $\rho:=\beta^n<1$ and $\delta_k:=\frac{R+\alpha\log m}{(1-\beta)^2}\|\gamma_{k+1}-\gamma_k\|_\infty$, then the tracking recursion can be written as
\[
e_{k+1}\le \rho e_k+\delta_k.
\]
Repeated substitution gives, for every $K\ge1$,
\begin{equation}
e_K
\le
\rho^K e_0
+
\sum_{j=0}^{K-1}\rho^{K-1-j}\delta_j.
\label{eq:proof-tracking-expanded}
\end{equation}

If $\|\gamma_{k+1}-\gamma_k\|_\infty\to0$, then
$\delta_k\to0$. Since $\rho<1$, the convolution of the vanishing
sequence $\{\delta_k\}$ with the geometrically decaying kernel in
\eqref{eq:proof-tracking-expanded} also vanishes, and therefore
$e_k\to0$.

Finally, suppose that $\sup_k\|\gamma_{k+1}-\gamma_k\|_\infty\le\eta$, then
\[
\delta_k
\le
D:=
\frac{(R+\alpha\log m)\eta}{(1-\beta)^2}.
\]
Using~(\ref{eq:proof-tracking-expanded}),
\[
e_K
\le
\rho^K e_0
+
D\sum_{j=0}^{K-1}\rho^{K-1-j}
=
\rho^K e_0
+
D\frac{1-\rho^K}{1-\rho}.
\]
Taking the limit superior and recalling that $\rho=\beta^n$ gives
\[
\limsup_{K\to\infty}e_K
\le
\frac{D}{1-\rho}
=
\frac{(R+\alpha\log m)\eta}
{(1-\beta)^2(1-\beta^n)},
\]
which proves~(\ref{eq:revised-steady-tracking}).
\end{proof}

\section{Details of Uncertainty-rule adaptive-$\gamma$ baseline}
\label{app:uncertainty_rule}
Inspired by \citet{kim2022adaptive}, who adapt the discount factor based on value-estimation uncertainty measured via ensemble disagreement, we implement an uncertainty-driven baseline within both the SAC and PPO frameworks.
In SAC, we treat the absolute difference between the two Q-heads, $|Q_{\theta_1}(s,a) - Q_{\theta_2}(s,a)|$, as a lightweight uncertainty proxy.
In PPO, which lacks a twin-critic architecture by default, we introduce an auxiliary value network $V_{\phi_2}$ alongside the primary value network $V_{\phi_1}$ and use their disagreement $|V_{\phi_1}(s) - V_{\phi_2}(s)|$ as the analogous uncertainty signal.
Crucially, the auxiliary $V_{\phi_2}$ is used \emph{solely} for computing the per-state discount factor; it does not participate in PPO's standard operations such as advantage estimation (GAE) or the policy gradient objective, which continue to rely exclusively on the primary $V_{\phi_1}$.
In both cases, a per-state discount is computed via
\begin{equation}
  \gamma(s)
  =
  \gamma_{\min}
  +
  (\gamma_{\max}-\gamma_{\min})
  \exp\!\bigl(-\eta\,\beta\,d(s)\bigr),
  \label{eq:uncertainty_gamma}
\end{equation}
where $d(s)\ge0$ denotes the critic disagreement, $\beta>0$ is a learnable scaling parameter, and $\eta>0$ is the uncertainty scale.
The intuition follows~\citet{kim2022adaptive}: when the two critics disagree strongly (high uncertainty), the discount is pushed toward $\gamma_{\min}$ to shorten the effective planning horizon and reduce reliance on unreliable bootstrapped targets; when they agree (low uncertainty), the discount approaches $\gamma_{\max}$ to enable longer-horizon credit assignment.
Compared with the original formulation in~\citet{kim2022adaptive}, which uses a larger ensemble and a different scaling rule, our variant is adapted to each algorithm's native architecture---reusing SAC's existing twin Q-networks and adding a minimal auxiliary value network for PPO---and grants the baseline additional flexibility through the learnable $\beta$.
This provides a competitive reference point that already incorporates state-dependent temporal adaptation, against which we isolate the additional benefit of AdaGamma's return-consistency training objective.

\section{Reward Calculation in the JD Logistics Deployment}
\label{app:jd_reward}

This section gives the detailed calculation of the reward components used in the
JD Logistics deployment. Let $\mathrm{order}_t$ denote the per-user order volume
in the current cycle, and let $\mathrm{order}_{t-1}$ denote the corresponding
quantity in the previous cycle. Let $\mathrm{discount}_t$ be the realized discount
under the selected coupon action, and let
$\mathrm{discount}^{\mathrm{med}}$ denote the medium-tier reference discount.

The raw absolute order-growth reward is
\begin{equation}
\label{eq:jd_abs_reward}
  r^{\mathrm{abs}}_t
  =
  a\left(\mathrm{order}_t-\mathrm{order}_{t-1}\right),
\end{equation}
where $a$ is a scaling coefficient. This term directly rewards an increase in
per-user order volume relative to the previous cycle.

The raw relative order-growth reward is
\begin{equation}
\label{eq:jd_rel_reward}
  r^{\mathrm{rel}}_t
  =
  b\,
  \frac{\mathrm{order}_t-\mathrm{order}_{t-1}}
       {\mathrm{order}_{t-1}+\epsilon},
\end{equation}
where $b$ is a scaling coefficient and $\epsilon>0$ is a small constant used for
numerical stability. This term captures the proportional growth in order volume
and is complementary to the absolute-growth term.

The raw discount-cost penalty is
\begin{equation}
\label{eq:jd_pen_reward}
  r^{\mathrm{pen}}_t
  =
  -c\max\left\{\mathrm{discount}_t-\mathrm{discount}^{\mathrm{med}},\,0\right\},
\end{equation}
where $c$ is a scaling coefficient. This term penalizes the additional discount
cost relative to the medium-tier reference. The penalty is applied only when the
selected action corresponds to a discount deeper than the medium-tier reference.

Because the three raw quantities have different units and numerical ranges, we
independently re-scale them before computing the final reward:
\begin{equation}
\label{eq:jd_rescale}
  \tilde r^{k}_t
  =
  \mathrm{Rescale}_k\!\left(r^{k}_t\right),
  \qquad
  k\in\{\mathrm{abs},\mathrm{rel},\mathrm{pen}\}.
\end{equation}
The re-scaled components are then combined as
\begin{equation}
\label{eq:jd_reward_appendix}
  r_t
  =
  \alpha_1\,\tilde r^{\mathrm{abs}}_t
  +
  \alpha_2\,\tilde r^{\mathrm{rel}}_t
  +
  \alpha_3\,\tilde r^{\mathrm{pen}}_t .
\end{equation}
Here $\alpha_1,\alpha_2,\alpha_3$ determine the relative importance of absolute
growth, relative growth, and discount cost after all three components have been
mapped to a shared scale.

\section{Cross-validated TD baseline}
\label{app:cv_loss}
As a comparator, we split a batch into halves $A$ and $B$, update the
critic on $A$, and train $\gamma_\phi$ on the squared residual of $B$:
\begin{equation}
  \mathcal{L}_\gamma^{\mathrm{CV}}(\phi)
    = \mathbb{E}_{(s_t,a_t)\in B}\bigl[(\delta_t^{(A)})^2\bigr],
  \label{eq:cv-loss}
\end{equation}
where $\delta_t^{(A)}$ is the SAC or PPO one-step residual after the
$A$-split update. Because that critic was not optimized on $B$,
shrinking $\gamma$ does not automatically reduce the loss; empirically
this baseline still collapses (Table~\ref{tab:gamma_sacs}).

\section{Experiment Results}
\label{app:experiment_results}

\subsection{TD-target collapse: numerical summary}
\label{app:collapse_table}
Both collapse arms disable adaptive $\bar{\gamma}$. Statistics are means over
the final $10\%$ of training, then mean~$\pm$\,standard deviation across 5
seeds, on a frozen 2048-state probe set.
A leave-one-out on the three auxiliary terms of~(\ref{eq:full_gamma_loss}) is reported in
Section~\ref{sec:ablations}, with numerical summaries in
Appendix~\ref{app:gamma_objective_ablation}.
\begin{table}[ht]
  \centering
  \caption{TD-target collapse ablation on \texttt{SafetyPointGoal1-v0}
  (SAC, 300k steps, 5 seeds).}
  \label{tab:collapse}
  \begin{tabular}{lcccc}
    \toprule
    Method & mean~$\gamma$ & \textbf{Std~$\gamma$} & Frac$_{\mathrm{floor}}$ & Return \\
    \midrule
    AdaGamma-RC
      & $0.971\!\pm\!0.003$
      & $0.0216\!\pm\!0.0013$
      & $0.000$
      & $27.17\!\pm\!0.12$ \\
    AdaGamma-TD
      & $0.911\!\pm\!0.000$
      & $0.0014\!\pm\!0.0003$
      & $0.000$
      & $26.35\!\pm\!0.20$ \\
    \bottomrule
  \end{tabular}
\end{table}

\begin{figure}[t]
  \centering
  \begin{subfigure}[t]{0.48\linewidth}
    \centering
    \includegraphics[width=\linewidth]{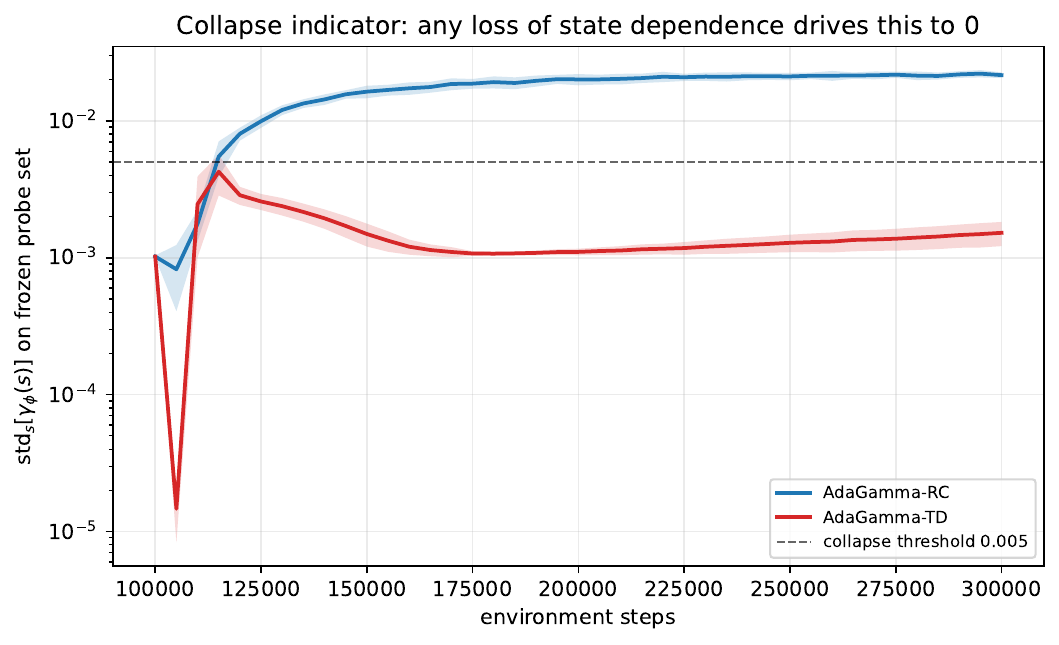}
    \caption{Cross-state standard deviation of~$\gamma_\phi$ on the frozen probe set
    (log scale).}
    \label{fig:collapse-std}
  \end{subfigure}\hfill
  \begin{subfigure}[t]{0.48\linewidth}
    \centering
    \includegraphics[width=\linewidth]{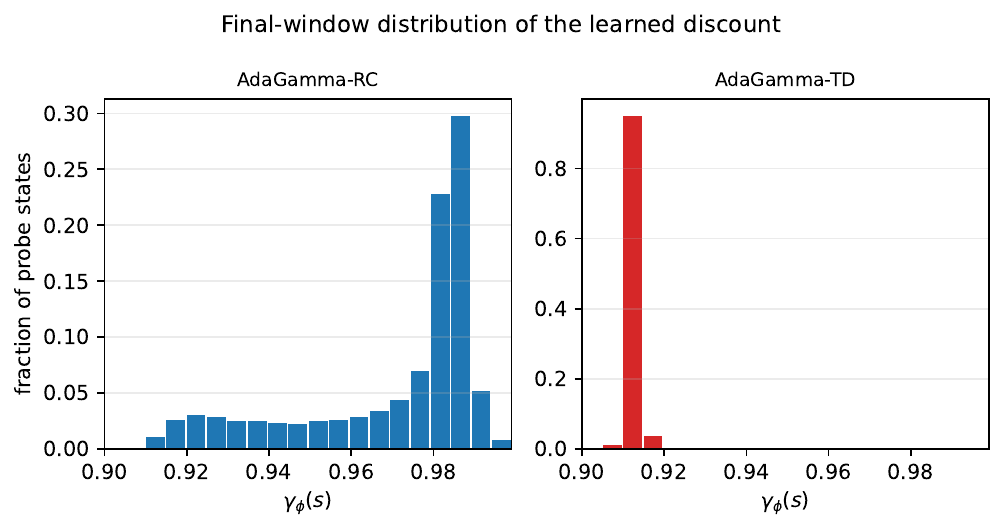}
    \caption{Final-window histogram of~$\gamma_\phi(s)$.}
    \label{fig:collapse-hist}
  \end{subfigure}
  \caption{TD-target collapse under SAC on \texttt{SafetyPointGoal1-v0}.
  Shaded bands are~$\pm$1\,standard deviation across 5 seeds.
  RC remains non-degenerate; squared-TD collapses to a near-constant discount.}
  \label{fig:collapse}
\end{figure}

\subsection{Leave-one-out on the full $\gamma$ objective}
\label{app:gamma_objective_ablation}
Numerical results for the leave-one-out of Section~\ref{sec:ablations}
appear below.
Statistics are means over the final $10\%$ of training, then
mean~$\pm$\,standard deviation across 5 seeds, on a frozen 2048-state
probe set.
Shared hyperparameters are listed in Table~\ref{tab:collapse-hparams}.
\begin{table}[ht]
  \centering
  \caption{Leave-one-out on~(\ref{eq:full_gamma_loss}) for the $\gamma$
  module on \texttt{SafetyPointGoal1-v0} (SAC, 300k steps, 5 seeds).
  Statistics are means over the final $10\%$ of training, then
  mean~$\pm$\,standard deviation across seeds, on a frozen 2048-state probe set.
  All runs disable adaptive~$\bar{\gamma}$.
  \textbf{Std~$\gamma$} is the primary collapse indicator;
  Frac$_{\mathrm{floor}}$ is the fraction of probe states in the bottom $5\%$
  of the~$\gamma$ range.}
  \label{tab:gamma_objective_ablation}
  \setlength{\tabcolsep}{3.5pt}
  \small
  \begin{tabular}{llcccc}
    \toprule
    Obj. & Auxiliaries & mean~$\gamma$ & \textbf{Std~$\gamma$} & Frac$_{\mathrm{floor}}$ & Return \\
    \midrule
    RC & full
      & $0.971\!\pm\!0.003$
      & $0.0216\!\pm\!0.0013$
      & $0.000$
      & $27.17\!\pm\!0.12$ \\
    RC & w/o $\lambda_{\mathrm{dev}}$
      & $0.971\!\pm\!0.002$
      & $0.0215\!\pm\!0.0013$
      & $0.000$
      & $26.88\!\pm\!0.10$ \\
    RC & w/o $\lambda_{\mathrm{var}}$
      & $0.971\!\pm\!0.003$
      & $0.0217\!\pm\!0.0015$
      & $0.000$
      & $27.46\!\pm\!0.31$ \\
    RC & w/o $\lambda_{\mathrm{bound}}$
      & $0.971\!\pm\!0.003$
      & $0.0234\!\pm\!0.0014$
      & $0.000$
      & $27.29\!\pm\!0.27$ \\
    RC & none
      & $0.970\!\pm\!0.003$
      & $0.0238\!\pm\!0.0017$
      & $0.000$
      & $27.20\!\pm\!0.19$ \\
    \midrule
    TD $\delta^2$ & paper defaults
      & $0.911\!\pm\!0.000$
      & $0.0014\!\pm\!0.0003$
      & $0.000$
      & $26.35\!\pm\!0.20$ \\
    TD $\delta^2$ & none
      & $0.900\!\pm\!0.000$
      & $0.0000\!\pm\!0.0000$
      & $1.000$
      & $26.09\!\pm\!0.37$ \\
    \bottomrule
  \end{tabular}
\end{table}

\begin{figure}[t]
  \centering
  \begin{subfigure}[t]{0.48\linewidth}
    \centering
    \includegraphics[width=\linewidth]{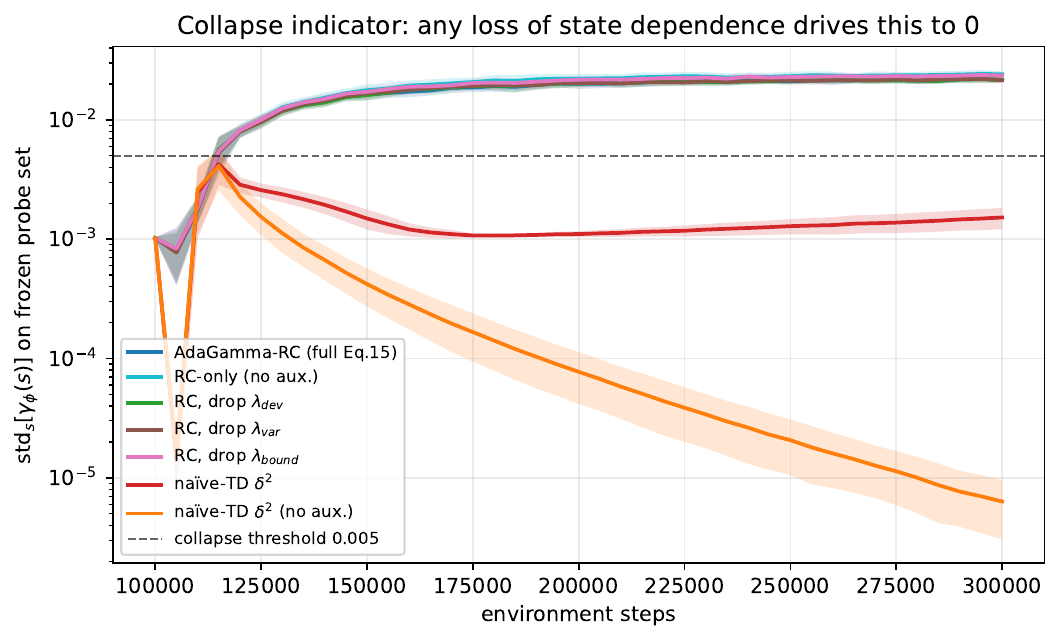}
    \caption{Cross-state standard deviation of~$\gamma_\phi$ on the frozen probe set
    (log scale).}
    \label{fig:gamma-objective-std}
  \end{subfigure}\hfill
  \begin{subfigure}[t]{0.48\linewidth}
    \centering
    \includegraphics[width=\linewidth]{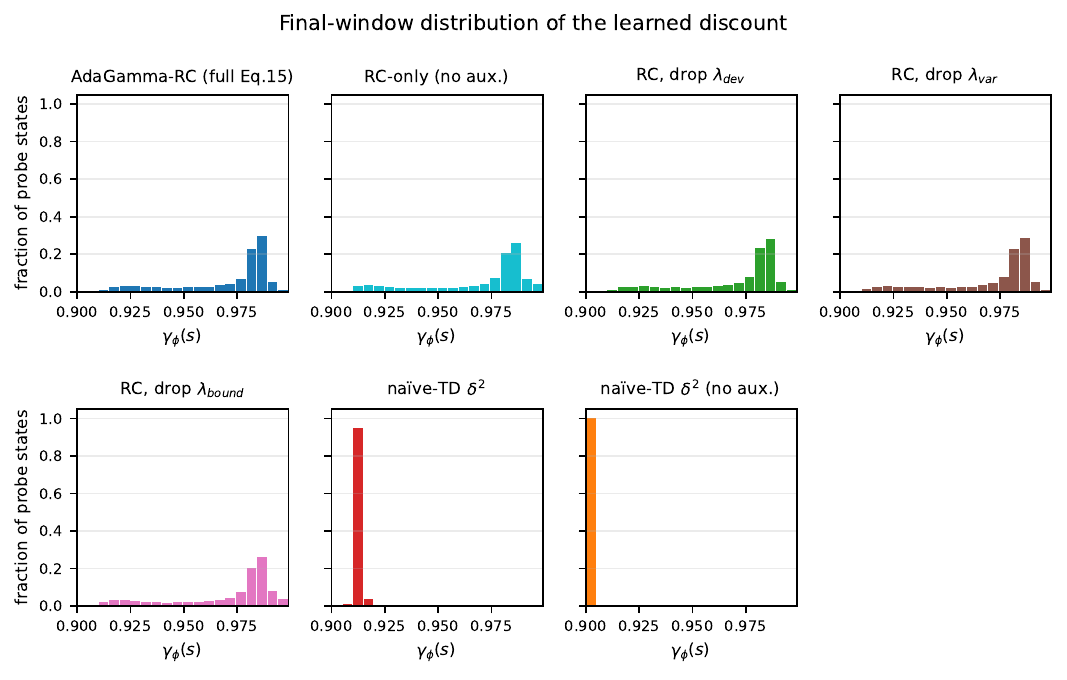}
    \caption{Final-window histogram of~$\gamma_\phi(s)$.}
    \label{fig:gamma-objective-hist}
  \end{subfigure}
  \caption{Leave-one-out on~(\ref{eq:full_gamma_loss}) under SAC on
  \texttt{SafetyPointGoal1-v0}.
  Shaded bands are~$\pm$1\,standard deviation across 5 seeds.
  Every return-consistency variant remains non-degenerate; squared-TD
  collapses to a near-constant discount with the paper's regularizers, and to
  a point mass at~$\gamma_{\min}$ without them.}
  \label{fig:gamma_objective_ablation}
\end{figure}

\begin{figure}[t]
  \centering
  \begin{subfigure}[t]{0.48\linewidth}
    \centering
    \includegraphics[width=\linewidth]{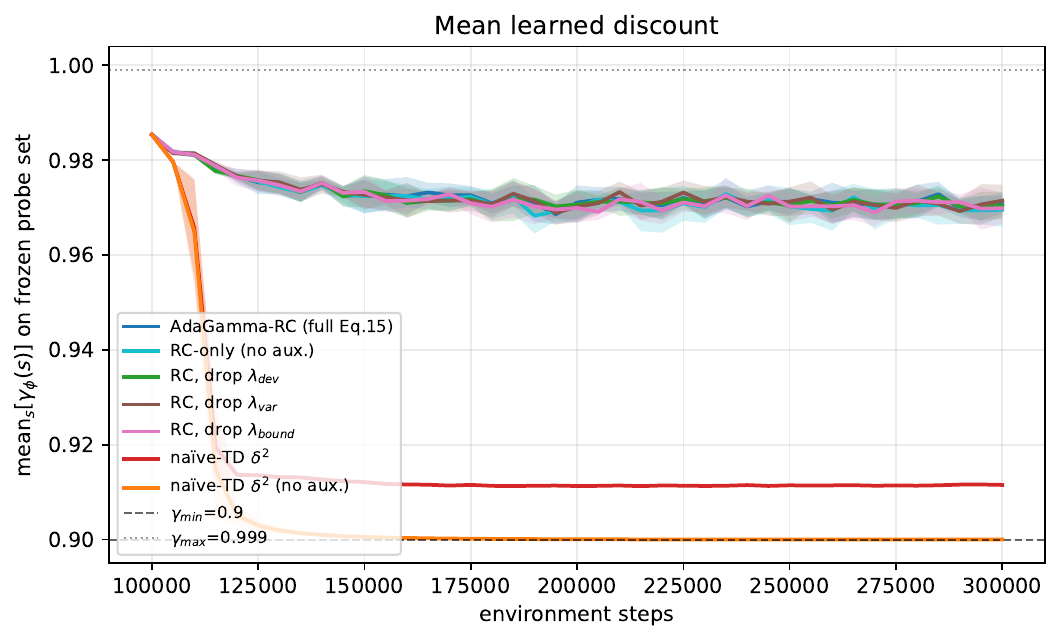}
    \caption{Mean probe~$\gamma_\phi$ over training.}
    \label{fig:gamma-objective-mean}
  \end{subfigure}\hfill
  \begin{subfigure}[t]{0.48\linewidth}
    \centering
    \includegraphics[width=\linewidth]{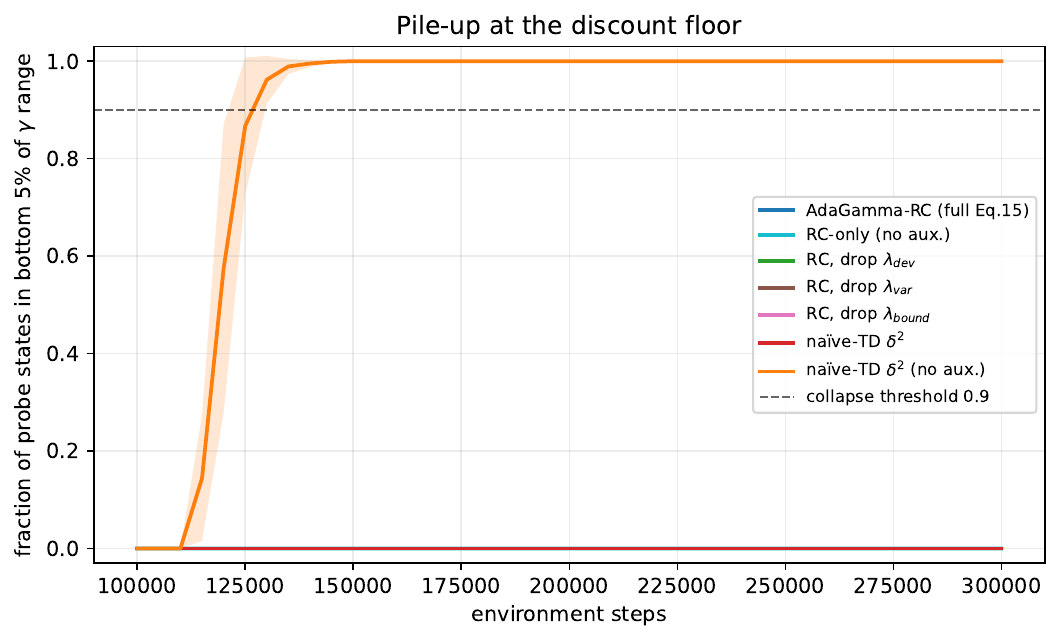}
    \caption{Fraction of probe states in the bottom $5\%$ of the~$\gamma$ range.}
    \label{fig:gamma-objective-floor}
  \end{subfigure}
  \caption{Additional diagnostics for~(\ref{eq:full_gamma_loss})
  leave-one-out.
  Mean~$\gamma$ alone cannot establish collapse.
  Only TD without auxiliaries piles up at~$\gamma_{\min}$; TD with regularizers
  settles interior to the range as a near-constant, so floor mass stays near
  zero while standard deviation collapses.}
  \label{fig:gamma_objective_diagnostics}
\end{figure}

\subsection{Statewise structure of $\gamma_\phi(s)$}
\label{app:gamma_heatmap}
Figure~\ref{fig:heatmap} visualizes SAC-AdaGamma's learned discount on
\texttt{Ant-v4} from rollout states of a fixed trained policy.
\begin{figure}[ht]
    \centering
    \begin{subfigure}[b]{0.24\textwidth}
        \centering
        \includegraphics[width=\textwidth]{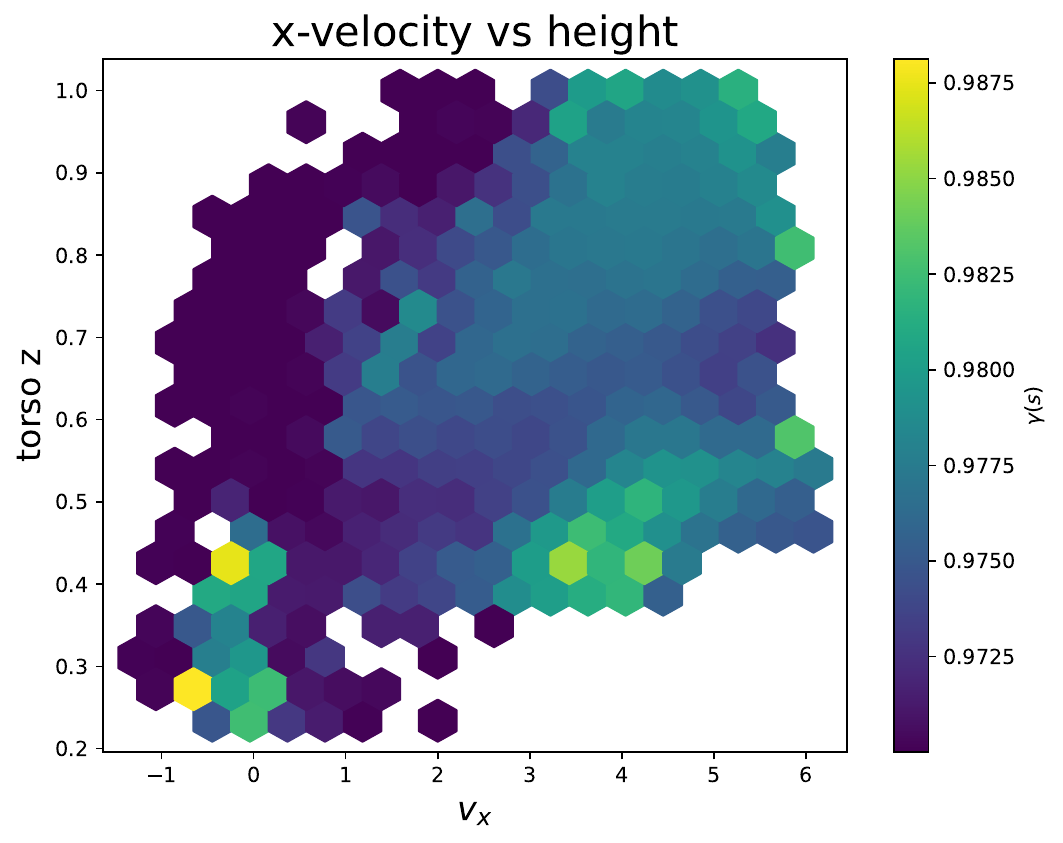}
    \end{subfigure}
    \hfill
    \begin{subfigure}[b]{0.23\textwidth}
        \centering
        \includegraphics[width=\textwidth]{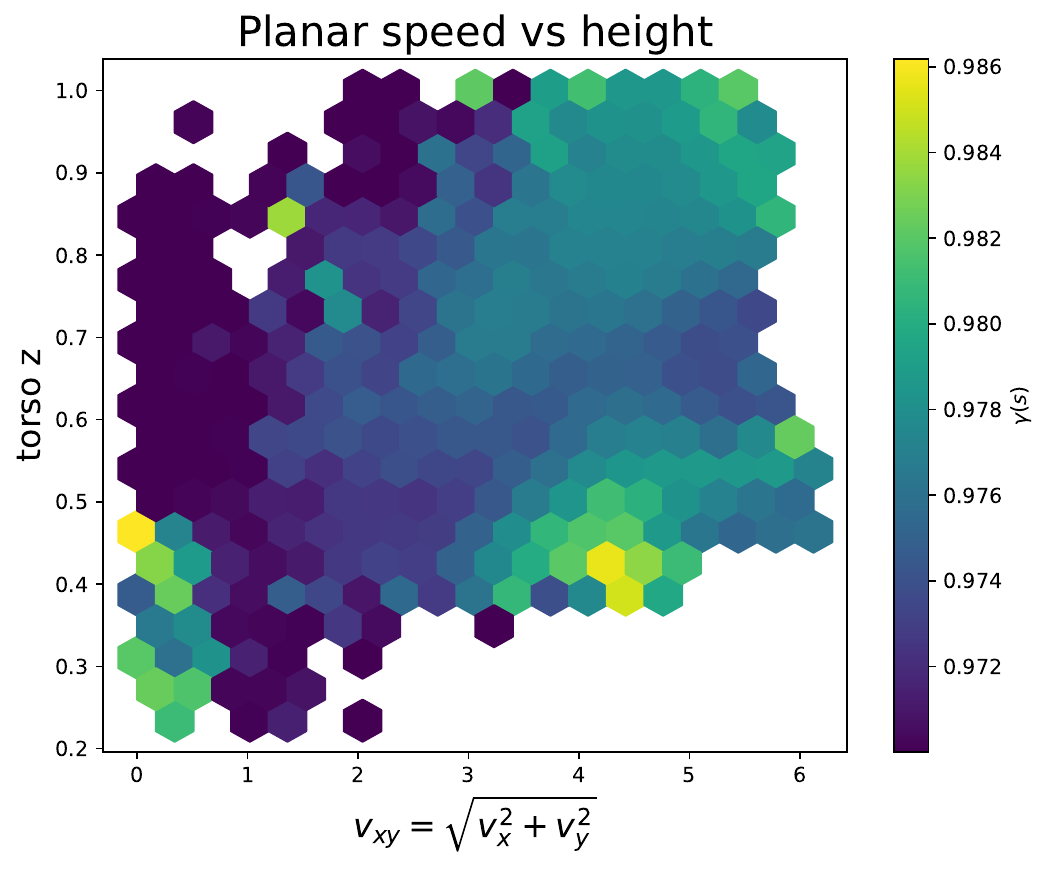}
    \end{subfigure}
    \begin{subfigure}[b]{0.48\textwidth}
        \centering
        \includegraphics[width=\textwidth]{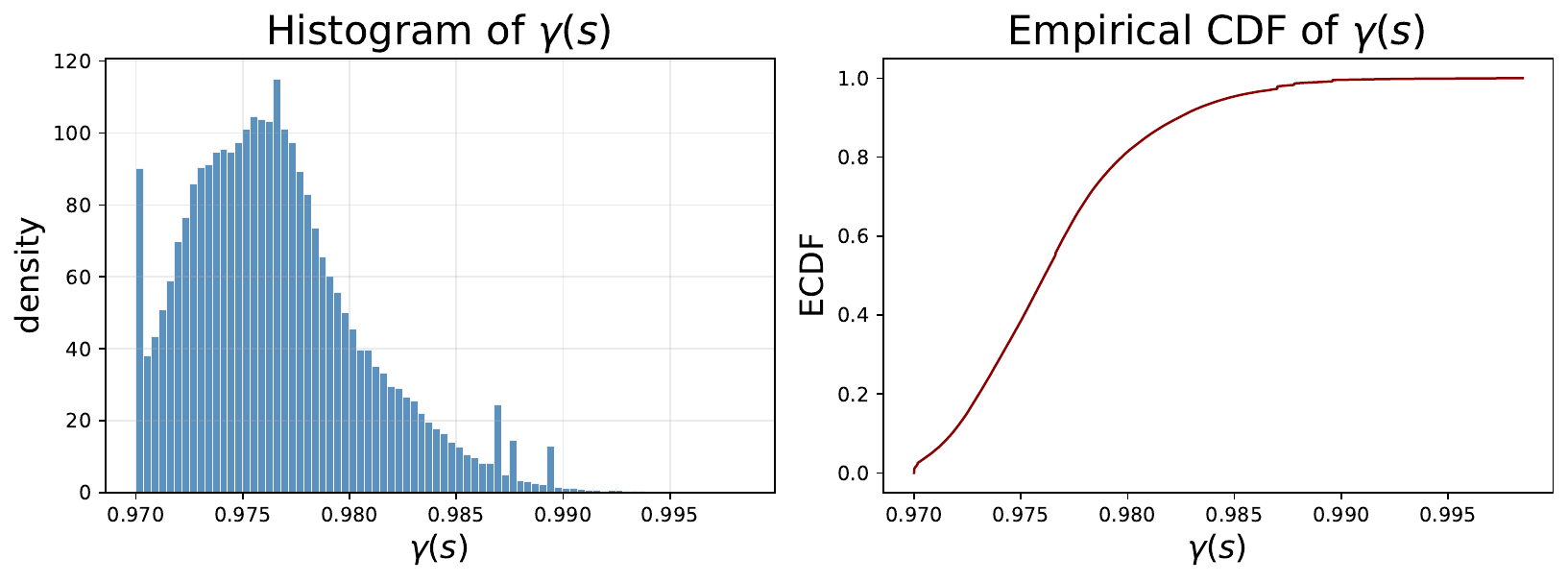}
    \end{subfigure}
    \caption{Statewise heatmaps and empirical distribution of \(\gamma_\phi(s)\) learned by SAC-AdaGamma on \texttt{Ant-v4}.}
    \label{fig:heatmap}
\end{figure}

\subsection{Cross-algorithm consistency}
\label{app:consistency}
Table~\ref{tab:gamma_cross_algorithm} reports the mean learned discount for
SAC-AdaGamma and PPO-AdaGamma across representative tasks.
\begin{table}[ht]
  \centering
  \caption{Mean learned \(\gamma_\phi(s)\) for SAC-AdaGamma and PPO-AdaGamma.}
  \scalebox{0.7}{
\begin{tabular}{l|l|l|l|l}
\toprule
Method Variants & SafetyPointGoal1-v0 & Humanoid-v4 & Ant-v4 & Pendulum-v1 \\
\hline
SAC-AdaGamma & 0.9989 & 0.9860 & 0.9733 & 0.9908 \\
PPO-AdaGamma & 0.9988 & 0.9846 & 0.9990 & 0.9908 \\
\bottomrule
\end{tabular}
}
\label{tab:gamma_cross_algorithm}
\end{table}

\subsection{Training Objective Comparison}
\label{app:training_objective}
Table~\ref{tab:gamma_sacs} summarizes the mean learned discounts of SAC variants
across the three main benchmarks.
\begin{table}[ht]
  \centering
 \caption{Mean value of $\gamma$ learned by different SAC variants.}
  \scalebox{0.7}{
\begin{tabular}{l|l|l|l}
\toprule
Method Variants    & SafetyPointGoal1-v0 & Humanoid-v4 & Ant-v4 \\
\hline
SAC-fixed $\gamma$ & 0.9900 & 0.9900 & 0.9900 \\ 
SAC-CrossValidate  & 0.9000 & 0.9000 & 0.9000 \\ 
SAC-Uncertainty    & 0.9495 & 0.9745 & 0.9745 \\ 
SAC-AdaGamma       & 0.9989 & 0.9860 & 0.9733 \\
\bottomrule
\end{tabular}
}
\label{tab:gamma_sacs}
\end{table}

\subsection{Network architecture}
\label{app:hidden_dim}
Figure~\ref{fig:hidden_dim} reports the performance of SAC-AdaGamma on
\texttt{Ant-v4} as the hidden dimension of the gamma network varies over
$\{8,16,32,64,128,256,512\}$.
\begin{figure}
    \centering
    \includegraphics[width=0.5\linewidth]{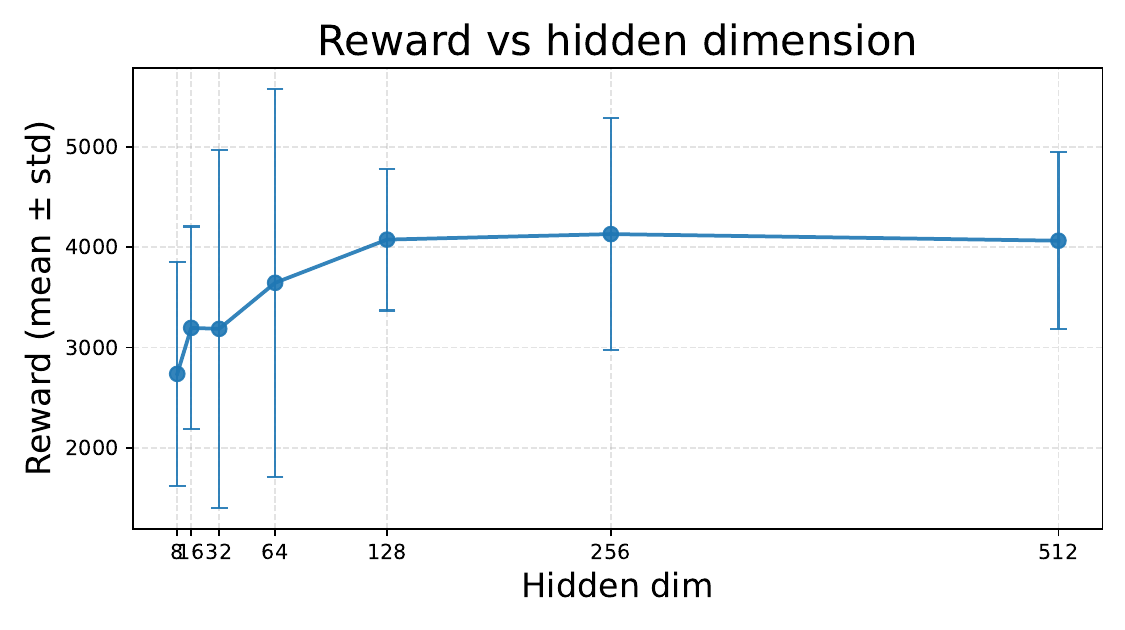}
    \caption{Performance of SAC-AdaGamma on \texttt{Ant-v4} as the gamma-network hidden dimension varies. The curve reports reward mean~$\pm$~standard deviation.}
    \label{fig:hidden_dim}
\end{figure}

\subsection{SAC/PPO with AdaGamma under different return horizons}
\label{app:return_horizon}
Table~\ref{tab:return_horizon} reports the rewards and costs of SAC-AdaGamma
and PPO-AdaGamma under different return horizons \(n\).

\begin{table}[ht]
  \centering
  \caption{Mean reward and cost, with standard deviation, for SAC-AdaGamma and PPO-AdaGamma with $n \in \{5,10,20\}$.}
\begin{tabular}{l|l|l}
\toprule
Method Variants    & SafetyPointGoal1-v0 Reward & SafetyPointGoal1-v0 Cost  \\ 
\hline
SAC-AdaGamma($n=5$)  & \textbf{28.254 $\pm$ 1.05} & \textbf{29.37 $\pm$ 18.59} \\
SAC-AdaGamma($n=10$) & 27.75 $\pm$ 0.85 & 56.20 $\pm$ 36.93 \\
SAC-AdaGamma($n=20$) & 27.71 $\pm$ 1.05 & 57.70 $\pm$ 48.42 \\
\hline
PPO-AdaGamma($n=5$)  & \textbf{26.30 $\pm$ 0.90} & \textbf{41.82 $\pm$ 19.42} \\
PPO-AdaGamma($n=10$) & 25.92 $\pm$ 1.27 & 48.60 $\pm$ 32.44 \\
PPO-AdaGamma($n=20$) & 16.06 $\pm$ 12.11 & 49.85 $\pm$ 36.39 \\
\bottomrule
\end{tabular}
\label{tab:return_horizon}
\end{table}

\subsection{Comparison to fixed-$\gamma$}
\label{app:fixed_gamma}
Table~\ref{tab:ab_fixed_gamma} compares AdaGamma with fixed discounts chosen
around the learned mean discount. Table~\ref{tab:grid_search} reports a broader
fixed-\(\gamma\) grid search for SAC and PPO.
\begin{table}[ht]
  \centering
  \caption{Performance of SAC and PPO with fixed $\gamma$ around the learned mean value.}
  \label{tab:ab_fixed_gamma}
  \scalebox{0.7}{
  \begin{tabular}{l|cc||l|cc}
    \toprule
    Method & Reward & Cost & Method & Reward & Cost\\
    \midrule
    \multicolumn{6}{c}{\texttt{SafetyPointGoal1-v0}} \\
    \hline
    SAC-Fixed-$\gamma = 0.9950$    & 27.57 $\pm$ 0.81 &  51.25 $\pm$ 38.92  & PPO-Fixed-$\gamma = 0.9950$    & 25.36 $\pm$ 1.63  & 53.25 $\pm$ 32.13  \\
    SAC-Fixed-$\gamma = 0.9989$      & 27.06 $\pm$ 4.55 & 53.20 $\pm$ 37.24  & PPO-Fixed-$\gamma = 0.9988$      & 23.86 $\pm$ 1.93 & 46.85 $\pm$ 27.73  \\
   SAC-Fixed-$\gamma = 0.9990$       &  25.92 $\pm$ 1.34 & 50.20 $\pm$ 40.17   & PPO-Fixed-$\gamma = 0.9990$        & 25.71 $\pm$ 1.32  & 48.60 $\pm$ 29.07 \\
    \hline
    SAC-AdaGamma          & $\mathbf{28.25 \pm 1.05}$ & $\mathbf{29.37 \pm 18.59}$ & PPO-AdaGamma          & $\mathbf{26.31 \pm 0.90}$ & $\mathbf{41.82 \pm 19.42}$ \\
    \hline
    \multicolumn{6}{c}{\texttt{Humanoid-v4}} \\
    \hline
    SAC-Fixed-$\gamma = 0.9850$    & 6050.05 $\pm$ 22.50  &   -  & PPO-Fixed-$\gamma = 0.9950$    & 456.19 $\pm$ 27.61   & -   \\
    SAC-Fixed-$\gamma = 0.9860 $      & 5830.34 $\pm$ 6.54  &  -  & PPO-Fixed-$\gamma = 0.9999$      & 442.33 $\pm$ 17.78  & - \\
   SAC-Fixed-$\gamma = 0.9990$       & 5576.23 $\pm$ 17.84   &  -   & PPO-Fixed-$\gamma = 0.9990$        & 447.99 $\pm$ 23.06 & - \\
    \hline
    SAC-AdaGamma          & $\mathbf{6907.99 \pm 21.76}$  & -  & PPO-AdaGamma          & $\mathbf{476.51 \pm 51.29}$  &  -  \\
    \hline
    \multicolumn{6}{c}{\texttt{Ant-v4}} \\
    \hline
     SAC-Fixed-$\gamma = 0.9700$    & 3601.58 $\pm$ 1812.92  &   -  & PPO-Fixed-$\gamma = 0.9980$    & 1064.52 $\pm$ 193.92 & -   \\
    SAC-Fixed-$\gamma = 0.9733$      & 3461.86 $\pm$ 1803.14  &  -  & PPO-Fixed-$\gamma = 0.9985$      &  987.77 $\pm$ 234.11  & - \\
   SAC-Fixed-$\gamma = 0.9750$       & 3778.66 $\pm$ 644.60  &  -   & PPO-Fixed-$\gamma = 0.9990$        & 1033.70 $\pm$ 192.00 & - \\
    \hline
    SAC-AdaGamma          & $\mathbf{4129.57 \pm 1155.08}$  & -  & PPO-AdaGamma          & $\mathbf{1172.47 \pm 189.84}$ &  -  \\
    \bottomrule
  \end{tabular}
  }
\end{table}

\begin{table}[ht]
  \centering
  \caption{Grid search of SAC and PPO with fixed $\gamma$.}
  \label{tab:grid_search}
  \scalebox{0.7}{
  \begin{tabular}{l|cc||l|cc}
    \toprule
    Method & Reward & Cost & Method & Reward & Cost\\
    \midrule
     \multicolumn{6}{c}{\texttt{SafetyPointGoal1-v0}} \\
    \hline
     SAC-Fixed-$\gamma = 0.9000$    & 27.13 $\pm$ 1.09   & 53.30 $\pm$ 40.87     & PPO-Fixed-$\gamma = 0.9000$    & 18.98 $\pm$ 2.34   & 51.80 $\pm$ 27.30   \\
    SAC-Fixed-$\gamma = 0.9500$    & 27.88 $\pm$ 1.23    & 49.45 $\pm$ 34.19     & PPO-Fixed-$\gamma = 0.9500$    &  24.43 $\pm$ 1.80  &  49.95 $\pm$ 26.76  \\
     SAC-Fixed-$\gamma = 0.9700$    & 27.32 $\pm$ 1.30   & 51.65 $\pm$ 33.76      & PPO-Fixed-$\gamma = 0.9700$    & 22.25 $\pm$ 1.57   & 46.45 $\pm$ 31.19   \\
     SAC-Fixed-$\gamma = 0.9800$    & 27.96 $\pm$ 1.12   & 46.45 $\pm$ 35.90      & PPO-Fixed-$\gamma = 0.9800$    & 21.60 $\pm$ 4.80    & 50.80 $\pm$ 37.66   \\
   SAC-Fixed-$\gamma = 0.9990$       &  25.92 $\pm$ 1.34 & 50.20 $\pm$ 40.17   & PPO-Fixed-$\gamma = 0.9990$        & 25.71 $\pm$ 1.32  & 48.60 $\pm$ 29.07 \\
    \hline
    SAC-AdaGamma          & $\mathbf{28.25 \pm 1.05}$ & $\mathbf{29.37 \pm 18.59}$ & PPO-AdaGamma          & $\mathbf{26.31 \pm 0.90}$ & $\mathbf{41.82 \pm 19.42}$ \\
    \hline
    \multicolumn{6}{c}{\texttt{Humanoid-v4}} \\
    \hline
   SAC-Fixed-$\gamma = 0.9000$    & 723.73 $\pm$ 72.03   &   -  & PPO-Fixed-$\gamma = 0.9000$    & 305.43 $\pm$ 61.33  & -   \\
      SAC-Fixed-$\gamma = 0.9500$    & 3280.41 $\pm$ 1826.32    &   -  & PPO-Fixed-$\gamma = 0.9500$    & 349.12 $\pm$ 57.15  & -   \\
     SAC-Fixed-$\gamma = 0.9700$    & 6185.12 $\pm$ 29.35   &   -  & PPO-Fixed-$\gamma = 0.9700$    & 315.13 $\pm$ 19.90  & -   \\
    SAC-Fixed-$\gamma = 0.9800$    & 6346.42 $\pm$ 15.40  &   -  & PPO-Fixed-$\gamma = 0.9800$    & 356.19 $\pm$ 27.61  & -   \\
   SAC-Fixed-$\gamma = 0.9990$       & 5576.23 $\pm$ 17.84   &  -   & PPO-Fixed-$\gamma = 0.9990$        & 447.99 $\pm$ 23.06 & - \\
    \hline
    SAC-AdaGamma          & $\mathbf{6907.99 \pm 21.76}$  & -  & PPO-AdaGamma          & $\mathbf{476.51 \pm 51.29}$  &  -  \\
    \hline
    \multicolumn{6}{c}{\texttt{Ant-v4}} \\
    \hline
     SAC-Fixed-$\gamma = 0.9000$    & 1223.375  $\pm$  137.70   &   -  & PPO-Fixed-$\gamma = 0.9000$    & 750.41  $\pm$  70.57  & -   \\
      SAC-Fixed-$\gamma = 0.9500$    & 3236.72 $\pm$ 1868.49   &   -  & PPO-Fixed-$\gamma = 0.9500$    & 823.30  $\pm$  120.42  & -   \\
     SAC-Fixed-$\gamma = 0.9700$    & 3601.58 $\pm$ 1812.92   &   -  & PPO-Fixed-$\gamma = 0.9700$    & 893.34  $\pm$  141.93  & -   \\
    SAC-Fixed-$\gamma = 0.9800$    & 3761.03 $\pm$ 1692.59   &   -  & PPO-Fixed-$\gamma = 0.9800$    & 994.73  $\pm$  187.78  & -   \\
   SAC-Fixed-$\gamma = 0.9990$    & 918.32 $\pm$ 17.68   &   -  & PPO-Fixed-$\gamma = 0.9990$    & 1033.70 $\pm$ 192.00  & -   \\
    \hline
    SAC-AdaGamma          & $\mathbf{4129.57 \pm 1155.08}$  & -  & PPO-AdaGamma          & $\mathbf{1172.47 \pm 189.84}$ &  -  \\
    \bottomrule
  \end{tabular}
  }
\end{table}

\section{Algorithmic Framework under AdaGamma}
\label{app:ada_algo}
This section summarizes the SAC and PPO implementations of AdaGamma.

\subsection{SAC-AdaGamma}
\label{app:sac_ada_algo}
Algorithm~\ref{alg:sac_adf} summarizes the SAC implementation of AdaGamma.
\begin{algorithm}[ht]
\caption{AdaGamma with SAC (Off-Policy Adapter)}
\label{alg:sac_adf}
\begin{algorithmic}[1]
\STATE \textbf{Input:} Policy $\pi_\psi$, Q-networks 
$Q_{\theta_1}, Q_{\theta_2}$, gamma network $g_\phi$, 
replay buffer $\mathcal{D}$
\FOR{each iteration}
  \FOR{each environment step}
    \STATE $a_t \sim \pi_\psi(\cdot|s_t)$, observe 
    $r_t, s_{t+1}, d_t$
    \STATE $\mathcal{D} \gets \mathcal{D} \cup 
    \{(s_t,a_t,r_t,s_{t+1},d_t)\}$
  \ENDFOR
  \FOR{each gradient step}
    \STATE Sample mini-batch from $\mathcal{D}$
    \STATE \textit{// Compute adaptive discount:}
    \STATE $\gamma_\phi(s_t) = \gamma_{\min} + 
    (\gamma_{\max}-\gamma_{\min})\cdot
    \sigma(g_\phi(s_t))$
    \STATE \textit{// Update Q-functions with adaptive target:}
    \STATE $\hat{Q} = r_t + \gamma_\phi(s_t)(1-d_t)
    [\min_i Q_{\bar{\theta}_i}(s_{t+1},a') 
    - \alpha\log\pi_\psi(a'|s_{t+1})]$
    \STATE $\theta_i \gets \theta_i - \lambda_Q 
    \nabla_{\theta_i} J_Q(\theta_i)$
    \STATE \textit{// Update policy and temperature 
    (unchanged):}
    \STATE $\psi \gets \psi - \lambda_\pi 
    \nabla_\psi J_\pi(\psi)$; \quad
    $\alpha \gets \alpha - \lambda_\alpha 
    \nabla_\alpha J(\alpha)$
    \STATE \textit{// Update gamma network 
    (return-consistency):}
    \STATE Compute $G_t^{(n)}$ with reference 
    $\bar{\gamma}$ and terminal masks (stop-gradient)
    \STATE Optionally update $\bar{\gamma}$ by EMA of 
    replay-mean $\gamma_\phi(s)$ after warmup
    \STATE $\phi \gets \phi - \lambda_\gamma 
    \nabla_\phi J_\gamma(\phi)$ 
    \quad ~(\ref{eq:full_gamma_loss})
    \STATE \textit{// Soft-update target networks:}
    \STATE $\bar{\theta}_i \gets \tau\theta_i + 
    (1-\tau)\bar{\theta}_i$
  \ENDFOR
\ENDFOR
\end{algorithmic}
\end{algorithm}

\subsection{PPO-AdaGamma}
\label{app:ppo_ada_algo}
Algorithm~\ref{alg:ppo_adf} summarizes the PPO implementation of AdaGamma.
\begin{algorithm}[ht]
\caption{AdaGamma with PPO (On-Policy Adapter)}
\label{alg:ppo_adf}
\begin{algorithmic}[1]
\STATE \textbf{Input:} Policy $\pi_\psi$, value network 
$V_\omega$, gamma network $g_\phi$
\FOR{each iteration}
  \STATE \textit{// Collect rollout:}
  \FOR{$t = 0, \ldots, T-1$}
    \STATE $a_t \sim \pi_\psi(\cdot|s_t)$, observe 
    $r_t, s_{t+1}, d_t$
  \ENDFOR
  \STATE \textit{// Compute $\gamma_\phi(s_t)$ for all 
  rollout states (freeze):}
  \STATE $\{\gamma_t\}_{t=0}^{T-1} \gets 
  \{\gamma_\phi(s_t)\}_{t=0}^{T-1}$
  \STATE \textit{// Modified GAE with state-dependent 
  $\gamma$:}
  \STATE $\delta_{T-1}=r_{T-1}+\gamma_{T-1}(1-d_{T-1})V(s_T)-V(s_{T-1})$
  \STATE $\hat{A}_{T-1}=\delta_{T-1}$
  \FOR{$t = T-2, \ldots, 0$}
    \STATE $\delta_t = r_t + \gamma_t(1-d_t)V(s_{t+1}) - V(s_t)$
    \STATE $\hat{A}_t = \delta_t + \gamma_t\lambda(1-d_t)\hat{A}_{t+1}$
  \ENDFOR
  \STATE Normalize: $\hat{A}_t \gets 
  (\hat{A}_t - \text{mean})/\text{std}$
  \STATE \textit{// PPO epochs (frozen $\gamma$):}
  \FOR{epoch $= 1, \ldots, K$}
    \STATE Update $\pi_\psi$ via clipped surrogate 
    with $\hat{A}_t$
    \STATE Update $V_\omega$ on value targets
  \ENDFOR
  \STATE \textit{// Update gamma network:}
  \STATE Compute $L_\gamma^{\text{RC}}$ from rollout 
  data with terminal masks~(\ref{eq:return_consistency})
  \STATE $\phi \gets \phi - \lambda_\gamma 
  \nabla_\phi J_\gamma(\phi)$
\ENDFOR
\end{algorithmic}
\end{algorithm}

State-dependent $\gamma$ makes advantage magnitudes vary across
states, which interacts with PPO's clipping threshold. We normalize
advantages by batch mean and standard deviation. During PPO's
multiple epochs on one rollout we freeze $\gamma_\phi(s)$, computed
once from the rollout states, so the gamma network cannot co-adapt
with the policy inside a trust-region update.
The product-of-gammas expansion also applies to $n$-step value
targets:
\begin{equation}
\small
V_{\text{target}}(s_t)
=
\sum_{k=0}^{n-1}
\left(
\prod_{j=0}^{k-1}\gamma_\phi(s_{t+j})(1-d_{t+j})
\right)r_{t+k}
+
\left(
\prod_{j=0}^{n-1}\gamma_\phi(s_{t+j})(1-d_{t+j})
\right)
V(s_{t+n}).
\label{eq:nstep_target}
\end{equation}

\subsection{Reward Curves}
\label{app:reward_curves}
This section reports the running reward curves of SAC and PPO, including
AdaGamma and the baselines, on \texttt{SafetyPointGoal1-v0},
\texttt{Humanoid-v4}, and \texttt{Ant-v4}.

\begin{figure}[ht]
    \centering
    \begin{subfigure}[b]{0.32\textwidth}
        \centering
        \includegraphics[width=\textwidth]{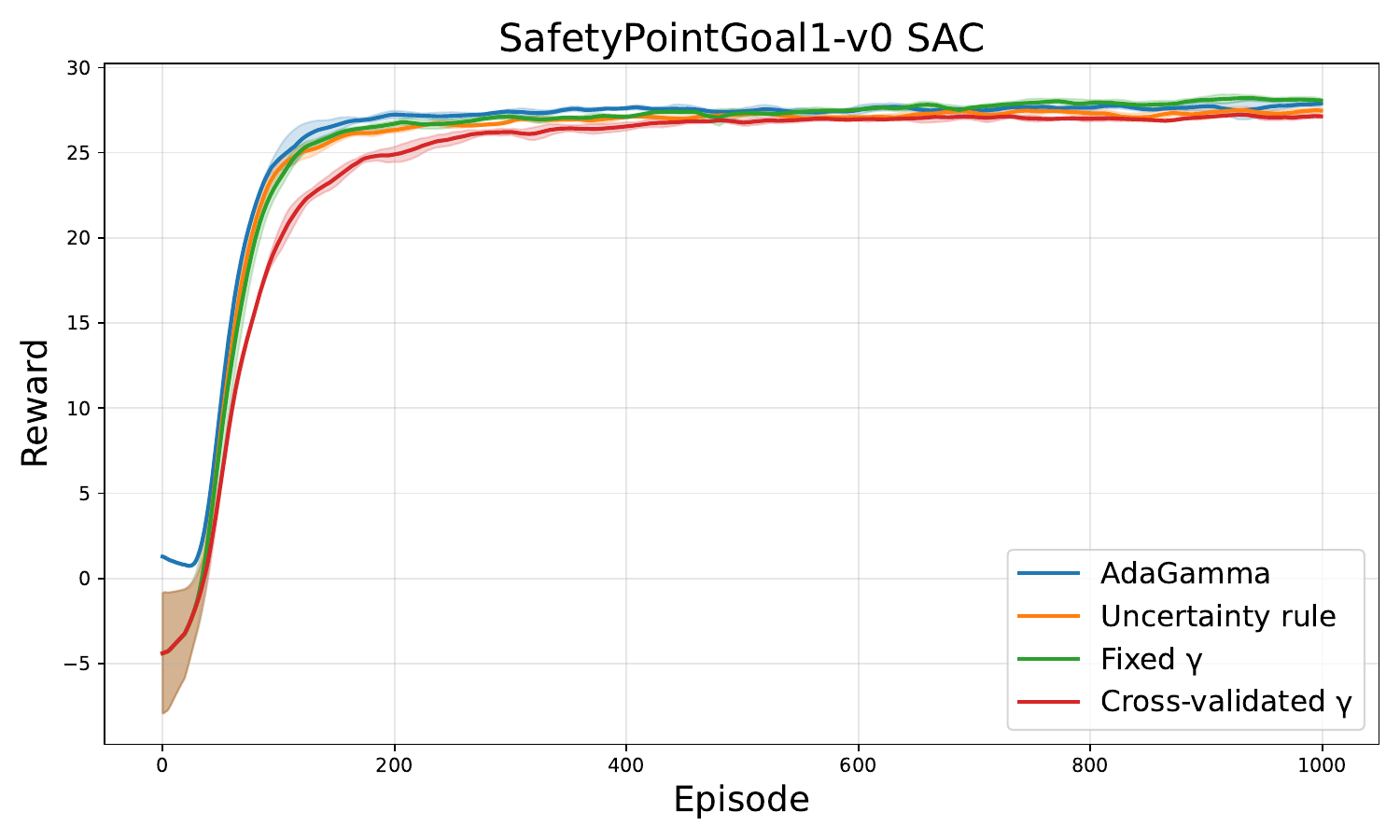}
    \end{subfigure}
    \hfill
    \begin{subfigure}[b]{0.32\textwidth}
        \centering
        \includegraphics[width=\textwidth]{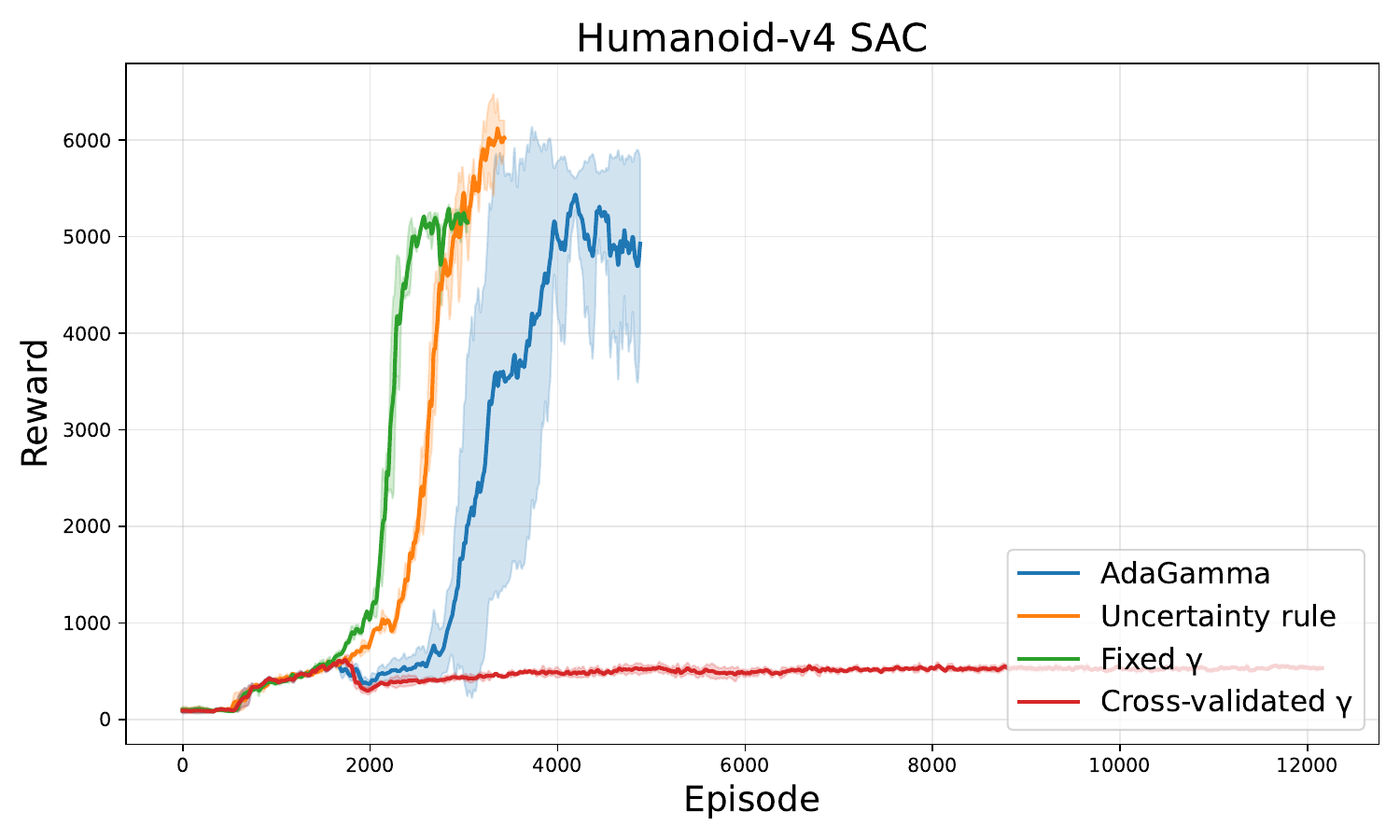}
    \end{subfigure}
    \hfill
    \begin{subfigure}[b]{0.32\textwidth}
        \centering
        \includegraphics[width=\textwidth]{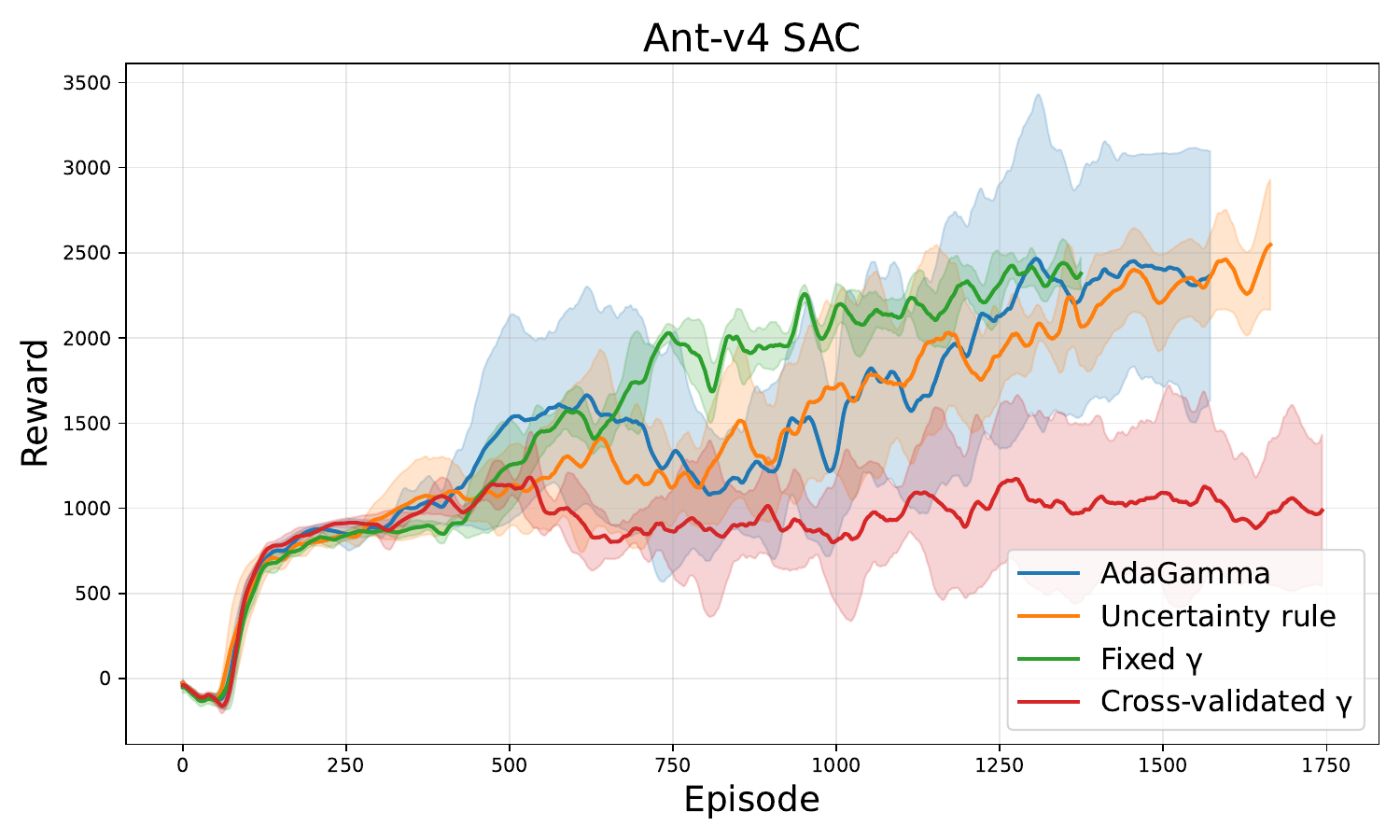}
    \end{subfigure}
    \caption{Learning reward curves of SAC-AdaGamma on \texttt{SafetyPointGoal1-v0}, \texttt{Humanoid-v4}, and \texttt{Ant-v4}.}
    \label{fig:reward_sac}
\end{figure}

\begin{figure}[ht]
    \centering
    \begin{subfigure}[b]{0.32\textwidth}
        \centering
        \includegraphics[width=\textwidth]{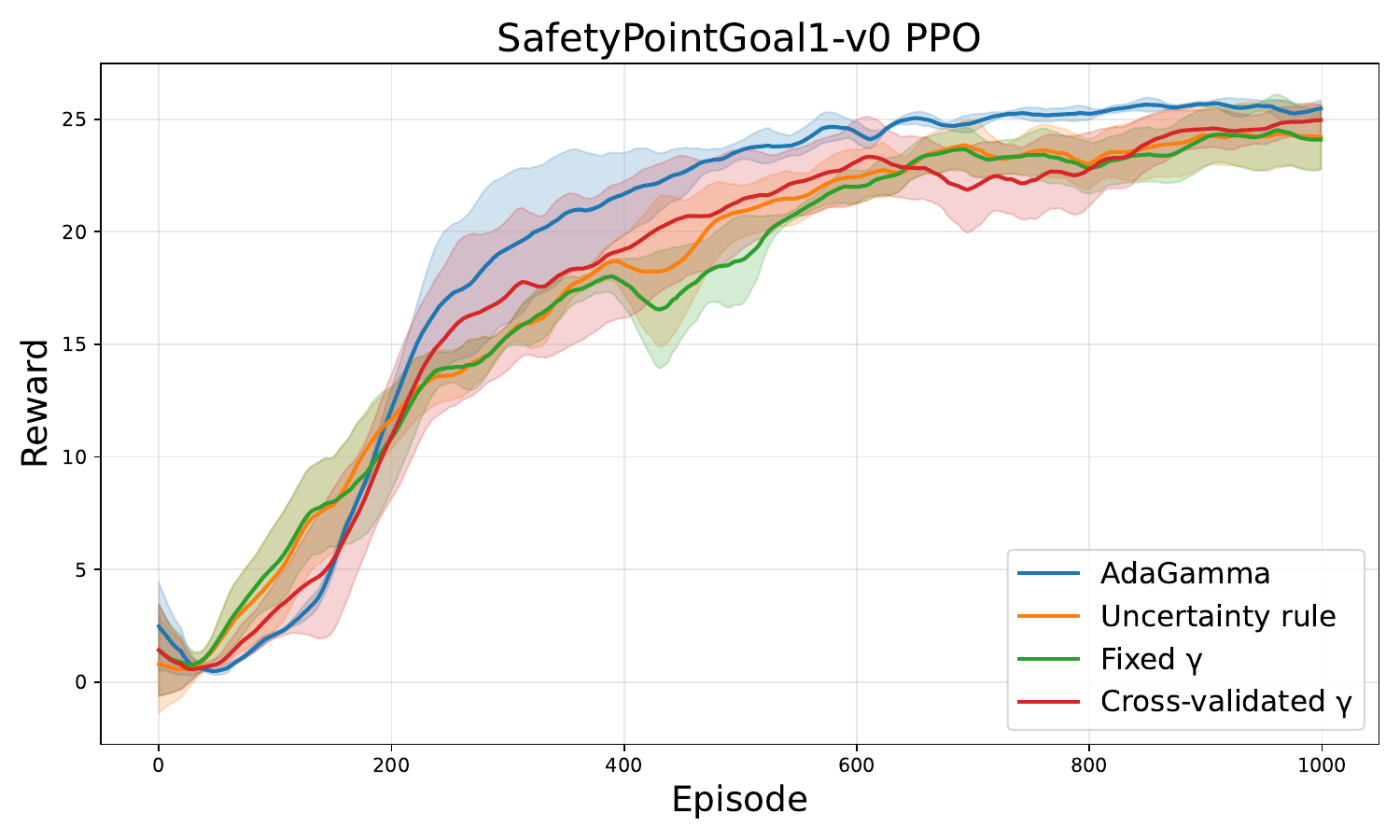}
    \end{subfigure}
    \hfill
    \begin{subfigure}[b]{0.32\textwidth}
        \centering
        \includegraphics[width=\textwidth]{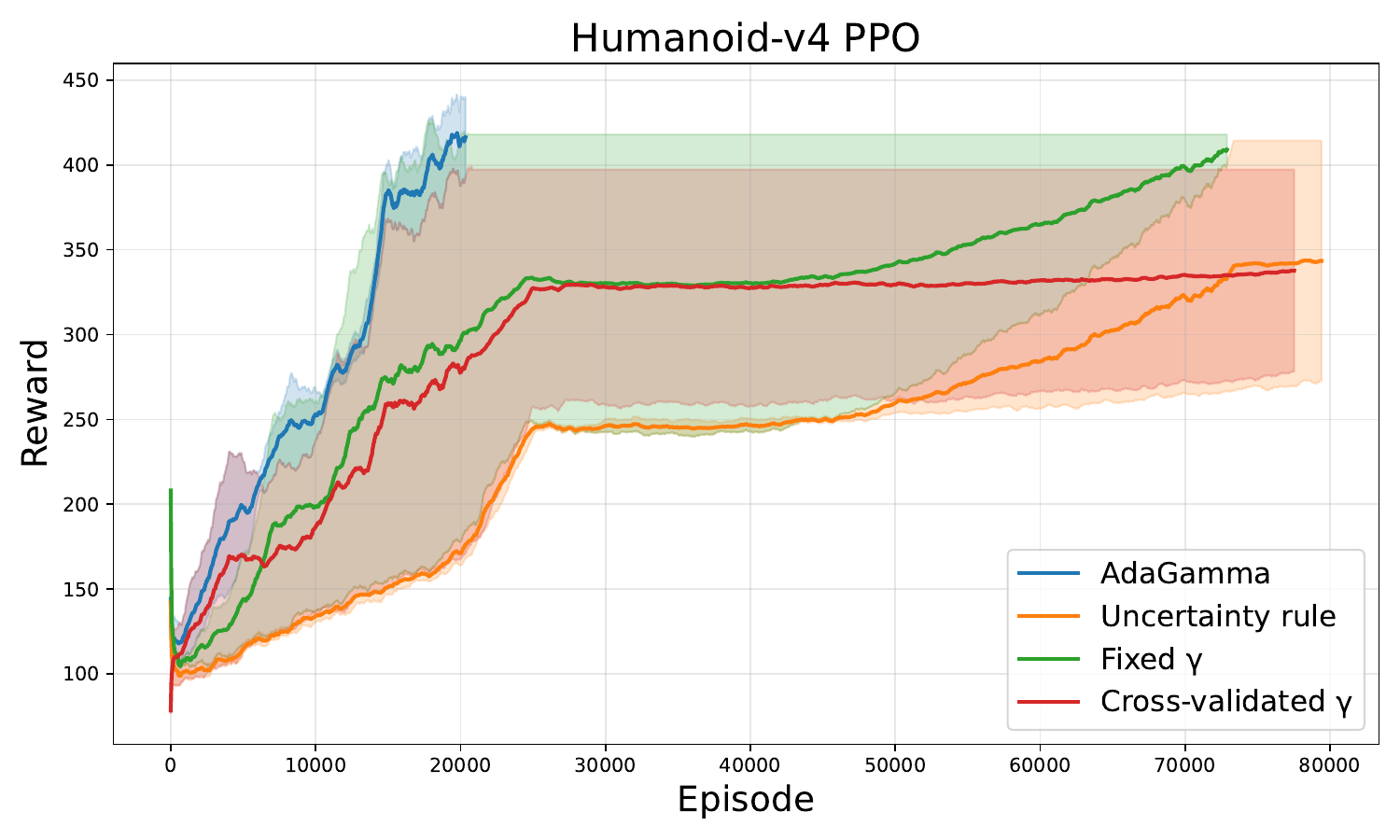}
    \end{subfigure}
    \hfill
    \begin{subfigure}[b]{0.32\textwidth}
        \centering
        \includegraphics[width=\textwidth]{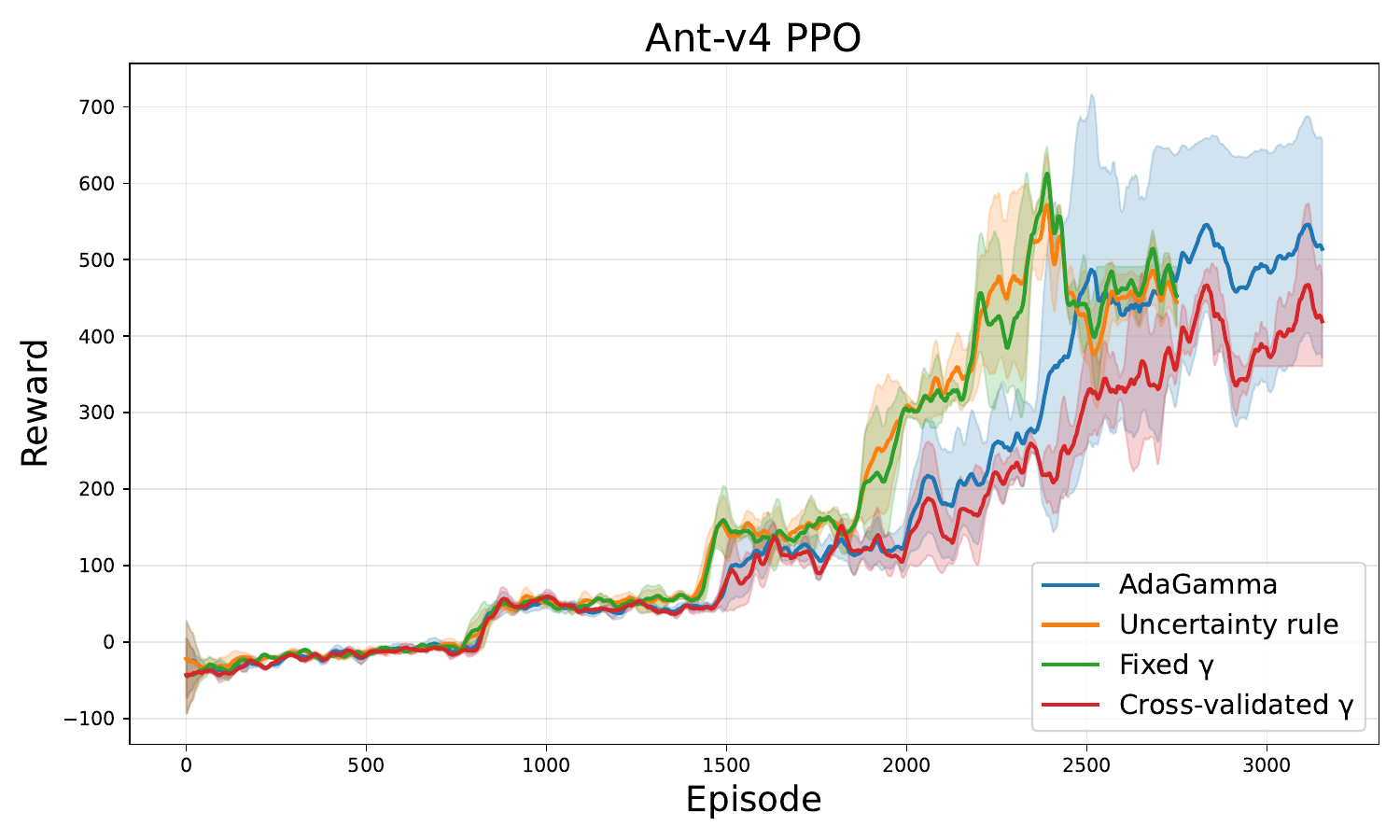}
    \end{subfigure}
    \caption{Learning reward curves of PPO-AdaGamma on \texttt{SafetyPointGoal1-v0}, \texttt{Humanoid-v4}, and \texttt{Ant-v4}.}
    \label{fig:reward_ppo}
\end{figure}

\section{Generalization to TRPO and DDPG}
\label{app:trpo_ddpg}
To examine whether AdaGamma transfers beyond our primary SAC/PPO implementations, we evaluate integrated DDPG~\citep{lillicrap2019continuouscontroldeepreinforcement} and TRPO~\citep{schulman2017trustregionpolicyoptimization} agents on the same three benchmarks used in Section~\ref{sec:experiments}: the constrained navigation task \texttt{SafetyPointGoal1-v0} (Table~\ref{tab:trpo_ddpg_safe}), and the high-dimensional MuJoCo locomotion domains \texttt{Humanoid-v4} (Table~\ref{tab:trpo_ddpg_human}) and \texttt{Ant-v4} (Table~\ref{tab:trpo_ddpg_ant}). We deliberately keep environments, seeds, horizons, evaluation protocol, and discount-related baselines aligned with the main experiments so that any differences can be attributed to the choice of base algorithm rather than task setup. \texttt{SafetyPointGoal1-v0} tests reward--cost tradeoffs under safety constraints, while \texttt{Humanoid-v4} and \texttt{Ant-v4} stress long-horizon coordination and contact-rich dynamics, respectively.

\begin{table}[ht]
  \centering
  \caption{Performance in \texttt{SafetyPointGoal1-v0}.}
  \label{tab:trpo_ddpg_safe}
    \scalebox{0.7}{
  \begin{tabular}{l | cc|| l | cc}
    \toprule
    Method Variants & Reward & Cost & Method Variants & Reward & Cost\\
    \midrule
   TRPO-Fixed-$\gamma$      & 24.48 $\pm$ 2.01  & \textbf{45.90 $\pm$ 28.90}  & DDPG-Fixed-$\gamma$      & 27.34 $\pm$ 1.18  & 51.45 $\pm$ 41.13  \\
    TRPO-CrossValidate       & 24.38 $\pm$ 7.11  & 46.35 $\pm$ 28.92  & DDPG-CrossValidate       & 27.26 $\pm$ 1.23  &  53.85 $\pm$ 41.92   \\
    TRPO-Uncertainty         &23.50 $\pm$ 1.90   & 64.55 $\pm$ 35.31  & DDPG-Uncertainty         & 27.10 $\pm$ 0.70  & 52.70 $\pm$ 40.64  \\
    TRPO-AdaGamma            & \textbf{25.93} $\pm$ 1.46  &  49.40 $\pm$ 29.17   & DDPG-AdaGamma            & \textbf{27.45 $\pm$ 1.40}  & \textbf{46.15 $\pm$ 32.72}   \\
    \bottomrule
  \end{tabular}
    }
\end{table}

\begin{table}[ht]
  \centering
  \caption{Performance on \texttt{Humanoid-v4}.}
  \label{tab:trpo_ddpg_human}
  \scalebox{0.7}{
  \begin{tabular}{l | r || l | r}
    \toprule
    Method Variants & Reward & Method Variants & Reward \\
    \midrule
    TRPO-Fixed-$\gamma$      & 218.15 $\pm$ 11.33  & DDPG-Fixed-$\gamma$      &  165.82 $\pm$ 5.67 \\
    TRPO-CrossValidate       & 221.10 $\pm$ 10.10 & DDPG-CrossValidate       &  356.85 $\pm$ 18.47   \\
    TRPO-Uncertainty         & 250.65 $\pm$ 25.40  & DDPG-Uncertainty         & 454.46$\pm$ 13.18  \\
    TRPO-AdaGamma            &  \textbf{284.49 $\pm$ 11.21}   & DDPG-AdaGamma            &  \textbf{457.02 $\pm$ 45.58}  \\
    \bottomrule
  \end{tabular}
  }
\end{table}

\begin{table}[ht]
  \centering
  \caption{Performance on \texttt{Ant-v4}.}
  \label{tab:trpo_ddpg_ant}
  \scalebox{0.7}{
  \begin{tabular}{l | r || l | r}
    \toprule
    Method Variants & Reward & Method Variants & Reward \\
    \midrule
    TRPO-Fixed-$\gamma$      & 1326.83 $\pm$ 103.37  & DDPG-Fixed-$\gamma$      &  3136.60  $\pm$  1106.70 \\
    TRPO-CrossValidate       & 940.26 $\pm$ 457.39  & DDPG-CrossValidate       &   1968.91 $\pm$  687.19  \\
    TRPO-Uncertainty         &  940.26 $\pm$ 457.39  & DDPG-Uncertainty         & 2915.64  $\pm$  1179.86  \\
    TRPO-AdaGamma            & \textbf{1451.34 $\pm$ 243.90}   & DDPG-AdaGamma            & \textbf{3774.19  $\pm$  458.63}   \\
    \bottomrule
  \end{tabular}
  }
\end{table}

\section{Classic Control Results}
\label{app:classic_results}

To further assess whether state-dependent discounting is useful beyond the safety and high-dimensional
locomotion settings, we evaluate AdaGamma on a set of standard Gymnasium control benchmarks
with both SAC and PPO. Table~\ref{tab:classic} shows that AdaGamma
improves or matches the base algorithm on most classic-control tasks. On \texttt{CartPole-v1}, both
SAC--AdaGamma and PPO--AdaGamma reach the maximum score of $500$, reducing the residual
variance observed in the fixed-discount baselines. On \texttt{Pendulum-v1},
\texttt{MountainCarContinuous-v0}, and \texttt{Acrobot-v1}, performance remains below the task
optimum and exhibits greater variability; even so, AdaGamma is often competitive with or modestly
better than the fixed-discount baselines.

\begin{table*}[ht]
  \centering
   \caption{Test results of SAC and PPO with and without adaptive-$\gamma$.}
\scalebox{0.9}{
  \begin{tabular}{c|cc|cc}
  \toprule
       Test Results (rewards)&  SAC  & SAC-AdaGamma &  PPO  & PPO-AdaGamma  \\
  \hline
      CartPole-v1    & 483.20 $\pm$ 14.85   & \textbf{500.00 $\pm$ 0.00} & 481.90 $\pm$ 78.70   & \textbf{500.00 $\pm$ 0.00}       \\
      Pendulum-v1    & -64.832  $\pm$  62.931   &  \textbf{-58.557  $\pm$  56.618}   &-212.27 $\pm$ 150.86   & \textbf{-198.12 $\pm$ 126.97}  \\
      MountainCarContinuous-v0 & 94.47 $\pm$  0.73   & \textbf{94.60 $\pm$ 0.60}   & 84.72 $\pm$ 11.48    & \textbf{86.84$\pm$ 3.24}    \\
      Acrobot-v1 & -82.91 $\pm$ 12.20  & \textbf{-82.61 $\pm$ 14.64}   & \textbf{-95.34 $\pm$ 21.46}  & -115.186 $\pm$  29.02   \\
  \bottomrule
  \end{tabular}
  }
  \label{tab:classic}
\end{table*}

\section{Snapshots of \texttt{Humanoid-v4} and \texttt{Ant-v4}}
\label{app:snapshots}
This section shows snapshots of trained SAC-AdaGamma and PPO-AdaGamma agents on
\texttt{Humanoid-v4} and \texttt{Ant-v4}.

\begin{figure}[ht]
    \centering
    \begin{subfigure}[b]{0.19\textwidth}
        \centering
        \includegraphics[width=\textwidth]{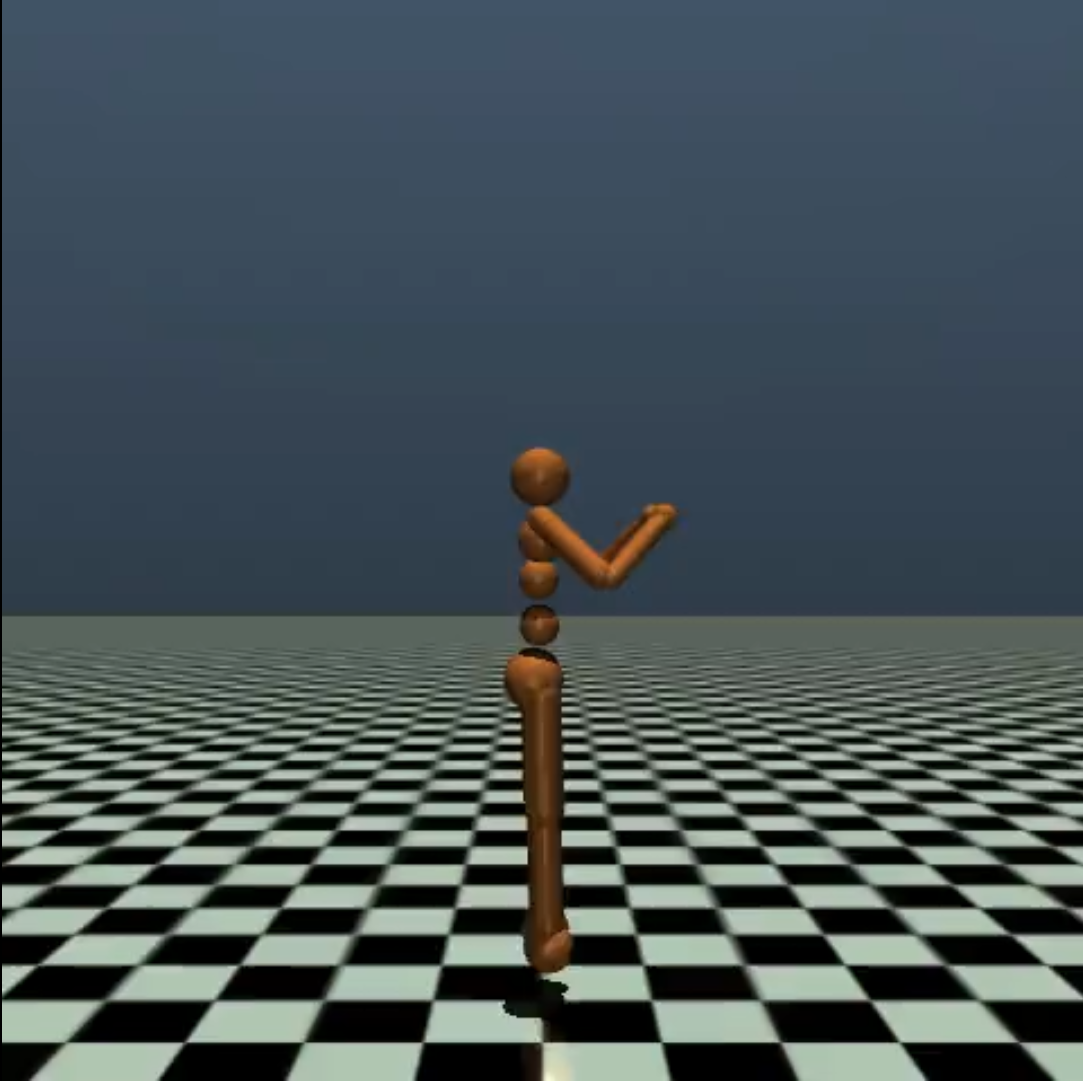}
    \end{subfigure}
    \hfill
    \begin{subfigure}[b]{0.19\textwidth}
        \centering
        \includegraphics[width=\textwidth]{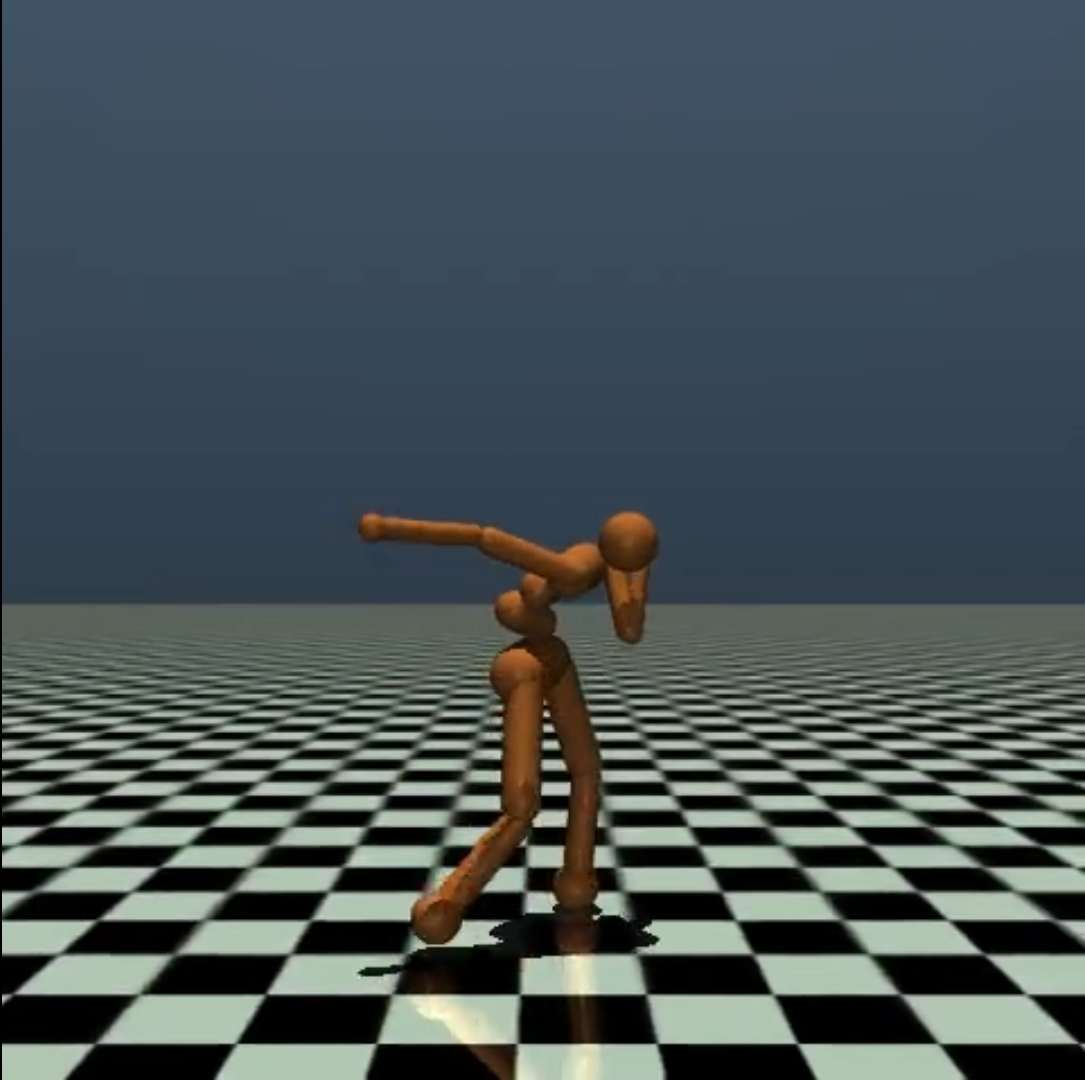}
    \end{subfigure}
    \hfill
    \begin{subfigure}[b]{0.19\textwidth}
        \centering
        \includegraphics[width=\textwidth]{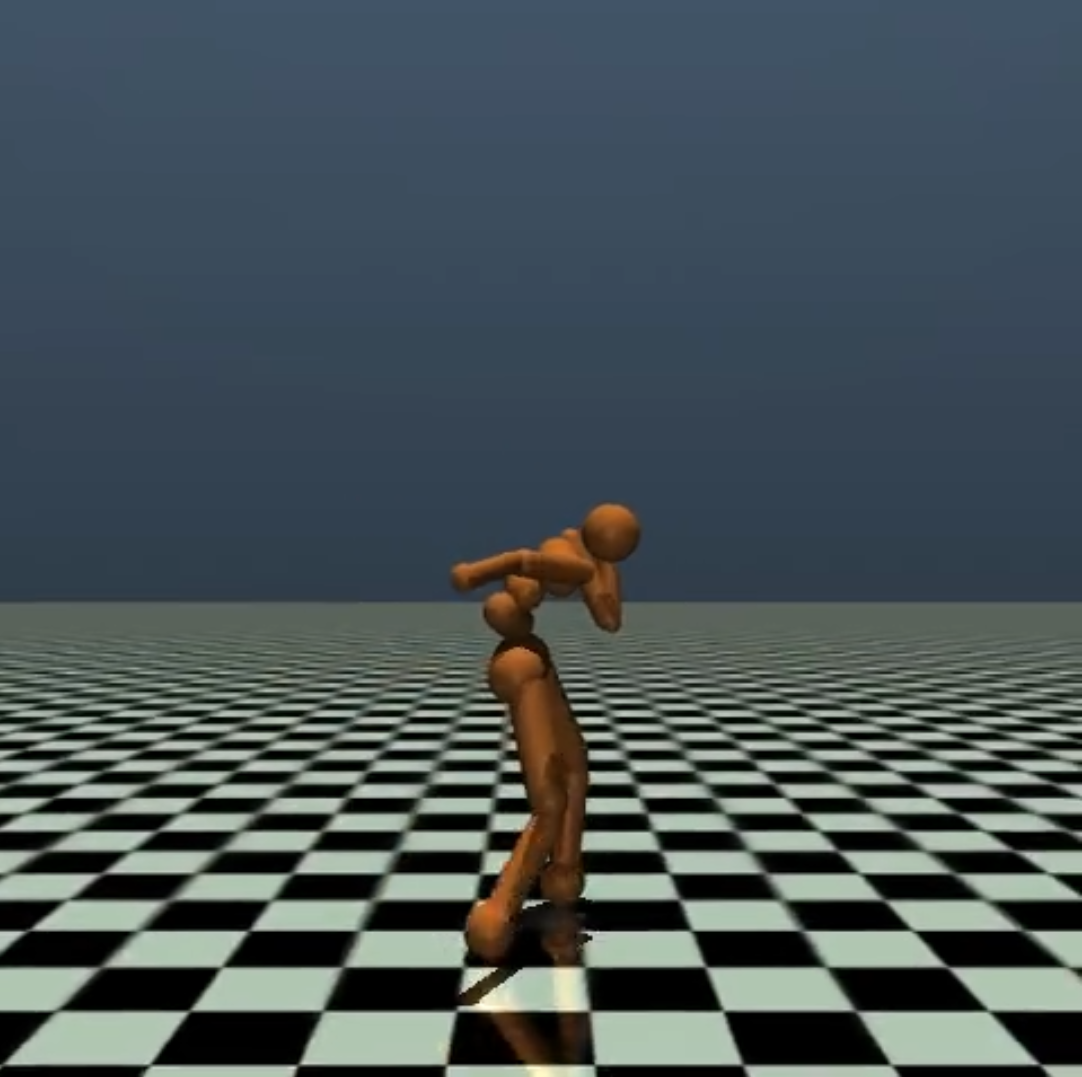}
    \end{subfigure}
    \hfill
    \begin{subfigure}[b]{0.19\textwidth}
        \centering
        \includegraphics[width=\textwidth]{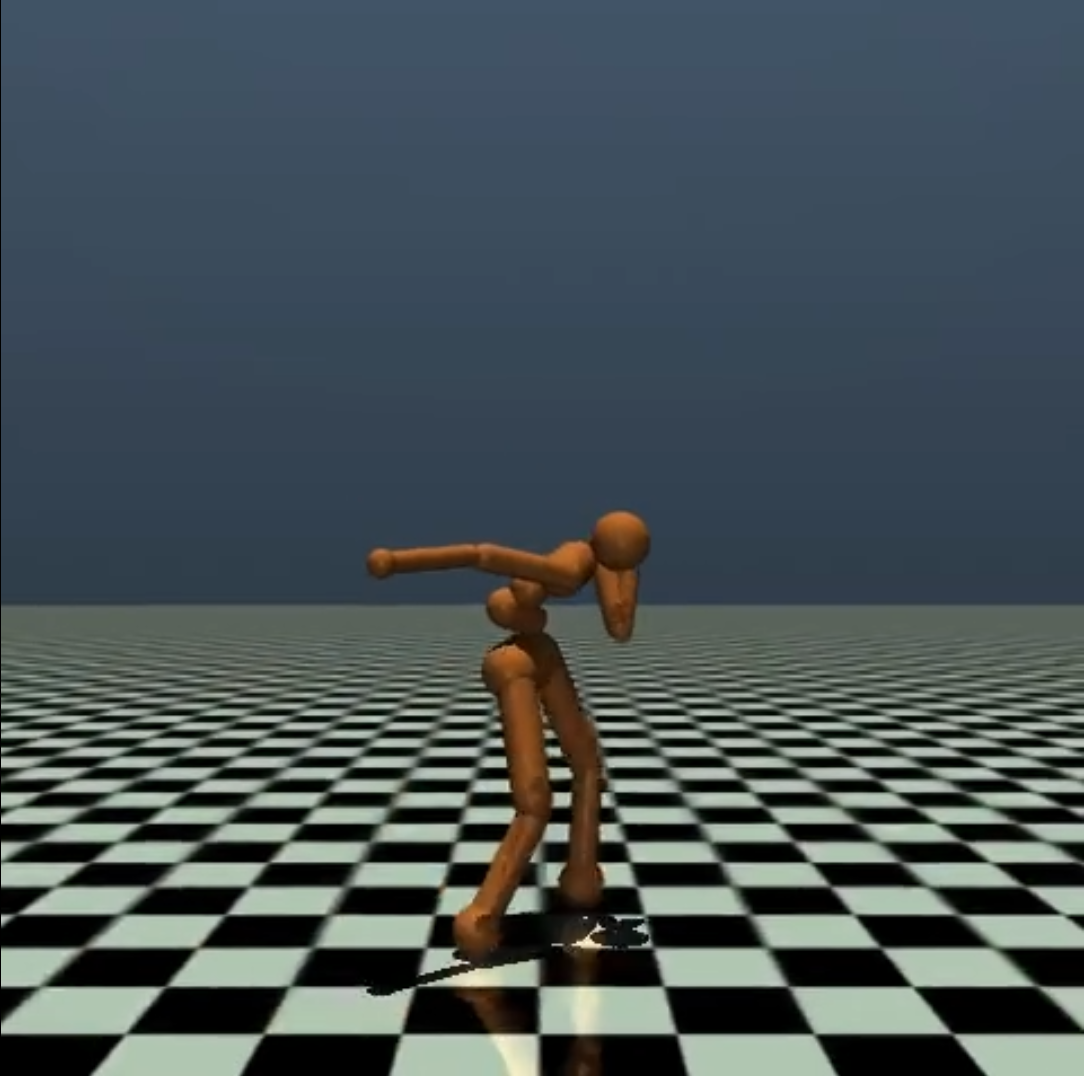}
    \end{subfigure}
    \hfill
    \begin{subfigure}[b]{0.19\textwidth}
        \centering
        \includegraphics[width=\textwidth]{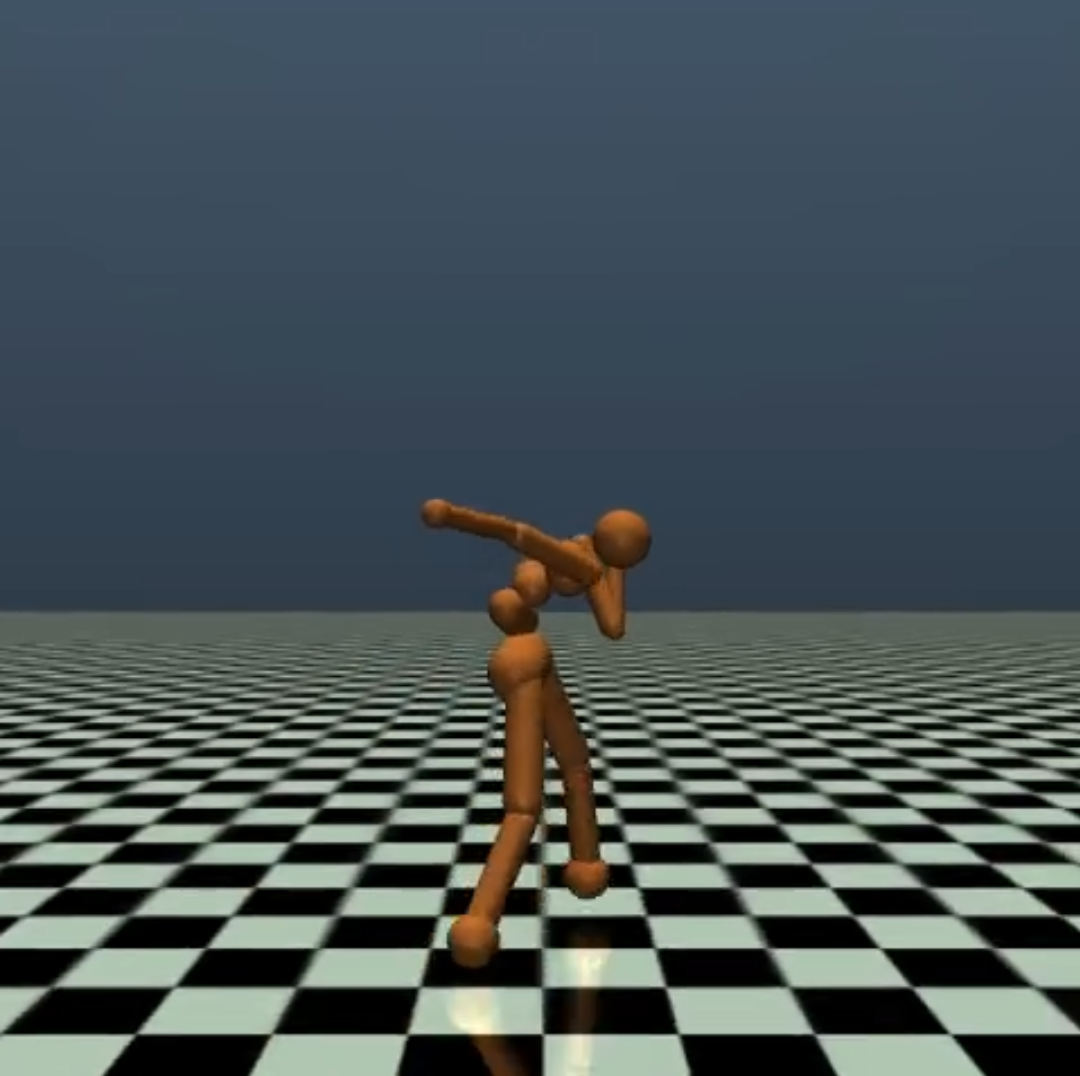}
    \end{subfigure}
     \caption{Snapshots of SAC-AdaGamma on \texttt{Humanoid-v4}.}
    \label{fig:sac_humanoid}
\end{figure}

\begin{figure}[ht]
    \centering
    \begin{subfigure}[b]{0.19\textwidth}
        \centering
        \includegraphics[width=\textwidth]{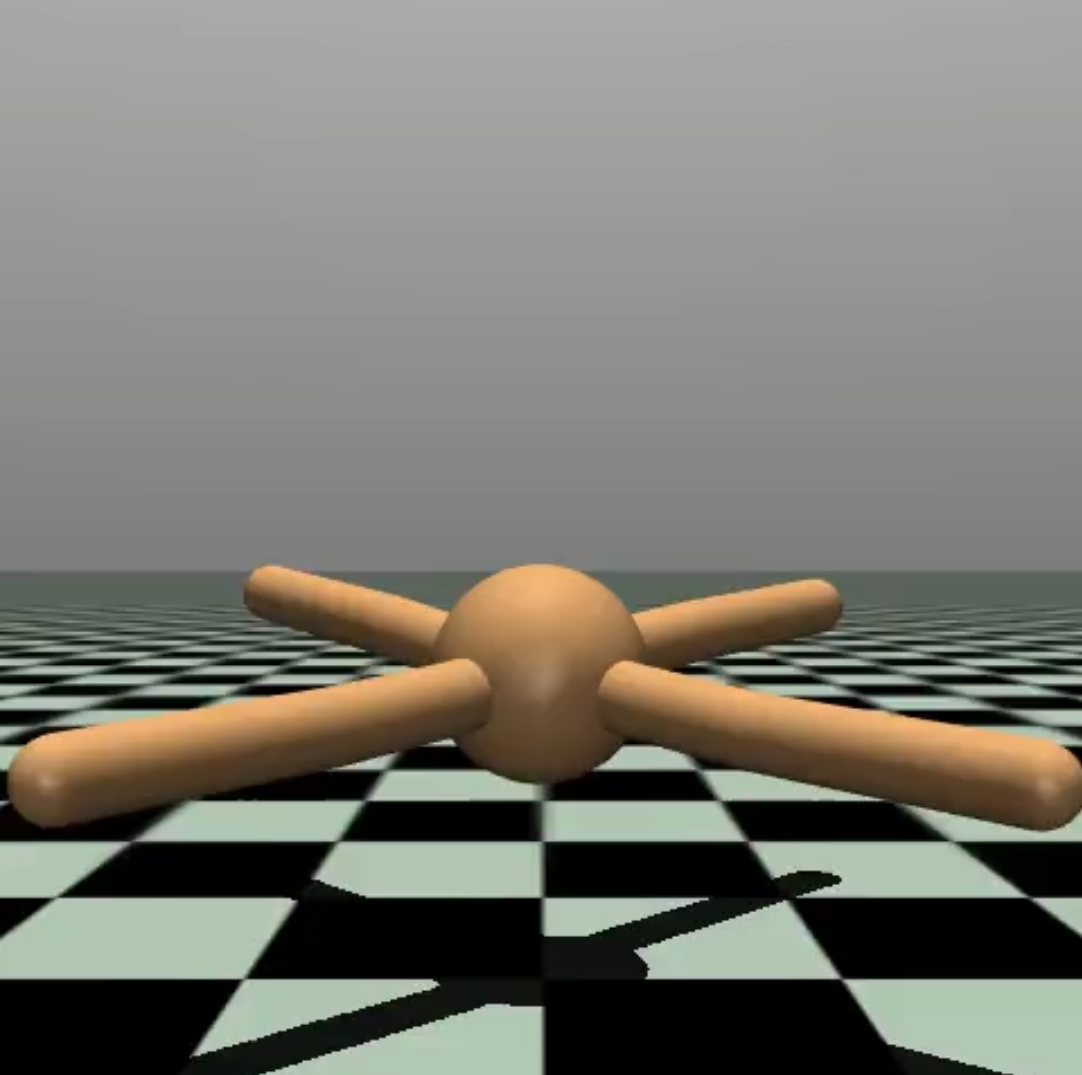}
    \end{subfigure}
    \hfill
    \begin{subfigure}[b]{0.19\textwidth}
        \centering
        \includegraphics[width=\textwidth]{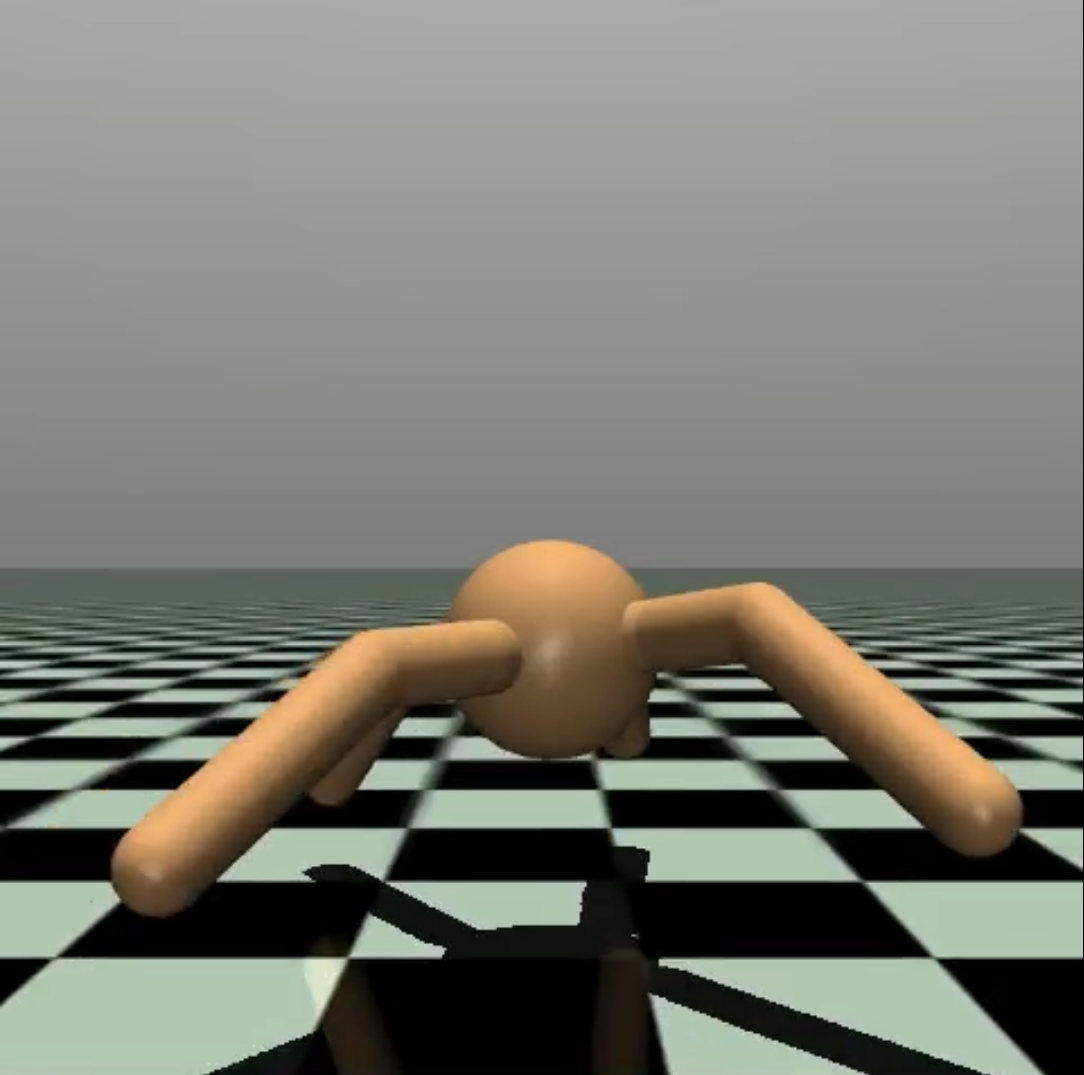}
    \end{subfigure}
    \hfill
    \begin{subfigure}[b]{0.19\textwidth}
        \centering
        \includegraphics[width=\textwidth]{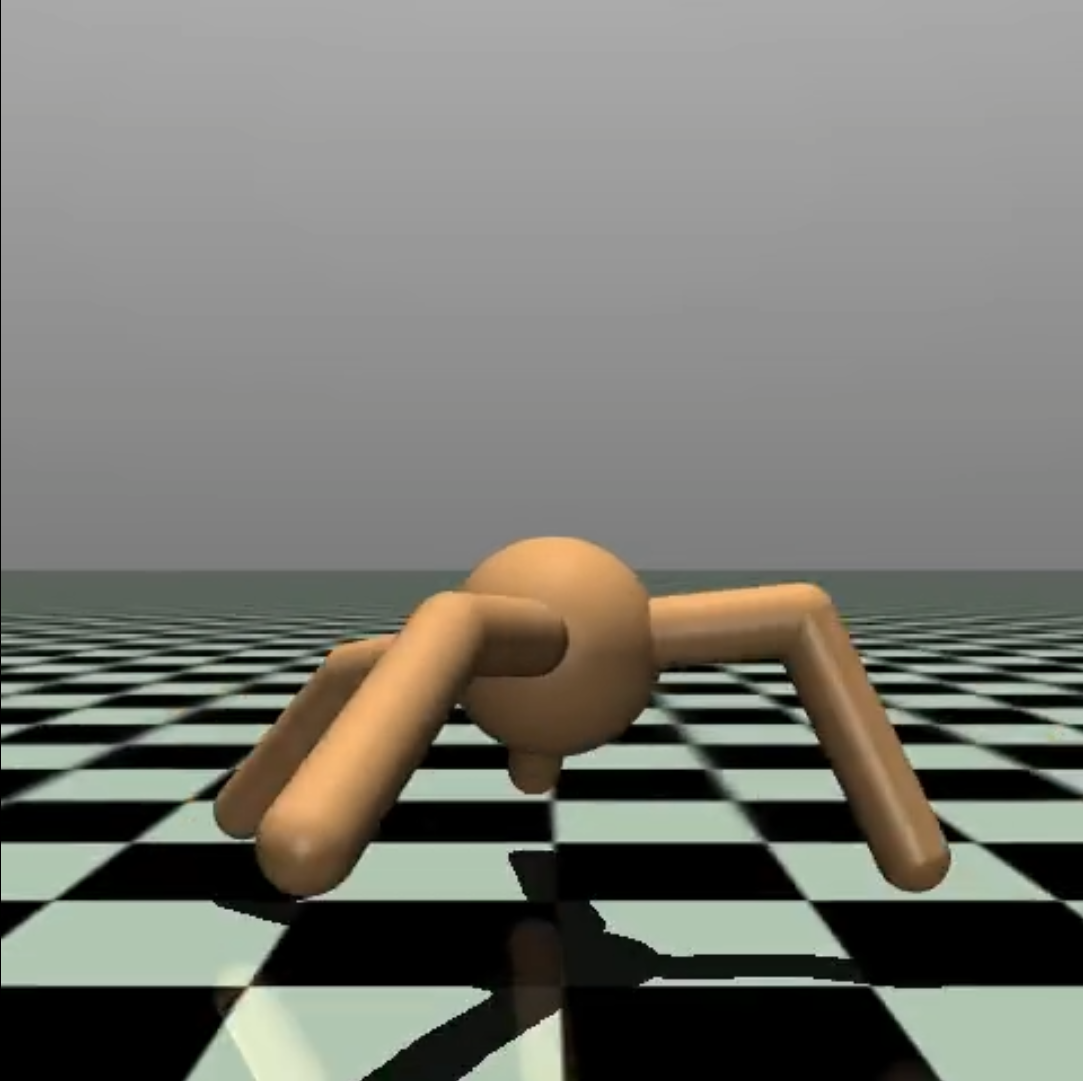}
    \end{subfigure}
    \hfill
    \begin{subfigure}[b]{0.19\textwidth}
        \centering
        \includegraphics[width=\textwidth]{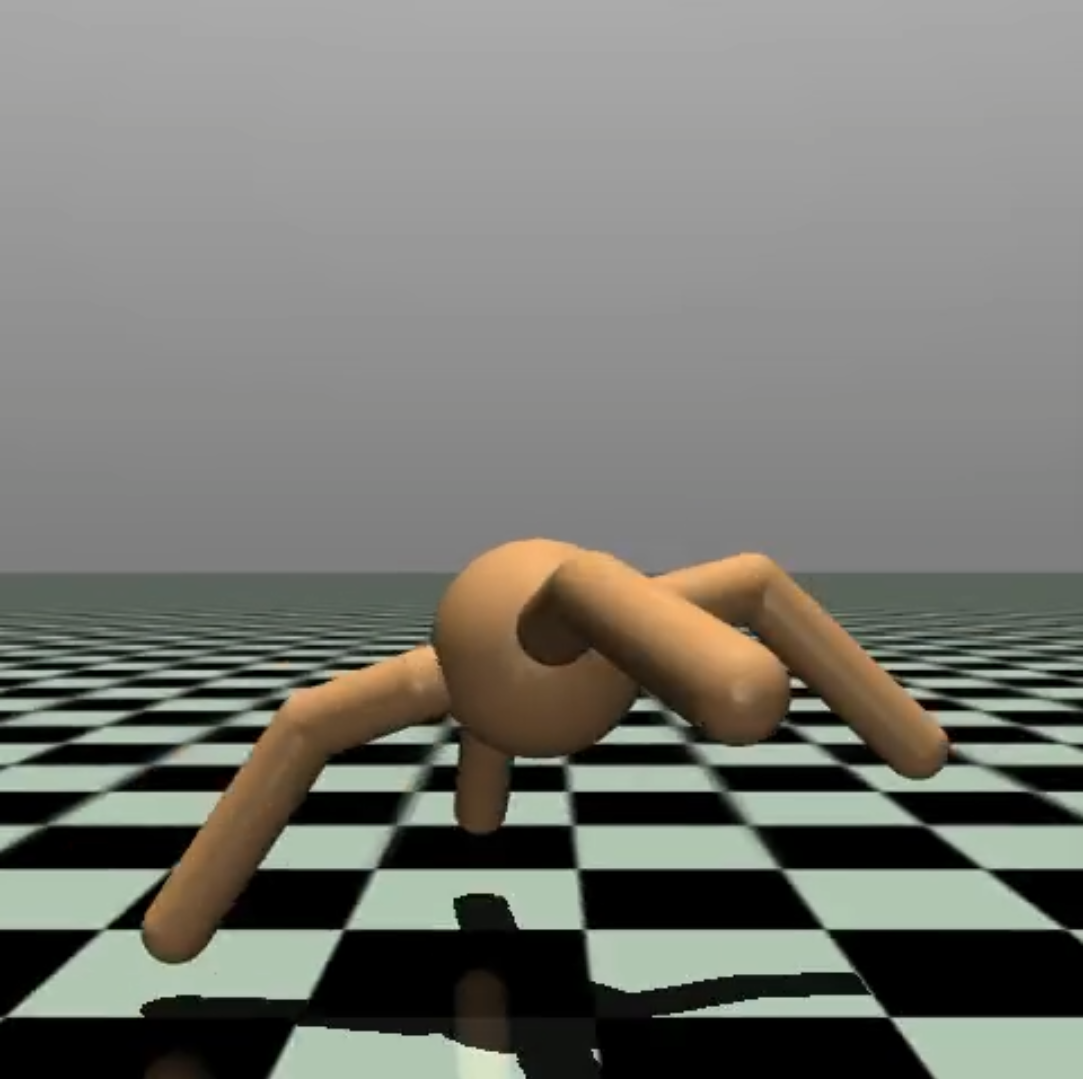}
    \end{subfigure}
    \hfill
    \begin{subfigure}[b]{0.19\textwidth}
        \centering
        \includegraphics[width=\textwidth]{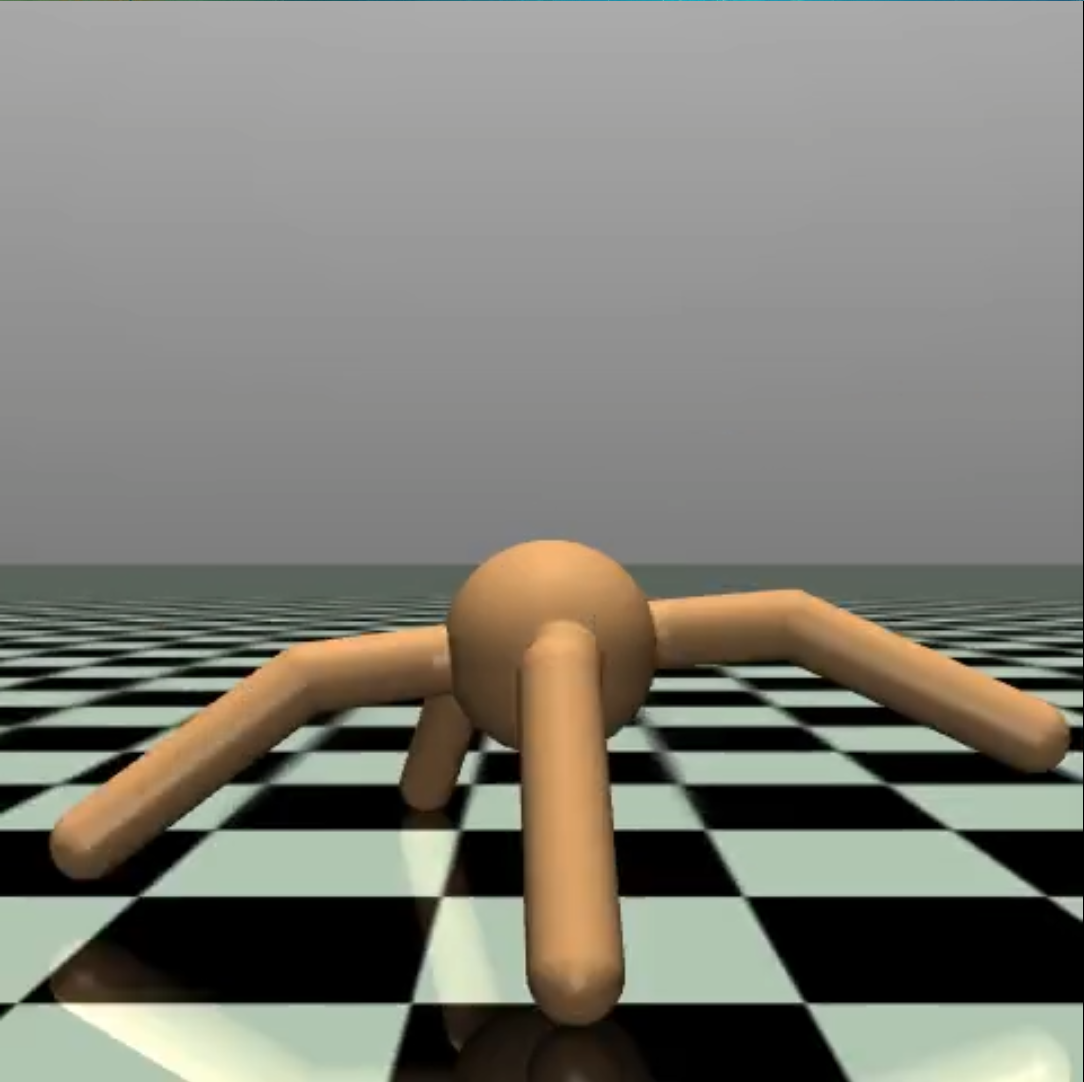}
    \end{subfigure}
     \caption{Snapshots of SAC-AdaGamma on \texttt{Ant-v4}.}
    \label{fig:sac_ant}
\end{figure}

\begin{figure}[ht]
    \centering
    \begin{subfigure}[b]{0.19\textwidth}
        \centering
        \includegraphics[width=\textwidth]{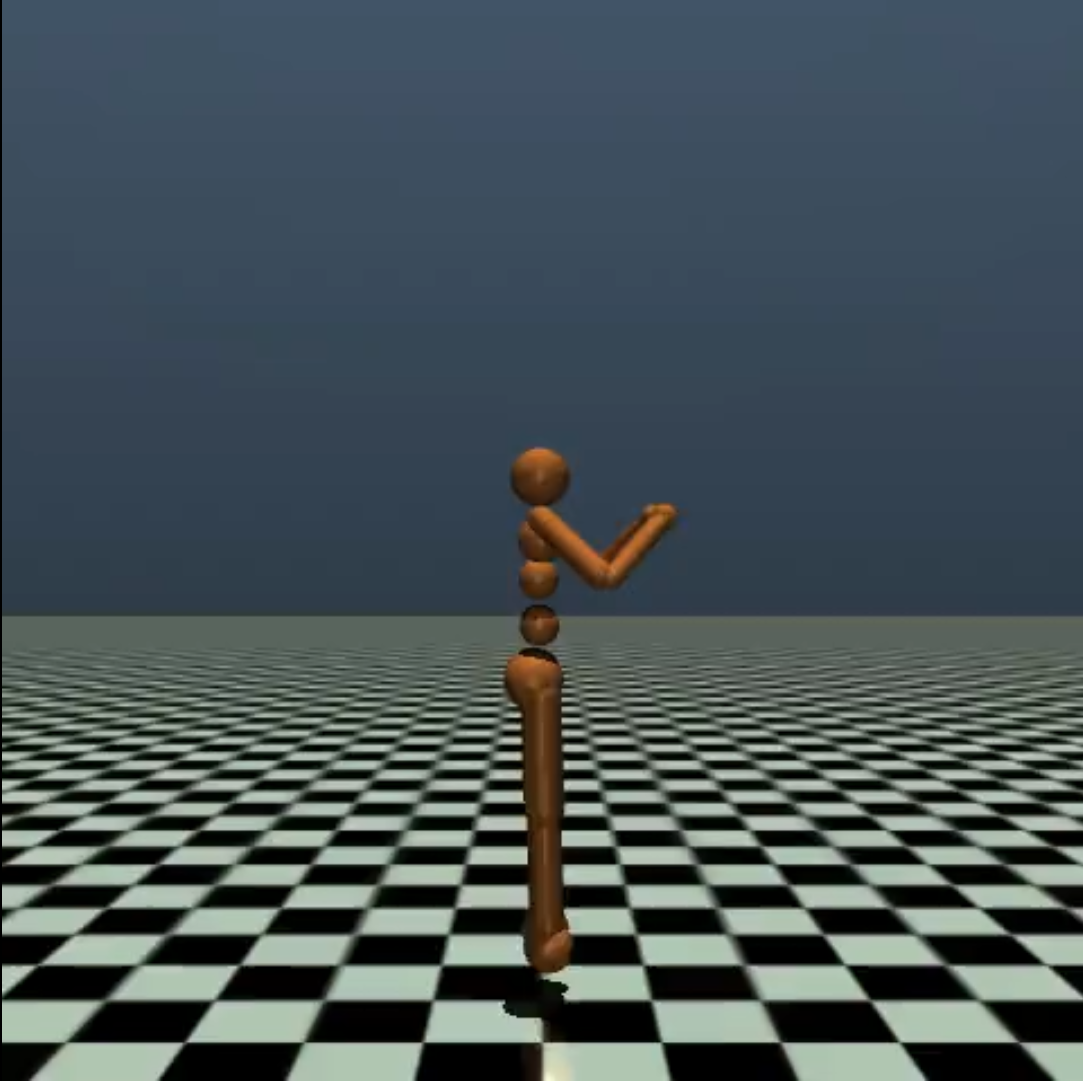}
    \end{subfigure}
    \hfill
    \begin{subfigure}[b]{0.19\textwidth}
        \centering
        \includegraphics[width=\textwidth]{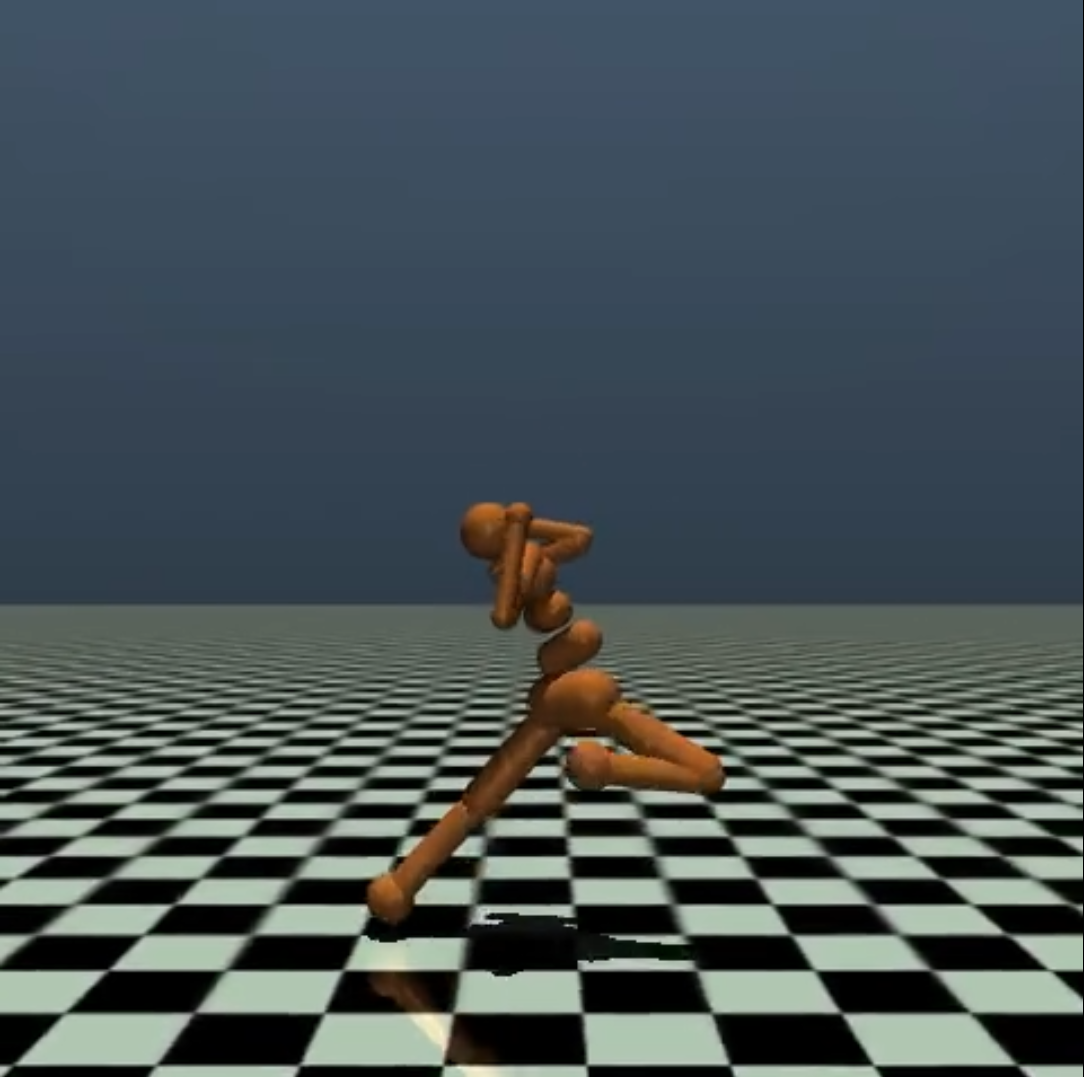}
    \end{subfigure}
    \hfill
    \begin{subfigure}[b]{0.19\textwidth}
        \centering
        \includegraphics[width=\textwidth]{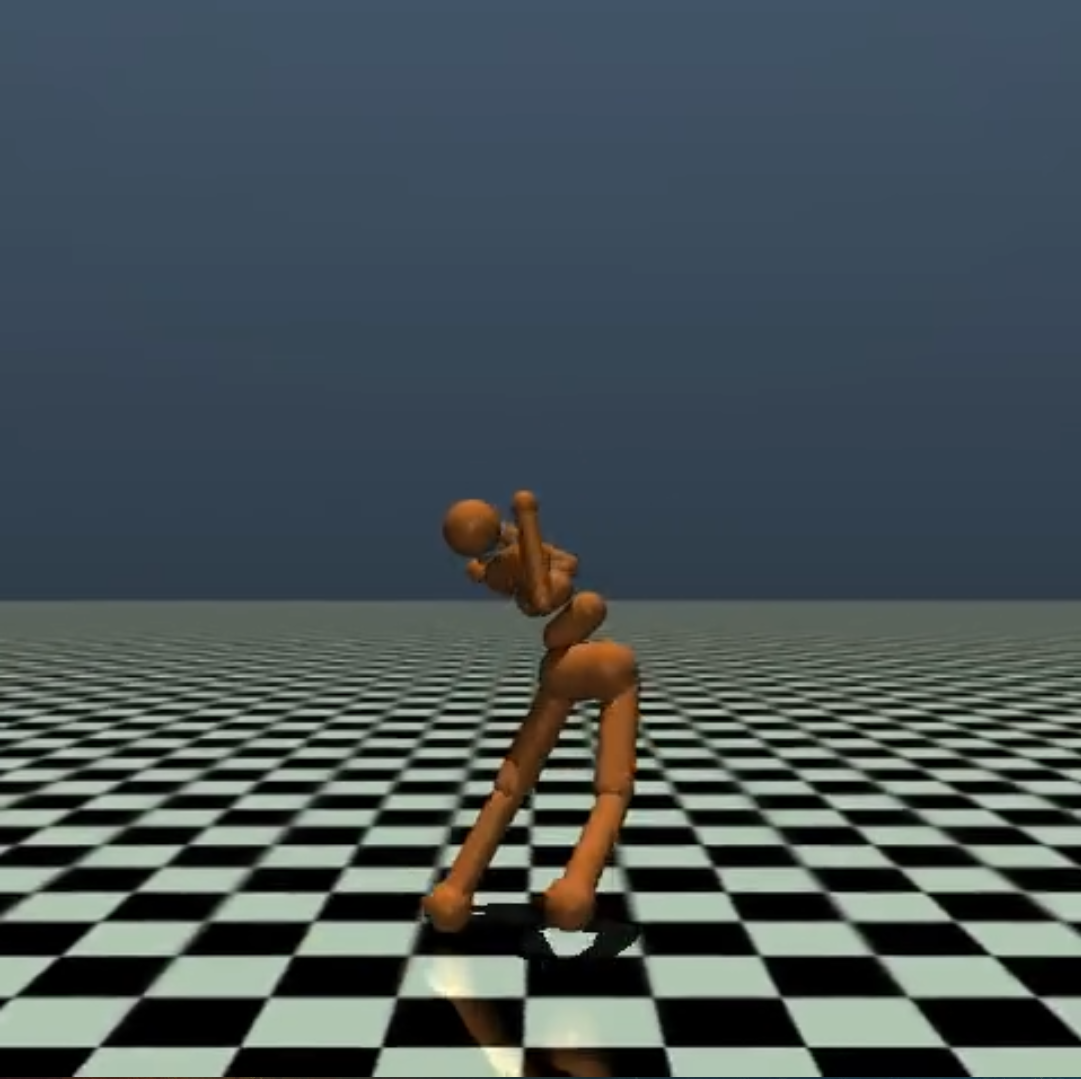}
    \end{subfigure}
    \hfill
    \begin{subfigure}[b]{0.19\textwidth}
        \centering
        \includegraphics[width=\textwidth]{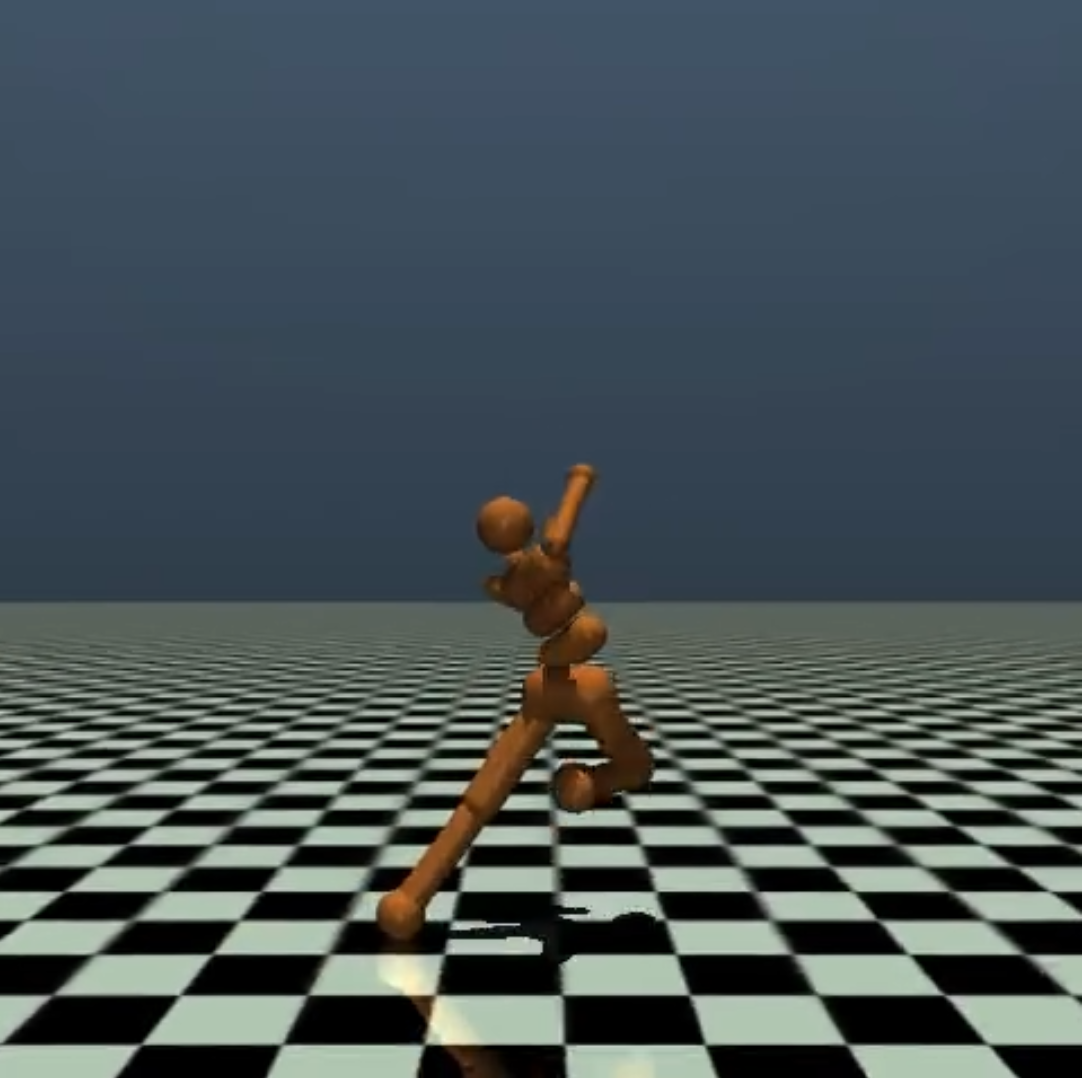}
    \end{subfigure}
    \hfill
    \begin{subfigure}[b]{0.19\textwidth}
        \centering
        \includegraphics[width=\textwidth]{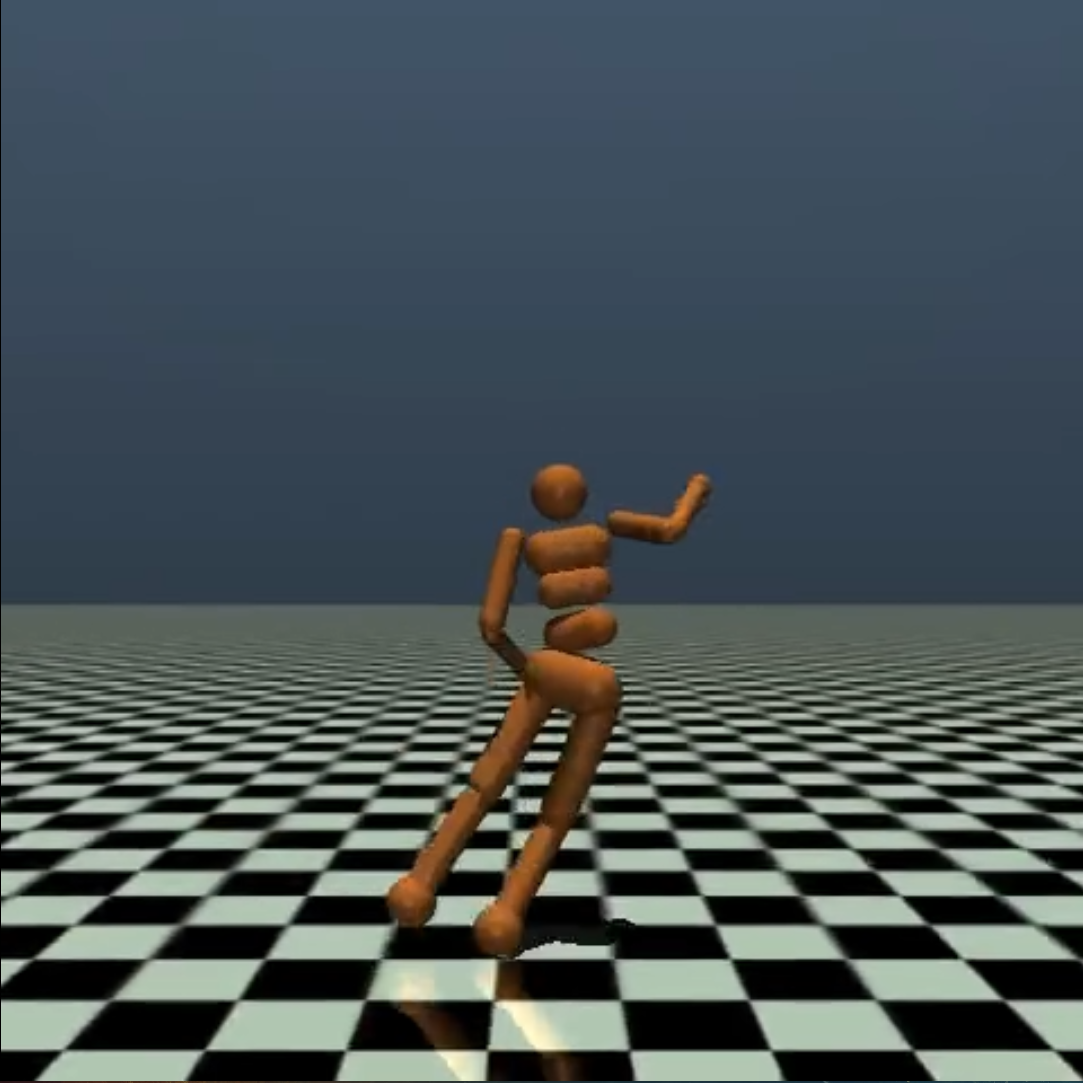}
    \end{subfigure}
     \caption{Snapshots of PPO-AdaGamma on \texttt{Humanoid-v4}.}
    \label{fig:ppo_humanoid}
\end{figure}

\begin{figure}[ht]
    \centering
    \begin{subfigure}[b]{0.19\textwidth}
        \centering
        \includegraphics[width=\textwidth]{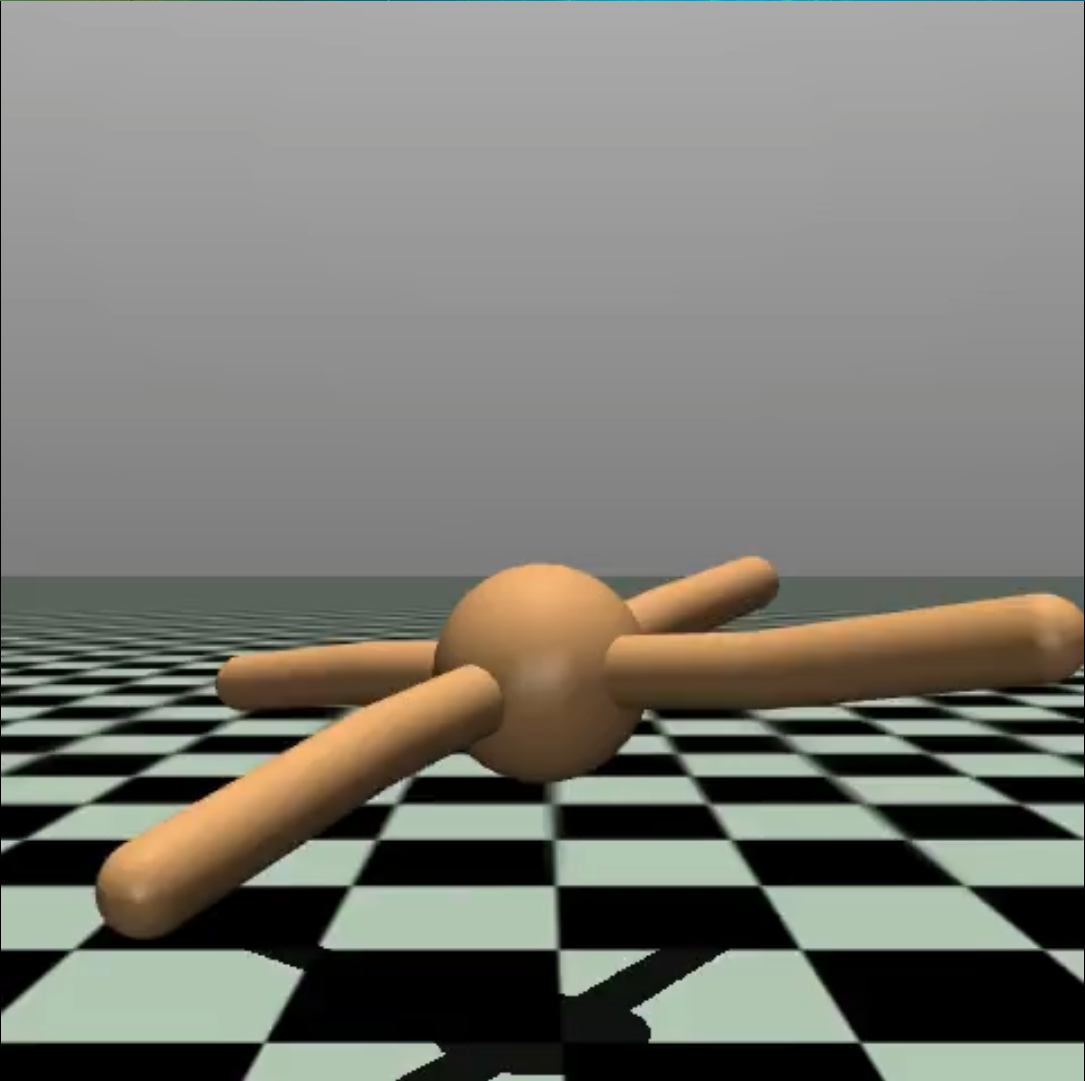}
    \end{subfigure}
    \hfill
    \begin{subfigure}[b]{0.19\textwidth}
        \centering
        \includegraphics[width=\textwidth]{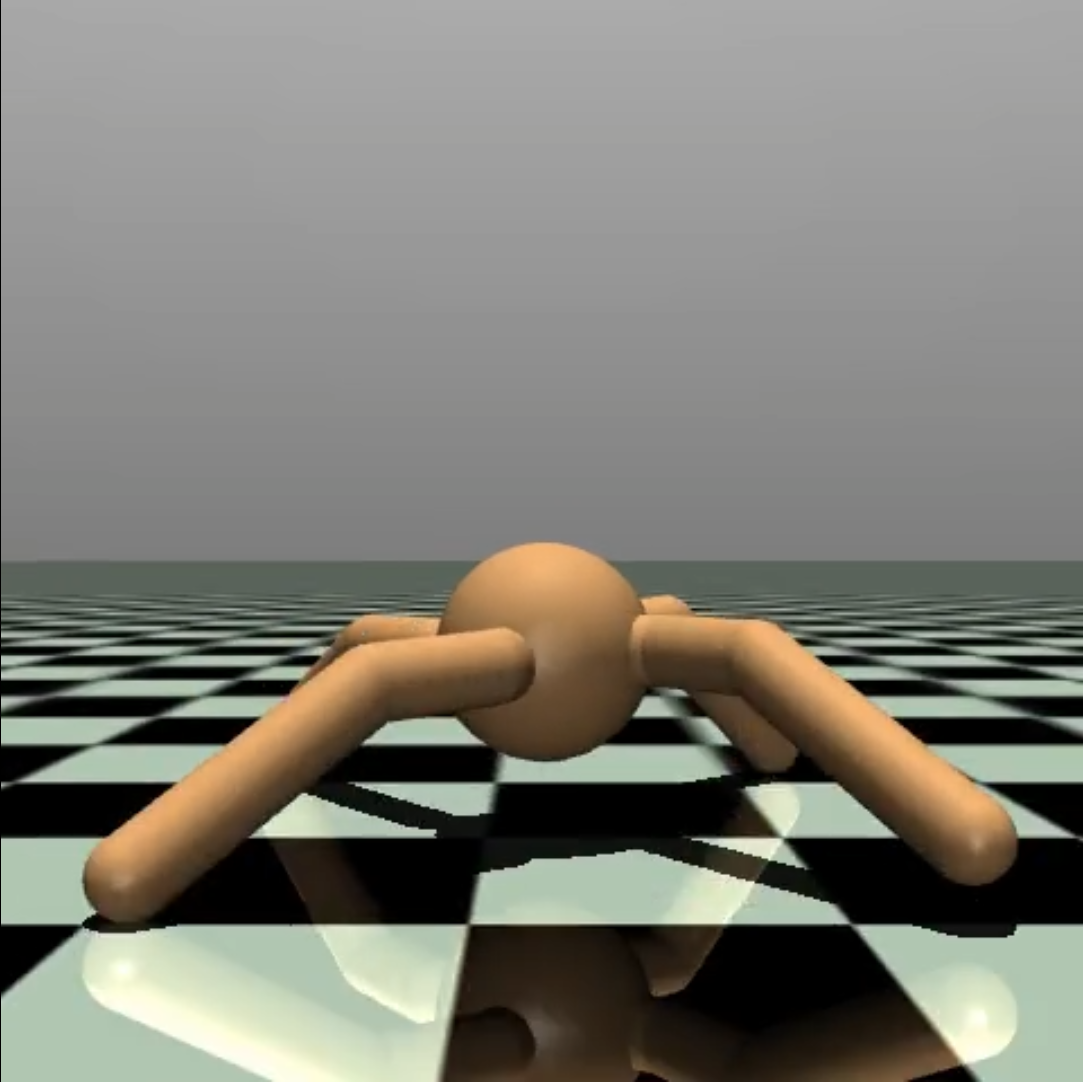}
    \end{subfigure}
    \hfill
    \begin{subfigure}[b]{0.19\textwidth}
        \centering
        \includegraphics[width=\textwidth]{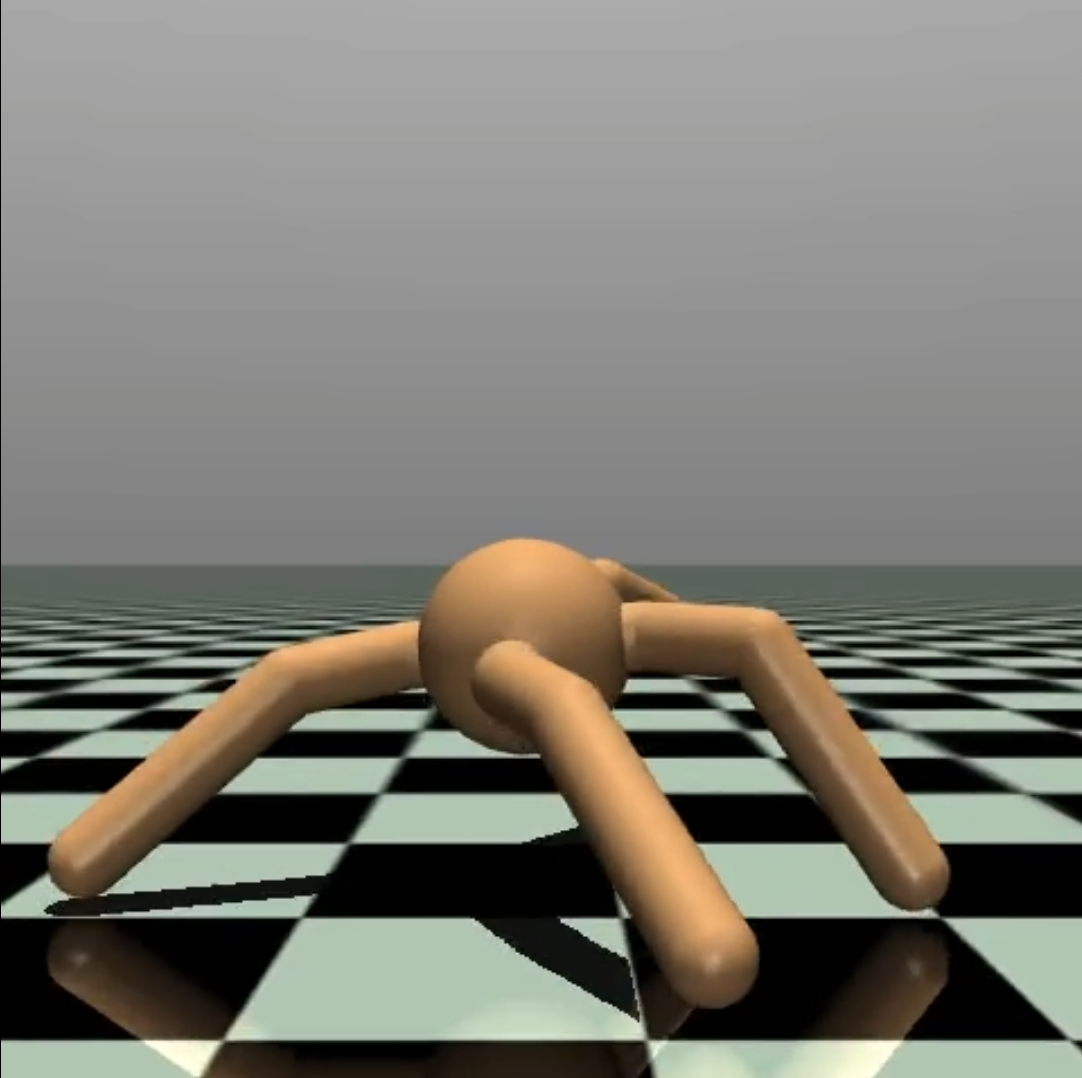}
    \end{subfigure}
    \hfill
    \begin{subfigure}[b]{0.19\textwidth}
        \centering
        \includegraphics[width=\textwidth]{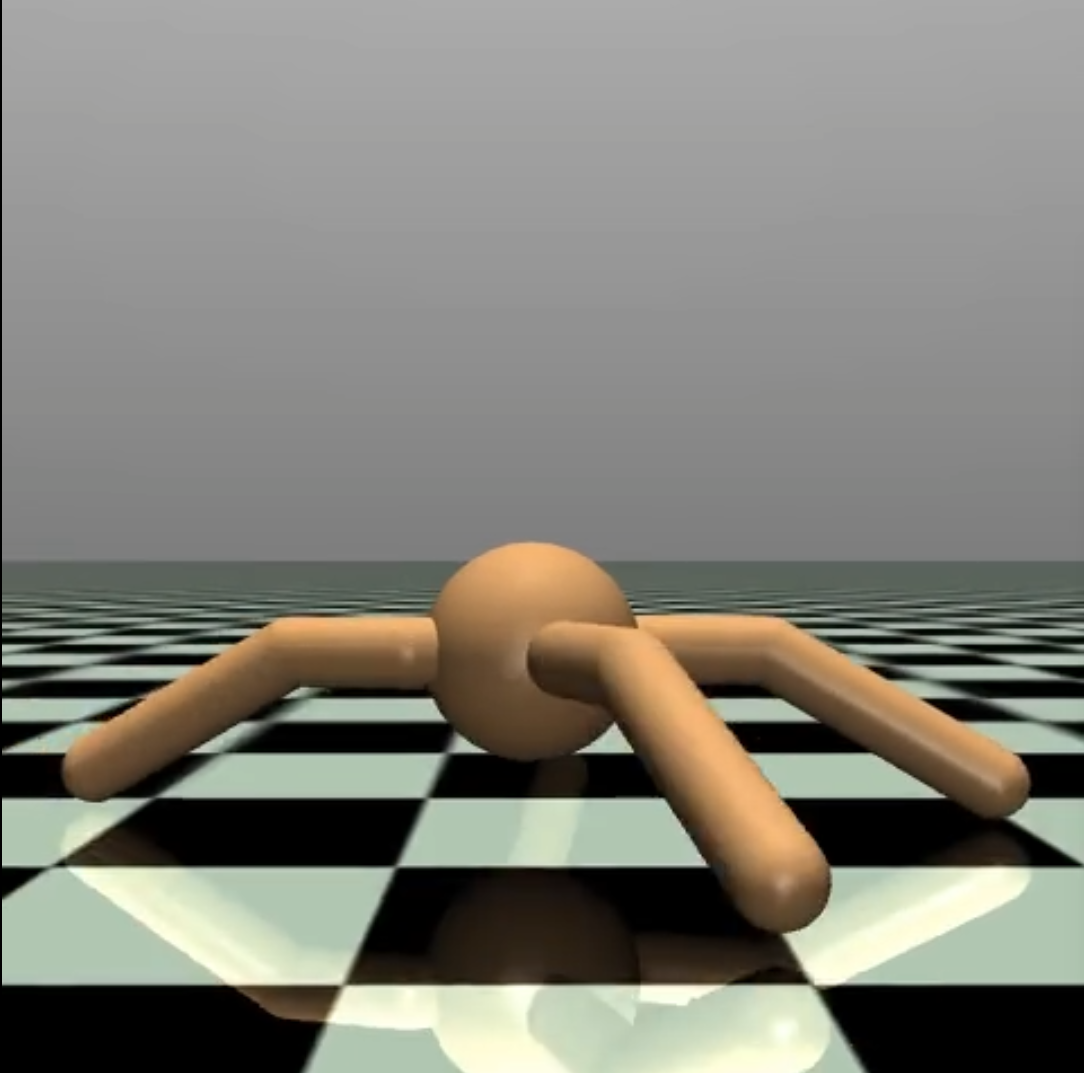}
    \end{subfigure}
    \hfill
    \begin{subfigure}[b]{0.19\textwidth}
        \centering
        \includegraphics[width=\textwidth]{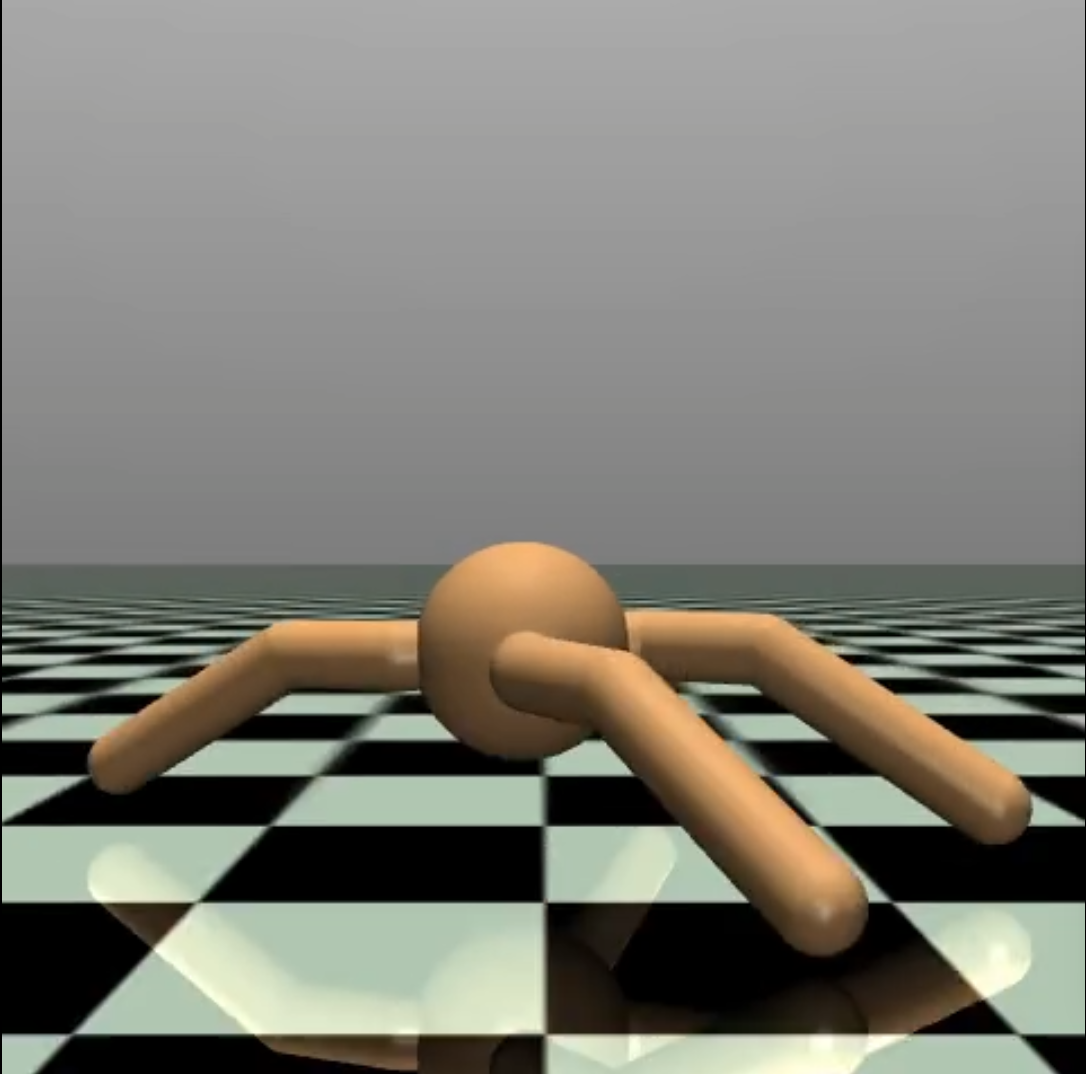}
    \end{subfigure}
     \caption{Snapshots of PPO-AdaGamma on \texttt{Ant-v4}.}
    \label{fig:ppo_ant}
\end{figure}

\section{Hyperparameter Configurations}
\label{app:hyperparams}

\begin{table*}[t]
  \centering
  \caption{SAC hyperparameters by task.}
  \label{tab:hyperparams-sac}
  \scalebox{0.7}{
  \setlength{\tabcolsep}{4pt}
  \begin{tabular}{@{}llccc@{}}
    \toprule
     & & SafetyPointGoal1-v0 &  Humanoid-v4 & Ant-v4 \\
    \midrule
    \multicolumn{5}{@{}l}{\emph{Task / schedule}} \\
    & Max env steps & $10^6$ & $10^6$ & $10^6$ \\
    & Eval interval & \multicolumn{3}{c}{$10^4$ steps} \\
    \midrule
    \multicolumn{5}{@{}l}{\emph{SAC backbone (shared \texttt{Config})}} \\
    & Actor/Critic LR & \multicolumn{3}{c}{$3\times 10^{-4}$} \\
    & Target smoothing $\tau$ & \multicolumn{3}{c}{$5\times 10^{-3}$} \\
    & Replay capacity & \multicolumn{3}{c}{$10^6$} \\
    & Min buffer before train & \multicolumn{3}{c}{$5000$} \\
    & Batch size & \multicolumn{3}{c}{$256$} \\
    & Grad steps / env step & \multicolumn{3}{c}{$1$} \\
    & Hidden dim (policy/Q) & \multicolumn{3}{c}{$256$} \\
    & Entropy coef.\ $\alpha$ & \multicolumn{3}{c}{$0.2$ (automatic tuning)} \\
    & Max grad norm & \multicolumn{3}{c}{$1.0$} \\
    & Episode horizon cap & \multicolumn{3}{c}{$1000$} \\
    \midrule
    \multicolumn{5}{@{}l}{\emph{Safety constraint (SafetyPointGoal1 only)}} \\
    & Enabled / cost limit & \multicolumn{3}{c}{yes / $25$} \\
    & Lagrange LR / init & \multicolumn{3}{c}{$5\times 10^{-3}$ / $10^{-3}$} \\
    \midrule
    \multicolumn{5}{@{}l}{\emph{Uncertainty-rule baseline ($\beta$)}} \\
    & $\beta_{\mathrm{init}}$; learnable; $\beta$ LR &
      \multicolumn{3}{c}{$2.0$; true; $10^{-3}$} \\
    & Target $\gamma$ (\texttt{uncertainty\_target}) & --- & $0.99$ & $0.99$ \\
    & CLI: \texttt{uncertainty-scale} $=s \Rightarrow \beta_{\mathrm{init}}=2s$ &
      \multicolumn{3}{c}{$s=1$ in scripts} \\
    \midrule
    \multicolumn{5}{@{}l}{\emph{Fixed / RC / cross-validated (per script)}} \\
    & Warmup steps (\texttt{gamma-warmup-steps}) &
      $10^5$ & $10^5$ & $10^5$ \\
    & $\gamma$-net LR / hidden & \multicolumn{3}{c}{$10^{-4}$ / $256$} \\
    & Clip $[\gamma_{\min},\gamma_{\max}]$ &
      $[0.9, 0.999]$ & $[0.9, 0.999]$ & $[0.97, 0.999]$ \\
    & $n$-step $n$ & $5$ & $5$ & $5$ \\
    & Fixed $\gamma$ & $0.99$ & $0.99$ & $0.99$ \\
    & RC: $\gamma$ / $\bar{\gamma}$ init &
      $0.98$ / $0.98$ & $0.98$ / $0.98$ & $0.98$ / $0.98$ \\
    & RC: $\bar{\gamma}$ adaptive; after warmup &
      \multicolumn{3}{c}{true / true} \\
    & RC: EMA $\tau$ / update every (episodes) &
      \multicolumn{3}{c}{$0.1$ / $5$} \\
    & RC: train every (\texttt{gamma-update-freq} steps) &
      $20$ & $20$ & $20$ \\
    & $\lambda_{\mathrm{rc}}$ & \multicolumn{3}{c}{$1.0$} \\
    & $\lambda_{\mathrm{dev}}/\lambda_{\mathrm{var}}/\lambda_{\mathrm{bound}}$ &
      \multicolumn{3}{c}{$0.005\,/\,0.012\,/\,0.05$} \\
    & CV: $\gamma$; $H_{\mathrm{cv}}$; loss weight &
      --- & $0.98$; $20$; $1.0$ & $0.98$; $20$; $1.0$ \\
    & Uncertainty: \texttt{gamma-base} & --- & $0.99$ & $0.98$ \\
    \bottomrule
  \end{tabular}
  }
\end{table*}

\begin{table*}[t]
  \centering
  \caption{PPO hyperparameters by task.}
  \label{tab:hyperparams-ppo-by-task}
  \scalebox{0.7}{
  \setlength{\tabcolsep}{3.5pt}
  \begin{tabular}{@{}llccc@{}}
    \toprule
     & Hyperparameter
     & SafetyPointGoal1-v0
     & Humanoid-v4
     & Ant-v4 \\
    \midrule
    \multicolumn{5}{@{}l}{\emph{Task / schedule}} \\
    & Max env steps & $10^6$ & $3\times 10^6$ & $10^6$ \\
    & Eval interval & \multicolumn{3}{c}{$10^4$ steps} \\
    \midrule
    \multicolumn{5}{@{}l}{\emph{PPO core}} \\
    & Actor LR / critic LR &
      $3\!\times\!10^{-4}$ / $3\!\times\!10^{-4}$ &
      $10^{-4}$ / $10^{-4}$ &
      $3\!\times\!10^{-4}$ / $3\!\times\!10^{-4}$ \\
    & Policy clip $\varepsilon$ & \multicolumn{3}{c}{$0.2$} \\
    & GAE $\lambda$ & \multicolumn{3}{c}{$0.95$} \\
    & PPO epochs & $10$ & $8$ & $10$ \\
    & Rollout buffer size & $4096$ & $16384$ & $4096$ \\
    & Mini-batch size & $128$ & $256$ & $128$ \\
    & Entropy coef. & $0.01$ & $0.005$ & $0.01$ \\
    & Max grad norm & \multicolumn{3}{c}{$0.5$} \\
    & Action std (init / floor / decay $\Delta$) &
      \multicolumn{3}{c}{$0.5$ / $0.1$ / $0.05$} \\
    & Action-std decay period (steps) &
      $2\times 10^5$ & $10^5$ & $2\times 10^5$ \\
    & Episode horizon cap & \multicolumn{3}{c}{$1000$} \\
    \midrule
    \multicolumn{5}{@{}l}{\emph{$\gamma$ network (AdaGamma / CV / RC branches)}} \\
    & Hidden dim / $\gamma$ LR &
      \multicolumn{3}{c}{$256$ / $3\times 10^{-4}$} \\
    & Clip $[\gamma_{\min}, \gamma_{\max}]$ &
      \multicolumn{3}{c}{$[0.9,\,0.999]$} \\
    & $\gamma$-net warmup (episodes) &
      $200$ \textit{(default; script unset)} &
      $500$ &
      $200$ \textit{(default; script unset)} \\
    & Reg.\ $\lambda_{\mathrm{dev}}/\lambda_{\mathrm{var}}/\lambda_{\mathrm{bound}}/\epsilon_{\mathrm{bound}}$ &
      \multicolumn{3}{c}{$0.01$ / $0.005$ / $0.05$ / $0.005$} \\
    \midrule
    \multicolumn{5}{@{}l}{\emph{Return-consistency (RC) --- where used}} \\
    & Horizon $n$ (\texttt{rc-horizon}) &
      $\{5,10,20\}$ \textit{(separate runs)} &
      $10$ &
      $10$ \\
    & $\lambda_{\mathrm{rc}}$ & \multicolumn{3}{c}{$1.0$} \\
    & $\bar{\gamma}$ init / adaptive / after warmup &
      \multicolumn{3}{c}{$0.98$ / true / true} \\
    & Ref EMA $\tau$ / update every (PPO upd.) &
      $0.1$ / $1$ &
      $0.05$ / $5$ &
      $0.1$ / $1$ \\
    \midrule
    \multicolumn{5}{@{}l}{\emph{Per-baseline settings (non-RC)}} \\
    & Fixed $\gamma$ &
      $0.99$ &
      $0.99$ &
      $0.99$ \\
    & Uncertainty \newline
      (\texttt{gamma-base}, $\gamma_{\min}$, scale, \texttt{num-q-ensembles}) &
      $0.99$, $0.9$, $1.0$, $5$ &
      $0.99$, $0.9$, $1.0$, $5$ &
      $0.99$, $0.9$, $1.0$, $5$ \\
    & Cross-validated\newline
      ($H_{\mathrm{cv}}$, $\lambda_{\mathrm{cv}}$) &
      $10$, $1.0$ &
      $10$, $1.0$ &
      $10$, $1.0$ \\
    \bottomrule
  \end{tabular}
  }
\end{table*}

\begin{table}[t]
  \centering
  \caption{Shared hyperparameters for the collapse ablation.}
  \label{tab:collapse-hparams}
  \begin{tabular}{lc}
    \toprule
    Hyperparameter & Value \\
    \midrule
    Environment & SafetyPointGoal1-v0 \\
    Total env.\ steps & 300{,}000 \\
    $\gamma$ warmup steps & 100{,}000 \\
    $[\gamma_{\min},\gamma_{\max}]$ & $[0.9,\,0.999]$ \\
    $\gamma$ / $\bar{\gamma}$ init & 0.98 / 0.98 \\
    Adaptive $\bar{\gamma}$ & off \\
    RC $n$-step & 5 \\
    $\gamma$ network hidden dim & 256 \\
    $\gamma$ learning rate & $1\times10^{-4}$ \\
    $\gamma$ update frequency & every 20 critic steps \\
    Probe size / interval & 2048 / 5{,}000 steps \\
    Seeds & 42, 43, 44, 45, 46 \\
    \bottomrule
  \end{tabular}
\end{table}


\end{document}